\documentclass[a4paper]{article}
\usepackage[round]{natbib}
\usepackage{hyperref}
\usepackage{graphicx}
\usepackage{booktabs}
\usepackage{url}
\usepackage{algorithm}
\usepackage[noend]{algpseudocode}
\usepackage[bm,algpseudocodecmds,subfigurepackage,nocitepackage,microtype]{pauli_new}
\usepackage{xspace}
\usepackage{color}
\usepackage{footmisc}
\usepackage{rotating}
\usepackage{multirow}

\usepackage{mathtools}
\DeclarePairedDelimiter{\ceil}{\lceil}{\rceil}

\usepackage{defines}

\graphicspath{{figs/}}

\expandafter\def\expandafter\UrlBreaks\expandafter{\UrlBreaks
  \do\a\do\b\do\c\do\d\do\e\do\f\do\g\do\h\do\i\do\j%
  \do\k\do\l\do\m\do\n\do\o\do\p\do\q\do\r\do\s\do\t%
  \do\u\do\v\do\w\do\x\do\y\do\z\do\A\do\B\do\C\do\D%
  \do\E\do\F\do\G\do\H\do\I\do\J\do\K\do\L\do\M\do\N%
  \do\O\do\P\do\Q\do\R\do\S\do\T\do\U\do\V\do\W\do\X%
    \do\Y\do\Z}
\newcommand\lword[1]{\leavevmode\nobreak\hskip0pt plus\linewidth\penalty50\hskip0pt plus-\linewidth\nobreak\textbf{#1}}

\begin{document}

\title{Algorithms for Approximate Subtropical Matrix Factorization}

\author{
  Sanjar Karaev and Pauli Miettinen \\
  Max-Planck-Institut f\"ur Informatik\\
  Saarland Informatics Campus\\
  Saarbr\"ucken, Germany\\
  \texttt{\{skaraev,pmiettin\}@mpi-inf.mpg.de}
}
\date{}

\maketitle

\begin{abstract}
Matrix factorization methods are important tools in data mining and analysis. They can be used for many tasks, ranging from dimensionality reduction to visualization. In this paper we concentrate on the use of matrix factorizations for finding patterns from the data. Rather than using the standard algebra -- and the summation of the rank-1 components to build the approximation of the original matrix -- we use the subtropical algebra, which is an algebra over the nonnegative real values with the summation replaced by the maximum operator. Subtropical matrix factorizations allow ``winner-takes-it-all'' interpretations of the rank-1 components, revealing different structure than the normal (nonnegative) factorizations. We study the complexity and sparsity of the factorizations, and present a framework for finding low-rank subtropical factorizations. We present two specific algorithms, called Capricorn and Cancer, that are part of our framework. They can be used with data that has been corrupted with different types of noise, and with different error metrics, including the sum-of-absolute differences, Frobenius norm, and Jensen--Shannon divergence. Our experiments show that the algorithms perform well on data that has subtropical structure, and that they can find factorizations that are both sparse and easy to interpret.


\end{abstract}

\section{Introduction}
\label{sec:introduction}

Finding simple patterns that can be used to describe the data is one of the main problems in data mining. The data mining literature knows many different techniques for this general task, but one of the most common pattern finding technique rarely gets classified as such. Matrix factorizations (or decompositions, these two terms are used interchangeably in this paper) represent the given input matrix $\mA$ as a product of two (or more) factor matrices, $\mA\approx \mB\mC$. This standard formulation of matrix factorizations makes their pattern mining nature less obvious, but let us write the matrix product $\mB\mC$ as a sum of rank-1 matrices, $\mB\mC = \mF_1 + \mF_2 + \cdots + \mF_k$, where $\mF_i$ is the outer product of the $i$th column of $\mB$ and the $i$th row of $\mC$. Now it becomes clear that the rank-1 matrices $\mF_i$ are the ``simple patterns'' and the matrix factorization is finding $k$ such patterns whose sum is a good approximation of the original data matrix.

This so-called ``component interpretation''~\citep{skillicorn07understanding} is more appealing with some factorizations than with others. For example, the classical singular value decomposition (SVD) does not easily admit such an interpretation, as the components are not easy to interpret without knowing the earlier components. On the other hand, the motivation for the nonnegative matrix factorization (NMF) often comes from the component interpretation, as can be seen, for example, in the famous ``parts of faces'' figures of \citet{lee_seung}.
The ``parts-of-whole'' interpretation is in the hearth of NMF: every rank-1 component adds something to the overall decomposition, and never removes anything. This aids with the interpretation of the components, and is also often claimed to yield sparse factors, although this latter point is more contentious \citep{hoyer04non-negative}. 

Perhaps the reason why matrix factorization methods are not often considered as pattern mining methods is that the rank-1 matrices are summed together to build the full data. Hence, it is rare for any rank-1 component to explain any part of the input matrix alone. But the use of summation as a way to aggregate the rank-1 components can be considered to be ``merely'' a consequence of the fact that we are using the standard algebra. If we change the algebra -- in particular, if we change how we define the summation -- we change the operator used for the aggregation. In this work, we propose to use the \emph{maximum} operator to define the summation over the nonnegative matrices, giving us what is known as the \emph{subtropical algebra}. As the aggregation of the rank-1 factors is now the element-wise maximum, we obtain what we call the ``winner-takes-it-all'' interpretation: the final value of each element in the approximation is defined only by the largest value in the corresponding element in the rank-1 matrices.

Not only does the subtropical algebra give us the intriguing winner-takes-it-all interpretation, it also provides guarantees about the sparsity of the factors, as we will show in Section~\ref{sec:sparsity}. Furthermore, the different algebra means that we are finding different factorizations compared to NMF (or SVD). The emphasis here is on the word \emph{different}: the factorizations can be better or worse in terms of the reconstruction error -- we will discuss this in Section~\ref{sec:other_alg} -- but the patterns they find are usually different to those found by NMF. Unfortunately, the related optimization problems are \NP-hard (see Section~\ref{sec:comput_complex}). In Section~\ref{sec:algorithms}, we will develop a general framework, called \Equator, for finding approximate, low-rank subtropical decompositions, and we will present two instances of this framework, tailored towards different types of data and noise, called \Capricorn and \Cancer.\!\footnote{This work is a combined and extended version of our preliminary papers that described these algorithms \citep{karaev16cancer,karaev16capricorn}.} \Capricorn assumes integer data with noise that randomly flips the value to some other integer, whereas \Cancer assumes continuous-valued data with standard Gaussian noise. 

Our experiments (see Section~\ref{sec:experiments}) show that both \Capricorn and \Cancer work well on datasets that have the kind of noise they are designed for, and they outperform SVD and different NMF methods when data has subtropical structure. On real-world data, \Cancer is usually the better of the two, although in terms of reconstruction error, neither of the methods can challenge SVD. On the other hand, both \Cancer and \Capricorn return interpretable results that show different aspects of the data compared to factorizations made under the standard algebra.


\section{Notation and Basic Definitions}
\label{sec:notation}

\paragraph{Basic notation.}
Throughout this paper,  we will denote a matrix by upper-case boldface letters ($\matr{A}$), and vectors by lower-case boldface letters ($\vec{a}$). The $i$th row of matrix $\matr{A}$ is denoted by $\matr{A}_{i}$ and the $j$th column by $\matr{A}^{j}$. The matrix $\mA$ with the $i$th column removed is denoted by $\mA^{-i}$, and $\mA_{-i}$ is the respective notation for $\mA$ with a removed row. 
Most matrices and vectors in this paper are restricted to the nonnegative real numbers $\Region = [0,\infty)$.

We use the shorthand $[n]$ to denote the set $\{1, 2, \ldots, n\}$.





\paragraph{Algebras.}
In this paper we consider matrix factorization over so called \emph{max-times} (or \emph{subtropical}) \emph{algebra}. It differs from the standard algebra of real numbers in that addition is replaced with the operation of taking the maximum. Also the domain is restricted to the set of nonnegative real numbers. 

\begin{definition} 
  \label{def:max-times}
  The \emph{max-times} (or \emph{subtropical}) algebra is a set $\Region$ of nonnegative real numbers together with operations $a \maxadd b = \max \lbrace a, b\rbrace$ (addition) and $a \maxmult b = ab$ (multiplication) defined for any $a, b \in \Region$. The identity element for addition is $0$ and for multiplication it is $1$.
\end{definition}
In the future we will use the notation $a\maxadd b$ and $\max \lbrace a, b\rbrace$  and the names \emph{max-times} and \emph{subtropical} interchangeably. It is straightforward to see that the max-times algebra is a \emph{dioid}, that is, a semiring with idempotent addition ($a \maxadd a = a$). It is important to note that subtropical algebra is anti-negative, that is, there is no subtraction operation.

A very closely related algebraic structure is the \emph{max-plus} (\emph{tropical}) algebra \citep[see e.g.][]{akian07max-plus}. 
\begin{definition}
  \label{def:max-plus}
  The \emph{max-plus} (or \emph{tropical}) algebra is defined over the set of extended real numbers $\R \cup \{-\infty\}$ with operations $a \tropadd b = \max \lbrace a, b\rbrace$ (addition) and $a \tropmult b = a+b$ (multiplication). The identity elements for addition and multiplication are $-\infty$ and $0$, respectively.
\end{definition}

The tropical and subtropical algebras are isomorphic \citep{blondel2000approximating}, which can be seen by taking the logarithm of the subtropical algebra or the exponent of the tropical algebra (with the conventions that $\log 0 = -\infty$ and $\exp(-\infty) = 0$).
Thus, most of the results we prove for subtropical algebra can be extended to their tropical analogues, although caution should be used when dealing with approximate matrix factorizations. The latter is because, as we will see in Theorem~\ref{thm:max_plus_bound}, the \emph{reconstruction error} of an approximate matrix factorization under the two different algebras does not transfer directly. 


\paragraph{Matrix products and ranks.}
The matrix product over the subtropical algebra is defined in the natural way:

\begin{definition} \label{def:mprod}
  The \emph{max-times matrix product} of  two matrices $\matr{B} \in \Region^{n \times k}$ and $\matr{C} \in \Region^{k \times m}$ is defined as
  \begin{equation}
    \label{eq:mprod}
(\matr{B} \maxprod \matr{C})_{ij} = \max_{s = 1}^k \matr{B}_{is} \matr{C}_{sj} \;.
  \end{equation}
\end{definition}

We will also need the matrix product over the \emph{tropical} algebra.

\begin{definition} \label{def:tropical:mprod}
  For two matrices $\matr{B} \in (\R \cup \{-\infty\})^{n \times k}$ and $\matr{C} \in (\R \cup \{-\infty\})^{k \times m}$, their \emph{tropical matrix product}  is defined as
  \begin{equation}
    \label{eq:mprod}
(\matr{B} \tropprod \matr{C})_{ij} = \max_{s = 1}^k \lbrace \matr{B}_{is} + \matr{C}_{sj}\rbrace \;.
  \end{equation}
\end{definition}

The \emph{matrix rank} over the subtropical algebra can be defined in many ways, depending on which definition of the normal matrix rank is taken as the starting point. We will discuss different subtropical ranks in detail in Section~\ref{sec:subtropical_ranks}. Here we give the main definition of the rank we are using throughout this paper, the so-called \emph{Schein} (or \emph{Barvinok}) \emph{rank} of a matrix.

\begin{definition}
  \label{def:mrank}
  The \emph{max-times (Schein or Barvinok) rank} of a matrix $\matr{A}\in\R_+^{n\times m}$ is the least integer $k$ such that $\matr{A}$ can be expressed as an element-wise maximum of $k$ rank-1 matrices, $\matr{A} = \matr{F}_1 \maxadd \matr{F}_2 \maxadd\cdots\maxadd\matr{F}_k$. Matrix $\mF\in \Region^{n\times m}$ has subtropical (Schein/Barvinok) rank of 1 if there exist column vectors $\vx\in\Region^n$ and $\vy\in\Region^m$ such that $\mF = \vx\vy^T$. Matrices with subtropical Schein (or Barvinok) rank of 1 are called \emph{blocks}.
\end{definition}

When it is clear from the context, we will use the term \emph{rank} (or \emph{subtropical rank}) without other qualifiers to denote the subtropical Schein/Barvinok rank.

\paragraph{Special matrices.}
The final concepts we need in this paper are \emph{pattern matrices} and \emph{dominating matrices}.

\begin{definition}
  \label{def:pattern}
  A \emph{pattern} of a matrix $\matr{A}\in\R^{n\times m}$ is an \by{n}{m} binary matrix $\matr{P}$ such that $\matr{P}_{ij} = 0$ if and only if $\matr{A}_{ij} = 0$, and otherwise $\matr{P}_{ij} = 1$.  We denote the pattern of $\mA$ by $\pattern(\mA)$.
\end{definition}

\begin{definition} \label{def:dominate}
  Let $\matr{A}$ and $\matr{X}$ be matrices of the same size, and let $\Gamma$ be a subset of their indices.  Then if for all indices $(i, j) \in \Gamma$, $\matr{X}_{ij} \ge \matr{A}_{ij}$, we say that \emph{$\matr{X}$ dominates $\matr{A}$ within  $\Gamma$}. If $\Gamma$ spans the entire size of $\matr{A}$ and $\matr{X}$, we simply say that $\matr{X}$ \emph{dominates} $\matr{A}$. Correspondingly, $\matr{A}$ is said to be \emph{dominated by} $\matr{X}$.
  \end{definition}

  \paragraph{Main problem definition.}
Now that we have sufficient notation, we can formally introduce the main problem considered in the paper.
\begin{problem}[Approximate subtropical rank-$k$ matrix factorization]
  \label{problem:mdecomp}
  Given a matrix $\matr{A} \in \Region^{n \times m}$ and an integer $k>0$, find factor matrices $\matr{B} \in \Region^{n \times k}$ and $\matr{C} \in \Region^{k \times m}$ minimizing
\begin{equation}
  \label{eq:mdecomp}
  E(\matr{A}, \matr{B}, \matr{C}) = 
  \norm{\matr{A} - \matr{B} \maxprod \matr{C}}\;.
\end{equation}
\end{problem}
Here we have deliberately not specified any particular norm. Depending on the circumstanses, different matrix norms can be used, but in this paper we will consider the two most natural choices -- the Frobenius and $L_1$ norms.


\section{Theory}
\label{sec:theory}

Our main contributions in this paper are the algorithms for the subtropical matrix factorization. But before we present them, it is important to understand the theoretical aspects of subtropical factorizations. We will start by studying the computational complexity of Problem~\ref{problem:mdecomp}. After that, we will show that the dominated subtropical factorizations of sparse matrices are sparse. Finally, we compare the subtropical factorizations to factorizations over other algebras, and discuss different ways to define the subtropical rank, and the relationships between these ranks. 

\subsection{Computational complexity}
\label{sec:comput_complex}

The computational complexity of different matrix factorization problems varies. For example, SVD can be computed in polynomial time \citep{golub}, while NMF is \NP-hard \citep{vavasis09complexity}. Unfortunately, the subtropical factorization is also \NP-hard.

\begin{theorem}\label{thm:npcomplete}
Computing the max-times matrix rank is an \NP-hard problem, even for binary matrices.
\end{theorem}

The theorem is a direct consequence of the following theorem by \citet{kim2005factorization}:
\begin{theorem}[\citealp{kim2005factorization}]
  \label{thm:trop_rank_nphard}
  Computing the max-plus (tropical) matrix rank is \NP-hard, even for matrices that take values only from $\{-\infty, 0\}$. 
\end{theorem}

While computing the rank deals with exact decompositions, its hardness automatically makes any approximation algorithm with provable multiplicative guarantees unlikely to exist, as the following corollary shows.

\begin{corollary}
\label{corollary:approxnphard}
It is \NP-hard to approximate Problem~\ref{problem:mdecomp} to within any polynomially computable factor.
\end{corollary}
\begin{proof}
Any algorithm that can approximate Problem~\ref{problem:mdecomp} to within a factor $\alpha$ must find a decomposition of error $\alpha\cdot 0 = 0$ if the input matrix has exact max-times rank-$k$ decomposition. As this implies solving the max-times rank, per Theorem~\ref{thm:npcomplete} it is only possible if \Poly=\NP. 
\end{proof}

\subsection{Sparsity of the factors}
\label{sec:sparsity}

It is often desirable to obtain sparse factor matrices if the original data is sparse, as well, and the sparsity of its factors is frequently mentioned as one of the benefits of using NMF~\citep[see, e.g.][]{hoyer04non-negative}. In general, however, the factors obtained by NMF might not be sparse, but if we restrict ourselves to \emph{dominated} decompositions, \citet{gillis10using} showed that the sparsity of the factors cannot be less than the sparsity of the original matrix. 

The proof of \citet{gillis10using} relies on the anti-negativity, and hence their proof is easy to adapt to max-times setting. Let the \emph{sparsity} of an \by{n}{m} matrix $\matr{A}$, $s(\matr{A})$, be defined as
\begin{equation}\label{fracnonzero}
s(\matr{A}) = \frac{nm - \nnz{\matr{A}}}{nm}\;,
\end{equation}
where $\nnz{\matr{A}}$ is the number of nonzero elements in $\matr{A}$. Now we have

\begin{theorem}\label{thm:sparsity} Let matrices $\matr{B} \in \Region^{n \times k}$ and $\matr{C} \in \Region^{k \times m}$ be such that their max-times product is dominated by an \by{n}{m} matrix $\matr{A}$. Then the following estimate holds
\begin{equation}\label{nzeros}
s(\matr{B}) + s(\matr{C} ) \ge s(\matr{A}) \;.
\end{equation}
\end{theorem}

\begin{proof}
  The proof follows that of \citet{gillis10using}.
We first prove \eqref{nzeros} for $k = 1$. Let $\vec{b} \in \Region^n$ and $\vec{c} \in \Region^m$ be such that $\vec{b}_i \vec{c}^T_j \le \matr{A}_{ij}$ for all $1 \le i \le n$, $ 1 \le j \le m$. Since $(\vec{b}\vec{c}^T)_{ij}>0$ if and only if $\vec{b}_i > 0$ and $\vec{c}_j > 0$, we have
\begin{equation}\label{bwnnonzero}
\nnz{\vec{b}\vec{c}^T} = \nnz{\vec{b}}\,\nnz{\vec{c}}\;.
\end{equation}  
By \eqref{fracnonzero} we have $\nnz{\vec{b}\vec{c}^T} = nm(1-s(\vec{b}\vec{c}^T))$, $\nnz{\vec{b}} = n(1-s(\vec{b}))$ and $\nnz{\vec{c}} = m(1-s(\vec{c}))$. Plugging these expressions into \eqref{bwnnonzero} we obtain $(1 - s(\vec{b}\vec{c}^T)) = (1-s(\vec{b}))(1-s(\vec{c}))$. Hence, the sparsity in a rank-$1$ dominated approximation of $\matr{A}$ is
\begin{equation}\label{intermediate1}
s(\vec{b}) + s(\vec{c}) \ge s(\vec{b}\vec{c}^T)\;.
\end{equation}
From \eqref{intermediate1} and the fact that the number of nonzero elements in $\vec{b}\vec{c}^T$ is no greater than in $\matr{A}$, it follows that
\begin{equation}\label{inermediate2}
s(\vec{b}) + s(\vec{c}) \ge s(\matr{A}) \;.
\end{equation}
Now let $\matr{B} \in \Region^{n \times k}$ and $\matr{C} \in \Region^{k \times m}$ be such that $\matr{B}\maxprod \matr{C}$ is dominated by $\matr{A}$. Then $\matr{B}_{il} \matr{C}_{lj} \le \matr{A}_{ij}$ for all $i \in [n]$, $j \in [m]$, and $l \in [k]$, which means that for each $l\in [k]$,  $\matr{B}^l\matr{C}_l\,$ is dominated by $\matr{A}$. To complete the proof observe that $s(\matr{B}) = k^{-1} \sum_{l = 1}^k \matr{B}^l$ and $s(\matr{C}) = k^{-1} \sum_{l=1}^k \matr{C}_l$ and that for each $l$ estimate \eqref{inermediate2} holds. 
\end{proof}

\subsection{Relation to other algebras}
\label{sec:other_alg}

Let us now study how the max-times algebra relates to other algebras, especially the standard, the Boolean, and the max-plus algebras. For the first two, we compare the ranks, and for the last, the reconstruction error.

Let us start by considering the Boolean rank of a binary matrix. The \emph{Boolean (Schein or Barvinok) rank} is the following problem:
\begin{problem}[Boolean rank]
  Given a matrix $\matr{A}\in\B^{n\times m}$ and an integer $k$, are there matrices $\matr{B}\in\B^{n\times k}$ and $\matr{C}\in\B^{k\times m}$ such that $\matr{A} = \matr{B}\bprod\matr{C}$, where $\bprod$ is the \emph{Boolean matrix product},
\[
(\matr{B}\bprod\matr{C})_{ij} = \bigvee_{l=1}^k \matr{B}_{il}\matr{C}_{lj}\; .
\]
\end{problem}

\begin{lemma}
  \label{lemma:brank_vs_strank}
  If $\mA$ is a binary matrix, then its Boolean and subtropical ranks are the same.
\end{lemma}

\begin{proof}
  We will prove the claim by first showing that the Boolean rank of a binary matrix is no less than the subtropical rank, and then showing that it is no larger, either.
  For the first direction, let the Boolean rank of $\mA$ be $k$, and let $\mB$ and $\mC$ be binary matrices such that $\mB$ has $k$ columns and $\mA = \mB\bprod\mC$. It is easy to see that $\mB\bprod\mC = \mB\maxprod\mC$, and hence, the subtropical rank of $\mA$ is no more than $k$.

  For the second direction, we will actually show a slightly stronger claim: Let $\mA\in \Region^{n \times m}$ and let $\pattern(\mA)$ be its pattern. Then the Boolean rank of $\pattern(\mA)$ is never more than the subtropical rank of $\mA$. As $\pattern(\mA) = \mA$ for a binary $\mA$, the claim follows. To prove the claim, let $\mA\in\Region^{n\times m}$ have subtropical rank of $k$ and let $\mB\in\Region^{n\times k}$ and $\mC\in\Region^{k\times m}$ be such that $\mA = \mB\maxprod\mC$. Let $(i,j)$ be such that $\mA_{ij} = 0$. By definition, $\max_{l=1}^k \mB_{il}\mC_{lj} = 0$, and hence
  \begin{equation}
    \label{eq:pattern_vals:0}
    \max_{l=1}^k \pattern(\mB)_{il}\pattern(\mC)_{lj} = \bigvee_{l=1}^k \pattern(\mB)_{il}\pattern(\mC)_{lj} = 0\; .
  \end{equation}
  On the other hand, if $(i,j)$ is such that $\mA_{ij} > 0$, then there exists $l$ such that $\mB_{il}, \mC_{lj}  > 0$ and consequently,
  \begin{equation}
    \label{eq:pattern_vals:1}
    \max_{l=1}^k \pattern(\mB)_{il}\pattern(\mC)_{lj} = \bigvee_{l=1}^k \pattern(\mB)_{il}\pattern(\mC)_{lj} = 1\; .
  \end{equation}

  Combining~\eqref{eq:pattern_vals:0} and \eqref{eq:pattern_vals:1} gives us
  \begin{equation}
    \label{eq:pattern_bprod}
    \pattern(\mA) = \pattern(\mB)\bprod\pattern(\mC)\; ,
  \end{equation}
  showing that the Boolean rank of $\pattern(\mA)$ is at most $k$.
\end{proof}

Notice that Lemma~\ref{lemma:brank_vs_strank} also furnishes us with another proof of Theorem~\ref{thm:npcomplete}, as the computation of the Boolean rank is an \NP-complete problem \citep[see, e.g.][]{miettinen09matrix}. Notice also that while the Boolean rank of the pattern is never more than the subtropical rank of the original matrix, it can be much less. This is easy to see by considering a matrix with no zeroes: it can have arbitrarily large subtropical rank, but it's pattern has Boolean rank 1.

Unfortunately, the Boolean rank does not help us with effectively estimating the subtropical rank, as its computation is an \NP-hard problem. The standard rank is (relatively) easy to compute, but the standard rank and the max-times rank are incommensurable, that is, there are matrices that have smaller max-times rank than standard rank and others that have higher max-times rank than standard rank. Let us consider an example of the first kind,
\[
\begin{pmatrix} 
  1 & 2 & 0 \\
  2 & 4 & 1 \\
  0 & 4 & 2
\end{pmatrix}
=
\begin{pmatrix}
  1 & 0 \\
  2 & 1 \\
  0 & 2
\end{pmatrix}
\maxprod
\begin{pmatrix}
  1 & 2 & 0 \\
  0 & 2 & 1
\end{pmatrix}\; .
\] 
As the decomposition shows, this matrix has max-times rank of $2$, while its normal rank is easily verified to be $3$. Indeed, it is easy to see that the complement of the \by{n}{n} identity matrix $\bar{\matr{I}}_n$, that is, the matrix that has $0$s at the diagonal and $1$s everywhere else, has max-times rank of $O(\log n)$ while its standard rank is $n$ (the result follows from similar results regarding the Boolean rank, see, e.g. \citealp{miettinen09matrix}). 

As we have discussed earlier, max-plus and max-times algebras are isomorphic, and consequently for any matrix $\mA \in \mathbb{R}_+^{n\times m}$ its max-times rank agrees with the max-plus rank of the matrix $\log(\mA)$. Yet, the errors obtained in approximate decompositions do not have to (and usually will not) agree. In what follows we characterize the relationship between max-plus and max-times errors. We denote by $\tropalg$ the extended real line $\R \cup \lbrace -\infty\rbrace$.

\begin{theorem}\label{thm:max_plus_bound}
  Let $\matr{A} \in \overline{\R}^{n \times m}$,   $\matr{B} \in \overline{\R}^{n \times k}$
and $\matr{C} \in \overline{\mathbb{R}}^{k \times m}$. Let $M = \exp\{N\}$, where
\[ 
N =  \max_{\substack{i\in [n]\\  j \in [m]}} \Bigl\{  \max \bigl\{ \matr{A}_{ij}, \max_{1 \le d \le k} \{ \matr{B}_{id} + \matr{C}_{dj} \} \bigr\} \Bigr\} \; .
\]

If an error can be bounded in max-plus algebra as
\begin{equation}\label{max_plus_bound}
\norm{\matr{A} - \matr{B} \tropprod \matr{C}}_F^2 \le \lambda\; ,
\end{equation}
then the following estimate holds with respect to the max-times algebra:
\begin{equation}\label{max_times_bound}
\norm{\exp\{\matr{A}\} - \exp\{\matr{B}\} \maxprod \exp\{\matr{C}\}}_F^2 \le M^2 \lambda\; .
\end{equation}
\end{theorem}

\begin{proof}
Let $\alpha_{ij} = \max_{d=1}^{ k} \lbrace \matr{B}_{id} + \matr{C}_{dj}\rbrace$.
From \eqref{max_plus_bound} it follows that there exists a set of numbers $\lbrace\lambda_{ij} \ge 0 \setcond i \in [n], j \in [m] \rbrace$ such that for any $i, j$ we have $(A_{ij} - \alpha_{ij})^2 \le \lambda_{ij}$ and $\sum_{ij} \lambda_{ij} = \lambda$. By the mean-value theorem, for every $i$ and $j$ we obtain
\[
  \abs{\exp\{\matr{A}_{ij}\} - \exp\{\alpha_{ij}\} } = \abs{ \matr{A}_{ij} - \alpha_{ij} }  \exp\{ \alpha_{ij}^*\} 
\le  \sqrt{\lambda_{ij}}  \exp\{\alpha_{ij}^*\}\; ,
\]
for some $\min\lbrace \matr{A}_{ij}, \alpha_{ij} \rbrace \le \alpha_{ij}^* \le \max\lbrace \matr{A}_{ij}, \alpha_{ij} \rbrace$. Hence,
\[
\left(\exp\{\matr{A}_{ij}\} - \exp\{\alpha_{ij}\}\right)^2 \le \lambda_{ij} (\exp \{\max \lbrace \matr{A}_{ij}, \alpha_{ij} \rbrace\})^2 \; .
\]
The estimate for the max-times error now follows from the monotonicity of the exponent:
\[
\begin{split}
\norm{\exp\{\matr{A}\} - \exp\{\matr{B}\} \maxprod \exp\{\matr{C}\}}_F^2
&\le \sum_{ij} \left(\exp\{\alpha_{ij}^*\}\right)^2 \lambda_{ij} \\
&\le \sum_{ij} \left(\exp\{\max\lbrace \matr{A}_{ij}, \alpha_{ij} \rbrace\}\right)^2 \lambda_{ij} 
\le M^2 \lambda\; ,
\end{split}
\] 
proving the claim. 
\end{proof}

\subsection{Different subtropical matrix ranks}
\label{sec:subtropical_ranks}

The definition of the subtropical rank we use in this work is the so-called Schein (or Barvinok) rank (see Definition~\ref{def:mrank}). Like in the standard linear algebra, this is not the only possible way to define the (subtropical) rank. Here we will review few other forms of subtropical rank that can allow us to bound the Schein/Barvinok rank of a matrix. Following the literature, we will present the definitions in this section over the tropical algebra. Recall that due to isomorphism, these definitions transfer directly to the subtropical case. Unless otherwise mentioned, the definitions are by \citet{guillon2015ultimate}; we refer the readers interested in more details to their work.

We begin with the tropical equivalent of the subtropical Schein/Barvinok rank:
\begin{definition}
  \label{def:schein_barvinok_rank}
  The \emph{tropical Schein/Barvinok rank} of a matrix $\mA\in\tropalg^{n\times m}$, denoted $\rankSB(\mA)$, is defined to be the least integer $k$ such that there exist matrices $\mB\in\tropalg^{n\times k}$ and $\mC\in\tropalg^{k\times m}$ for which $\mA = \mB\tropprod\mC$. 
\end{definition}

Analogous to the standard case, we can also define the rank as the number of linearly independent rows or columns.
The following definition of linear independence of a family of vectors in a tropical space is due to \citet{gondran1984linear}.
\begin{definition} 
\label{def:gondran:minoux:linear}
  A set of vectors $\vec{x}_1, \dots, \vec{x}_k$ from $\tropalg^n$ is called \emph{linearly dependent} if there exist disjoint sets $I, J \subset \{1,\dots, k\}$ and scalars $\{\lambda_i\}_{i\in I \cup J}$, such that $\lambda_i \ne -\infty$ for all $i$ and
  \begin{equation} \label{lin_depend}
\max_{i\in I} \{\lambda_i + \vec{x}_i\} = \max_{j\in J} \{\lambda_j + \vec{x}_j\} \;.
\end{equation}
Otherwise the vectors $\vec{x}_1, \dots, \vec{x}_k$ are called \emph{linearly independent}.
\end{definition}

This gives rise to the so-called \emph{Gondran--Minoux ranks}: 
\begin{definition}
      \label{def:gondran:minoux:rank}
  The \emph{Gondran--Minoux} row (column) rank of a matrix $\matr{A} \in \tropalg^{n\times m}$ is defined as the maximal $k$ such that $\matr{A}$ has $k$ independent rows (columns). They are denoted by $\rankGMr(\matr{A})$  and $\rankGMc(\matr{A})$  respectively.
\end{definition}

Another way to characterize the rank of the matrix is to consider the space its rows or columns can span.

\begin{definition}
  \label{def:convex} 
  A set $X \subset \tropalg^n$  is called \emph{tropically convex} if for any vectors $\vec{x}, \vec{y} \in X$ and scalars $\lambda, \mu \in \tropalg$, we have $\max\{\lambda + \vec{x}, \mu + \vec{y}\} \in X$.
\end{definition}

\begin{definition} 
  \label{def:convex:hull}
  The \emph{convex hull} $H(\vec{x}_1, \dots \vec{x}_k)$ of a finite set of vectors $\{\vec{x}_i  \}_{i=1}^k \in \tropalg^n$ is defined as follows 
  \[
H(\vec{x}_1, \dots \vec{x}_k) = \left\{ \max_{i=1}^k \{\lambda_i + \vec{x}_i\} \setcond \lambda_i \in \tropalg \right\} \;.
  \]
\end{definition}

\begin{definition}
  \label{def:weak:dim} 
  The \emph{weak dimension} of a finitely generated tropically convex subset of $\tropalg^n$ is the cardinality of its minimal generating set. 
\end{definition}

We can define the rank of the matrix by looking at the weak dimension of the (tropically) convex hull its rows or columns span.

\begin{definition}
  \label{def:row:rank}
The \emph{row rank} and the \emph{column rank} of a matrix $\matr{A} \in \tropalg^{n\times m}$ are defined as the weak dimensions of the convex hulls of the rows and the columns of $\matr{A}$ respectively. They are denoted by $\rankRW(\matr{A})$  and $\rankCL(\matr{A})$.
\end{definition}

None of the above definitions coincide \citep[see][]{akian2009linear}, unlike in the standard algebra. We can, however, have a partial ordering of the ranks:

\begin{theorem} \label{thm:rank:relations}
  \citep{guillon2015ultimate, akian2009linear}
  Let $\matr{A} \in \tropalg^{n\times m}$. Then the the following relations are true for the above definitions of the rank of $\matr{A}$:
  \begin{equation} \label{rank:bounds}
    \left .
      \begin{aligned}
        &\rankGMr(\mA)\\
        &\rankGMc(\mA)
      \end{aligned}\right\}
    \le \rankSB(\mA) \le
    \left\{\!\begin{aligned}
        &\rankRW(\mA)\\
        &\rankCL(\mA)
      \end{aligned}\right . \;.
  \end{equation}
\end{theorem}


The row and column ranks of an \by{n}{n} tropical matrix can be computed in $O(n^3)$ time \citep{butkovivc2010max}, allowing us to bound the Schein/Barvinok rank from above. Unfortunately, no efficient algorithm for the Gondran--Minoux rank is known. On the other hand, \citet{guillon2015ultimate} presented what they called the \emph{ultimate tropical rank} that lower-bounds the Gondran--Minoux rank and can be computed in time $O(n^3)$. We can also check if a matrix has full Schein/Barvinok rank in time $O(n^3)$ \citep[see][]{butkovivc1985condition}, even if computing any other value is \NP-hard.

These bounds, together with Lemma~\ref{lemma:brank_vs_strank} yield the following corollary regarding the bounding of the \emph{Boolean rank} of a square matrix:

\begin{corollary}
  \label{corollary:brank_bounds}
  Given an \by{n}{n} binary matrix $\mA$, it's Boolean rank can be bound from below, using the ultimate rank, and from above, using the tropical column and row ranks, in time $O(n^3)$.
\end{corollary}


\section{Algorithms}
\label{sec:algorithms}
The problem of subtropical matrix factorization has some unique challenges that stem from the lack of linearity and smoothness of the max-times algebra. One of such issues is that dominated elements in a decomposition have no impact on the final result. Namely, if we consider the subtropical product of two matrices $\matr{B}\in \Region^{n\times k}$ and $\matr{C}\in \Region^{k\times m}$, we can see that each entry $(\matr{B} \maxprod \matr{C})_{ij} = \max_{1 \le s \le k}{\matr{B}_{is} \matr{C}_{sj}}$ is completely determined by a single element with index $\argmax_{1 \le s \le k}{\matr{B}_{is} \matr{C}_{sj}}$. This means that all entries $t$ with $\matr{B}_{it} \matr{C}_{tj}<\max_{1 \le s \le k}{\matr{B}_{is} \matr{C}_{sj}}$ do not contribute at all to the final decomposition. To see why this is a problem, observe that many optimization methods used in matrix factorization algorithms rely on local information to choose the direction of the next step (e.g. various forms of gradient descent). In the case of the subtropical algebra, however, the local information is practically absent, and hence we need to look elsewhere for effective optimization techniques.

A common approach to matrix decomposition problems is to update factor matrices alternatingly, which utilizes the fact that the problem $\min_{\matr{B}, \matr{C}}{\norm{\matr{A} - \matr{B}\matr{C}}}_F$ is biconvex. Unfortunately, the subtropical matrix factorization problem does not have the biconvexity property, which makes alternating updates less useful.

Here we present a different approach that, instead of doing alternating factor updates, constructs the decomposition by adding one rank-1 matrix at a time, following the idea by \cite{kolda2000}. The corresponding algorithm is called \Equator (Algorithm~\ref{alg:equator}).

First observe that the max-times product can be represented as an elementwise maximum of rank-1 matrices (blocks)
\begin{equation}\label{blockwise}
\matr{B} \maxprod \matr{C} = \max\limits_{1\le s \le k}{\matr{B}^s \matr{C}_s}\;.
\end{equation}
Hence, Problem~\ref{problem:mdecomp} can be split into $k$ subproblems of the following form: given a rank-$(l-1)$ decomposition $\matr{B}\in \Region^{n\times (l-1)}$, $\matr{C}\in \Region^{(l-1)\times m}$ of a matrix $\matr{A} \in \Region^{n \times m}$, find a column vector $\vec{b}\in \Region^{n\times 1}$ and a row vector $\vec{c} \in \Region^{1\times m}$ such that the error
\begin{equation}\label{maxtimesrank1}
\norm{  {\matr{A} - \max{\lbrace \matr{B}\maxprod \matr{C}}, \vec{b} \vec{c} \rbrace} }
\end{equation}
is minimized. We assume by definition that the rank-$0$ decomposition is an all zero matrix of the same size as $\matr{A}$. The problem of rank-$k$ subtropical matrix factorization is then reduced to solving \eqref{maxtimesrank1} $k$ times. One should of course remember that this scheme is just a heuristic and finding optimal blocks on each iteration does not guarantee converging to a global minimum.

One prominent issue with the above approach is that an optimal rank-$(k-1)$ decomposition might not be very good when considered as a part of a rank-$k$ decomposition. This is because for smaller ranks we generally have to cover the data more crudely, whereas when the rank increases we can afford to use smaller and more refined blocks. In order to deal with this problem, we find and then update the blocks repeatedly, in a cyclic fashion. That means that after discovering the last block, we go all the way back to block one. The input parameter $M$ defines the number of full cycles we make. 

\begin{algorithm}[tbp]
  \flushleft
  \caption{\Equator}\label{alg:equator}
  \begin{algorithmic}[1]
    \Input $\matr{A} \in \Region^{n \times m}$, $k>0$, $M>0$
    \Output $\Bbest \in \Region^{n \times k}$, $\Cbest \in \Region^{k \times m}$
    \Function{\Equator}{\matr{A}, k, M}
      \State $\matr{B} \gets 0^{n \times k}$, $\matr{C} \gets 0^{k \times m}$  \label{init}
      \State $\Bbest \gets \matr{B}, \Cbest \gets \matr{C}$  \label{best:factor:init}
      \State $\bestError \gets E(\matr{A}, \matr{B}, \matr{C})$ \label{error:init}
      \For{$\mathit{count} \gets 1$ \textbf{to} $k\times M$} \label{begincyclic}
          \State $l \gets (\mathit{count}-1) \pmod k + 1$                     \Comment{Index of the current block} \label{getcurrentindexequator}
        \State $[\matr{B}^l, \matr{C}_l] \gets \UpdateBlock(\matr{A}, \matr{B}, \matr{C}, count)$ \label{Cancer:UpdateBlock}
        \If{$E(\matr{A}, \matr{B}, \matr{C}) < \bestError$} \label{compbegin}
          \State $\Bbest \gets \matr{B}, \Cbest \gets \matr{C}$
          \State $\bestError \gets E(\matr{A}, \matr{B}, \matr{C})$ \label{compend}
          \EndIf
      \EndFor \label{finishloop} 
      \State \textbf{return} $\Bbest$,  $\Cbest$
    \EndFunction
  \end{algorithmic}
\end{algorithm}



On a high level \Equator works as folows. First the factor matrices are initialized to all zeros (line~\ref{init}). Since the algorithm makes iterative changes to the current solutions that might in some cases lead to worsening of the results, it also stores the best reconstruction error and the corresponding factors found so far. They are initalized with the starting solution on lines~\ref{best:factor:init}--\ref{error:init}. The main work is done in the loop on lines~\ref{begincyclic}--\ref{finishloop}, where on each iteration we update a single rank-1 matrix in the current decomposition using the \UpdateBlock routine (line~\ref{Cancer:UpdateBlock}), and then check if the update improves the best result (lines~\ref{compbegin}--\ref{compend}).

We will present two versions of the \UpdateBlock function, one called \Capricorn and the other one \Cancer. \Capricorn is designed to work with discrete (or flipping) noise, when some of the elements in the data are randomly changed to different values. In this setting the level of noise is the proportion of the flipped elements relative to the total number of nonzeros. \Cancer on the other hand is robust with continuous noise, when many elements are affected (e.g. Gaussian noise). We will discuss both of them in detail in the following subsections. In the rest of the paper, especially when presenting the experiments, we will use names \Capricorn and \Cancer not only for a specific variation of the \UpdateBlock function, but also for the \Equator algorithm that uses it.

\subsection{\Capricorn}
We first describe \Capricorn, which is designed to solve the subtropical matrix factorization problem in the presence of discrete noise, and minimizes the $L1$ norm of the error matrix. 
The main idea behind the algorithm is to spot potential blocks by considering ratios of matrix rows. Consider an arbitrary rank-1 block $\matr{X} = \vec{b} \vec{c}$, where $\vec{b} \in \Region^{n \times 1}$ and $\vec{c} \in \Region^{1 \times m}$. For any indices $i$ and $j$ such that $\vec{b}_i>0$ and $\vec{b}_j>0$, we have $\matr{X}_j = \frac{\vec{b}_j}{\vec{b}_i} \matr{X}_i$. This is a characteristic property of rank-1 matrices -- all rows are multiples of one another. Hence, if a block $\matr{X}$ dominates some region $\Gamma$ of a matrix $\matr{A}$, then rows of $\matr{A}$ should all be multiples of each other within $\Gamma$. These rows might have different lengths due to block overlap, in which case the rule only applies to their common part.

\UpdateBlock starts by identifying the index of the block that has to be updated at the current iteration (line~\ref{getcurrentindex}). In order to find the best new block we need to take into account that some parts of the data have already been covered, and we must ignore them. This is accomplished by replacing the original matrix with a residual $\matr{R}$ that represents what there is left to cover. The building of the residual (line~\ref{capricorn:residual}) reflects the winner-takes-it-all property of the max-times algebra: if an element of $\matr{A}$ is approximated by a smaller value, it appears as such in the residual; if it is approximated by a value that is at least as large, then the corresponding residual element is \NaN, indicating that this value is already covered. We then select a seed row (line~\ref{seedrow}), with an intention of growing a block around it. We choose the row with the largest sum as this increases the chances of finding the most prominent block. In order to find the best block $\matr{X}$ that the seed row passes through, we first find a binary matrix $\matr{H}$ that represents the pattern of $\matr{X}$ (line~\ref{getpattern}). Next, on lines \ref{startgetbc}--\ref{endgetbc} we choose an approximation of the block pattern with index sets ${b\_idx}$ and $c\_idx$, which define what elements of $\vec{b}$ and $\vec{c}$ should be nonzero. The next step is to find the actual values of elements within the block with the function \RecoverBlock (line~\ref{recoverblock}). Finally, we inflate the found core block with \ExpandBlock (line~\ref{expandblock}).

\begin{algorithm}[tbp]
  \flushleft\small%
  \caption{\UpdateBlock (\Capricorn)}\label{alg:updateblock}
  \begin{algorithmic}[1]
    \Input $\matr{A} \in \Region^{n \times m}$, $\matr{B} \in \Region^{n \times k}$, $\matr{C} \in \Region^{k \times m}$, $\mathit{count}>0$
    \Output $\vec{b} \in \Region^{n \times 1}$, $\vec{c} \in \Region^{1 \times m}$ 
    \Parameters $\bucketSize>0$, $\delta>0$, $\theta>0$, $\tau\in[0,1]$ 
    \Function{\UpdateBlock}{\matr{A}, \matr{B}, \matr{C}, \mathit{count}}
    \State $l \gets (\mathit{count}-1) \pmod k + 1$                     \Comment{Index of the current block} \label{getcurrentindex}
                  \State $\matr{R}_{ij} \gets
         \begin{cases}
           \matr{A}_{ij} &  (\matr{B}^{-l} \maxprod \matr{C}_{-l})_{ij} < \matr{A}_{ij}\\
           \NaN &  \text{otherwise}
         \end{cases}$
        \Comment{Residual matrix}  \label{capricorn:residual}
    \State $\idx \gets \argmax_i \sum_j r_{ij}$ \label{seedrow} 
    \State $\matr{H} \gets \CorrelationsWithRow(\matr{R}, \idx, \bucketSize, \delta, \tau)$ \label{getpattern}
    \State $r \gets \argmax_{i} \sum_j h_{ij}$ \label{startgetbc} 
    \State $c \gets \argmax_j \sum_i h_{ij}$ 
    \State ${b\_idx} \gets \lbrace i \setcond \matr{H}_{i c} = 1\rbrace$ 
    \State $c\_idx \gets \lbrace i \setcond \matr{H}_{r i} = 1\rbrace$ \label{endgetbc}
    \State $[\vec{b}, \vec{c}] \gets \RecoverBlock(\matr{R}, {b\_idx}, c\_idx)$ \label{recoverblock}
    \State $\vec{b} \gets \AddRows(\vec{b}, \vec{c}, \matr{A}, \theta, \bucketSize, \delta)$ \label{expandblock}
    \State $\vec{c} \gets \AddRows(\vec{c}^T, \vec{b}^T, \matr{A}^T, \theta, \bucketSize, \delta)^T$
    \State \textbf{return}  $\vec{b}$, $\vec{c}$
    \EndFunction
  \end{algorithmic}
\end{algorithm}

The function \lword{\CorrelationsWithRow} (Algorithm~\ref{alg:correlationsWithRow}) finds the pattern of a new block. It does so by comparing a given seed row to other rows of the matrix and extracting sets where the ratio of the rows is almost constant. As was mentioned before, if two rows locally represent the same block, then one should be a multiple of the other, and the ratios of their corresponding elements should remain level. \lword{\CorrelationsWithRow} processes the input matrix row by row using the function \FindRowSet, which for every row outputs the most likely set of indices, where it is correlated with the seed row (lines \ref{startcorr}--\ref{endcorr}). Since the seed row is obviously the most correlated with itself, we compensate for this by replacing its pattern with that of the second most correlated row (lines \ref{replaceseedbegin}--\ref{replaceseedend}). Finally, we drop some of the least correlated rows after comparing their correlation value $\phi$ to that of the second most correlated row (after the seed row). The correlation function $\phi$ is defined as follows
\begin{equation}
\phi(\matr{H}, \idx, i) = \frac{\langle \matr{H}_i, \matr{H}_{\idx}\rangle}{\langle \matr{H}_i, \matr{H}_i\rangle + 1} \;. \label{eq:phi}
\end{equation}
The parameter $\tau$ is a threshold determining whether a row should be discarded or retained. The auxiliary function \FindRowSet (Algorithm~\ref{alg:FindRowSet}) compares two vectors and finds the biggest set of indices where their ratio remains almost constant. It does so by sorting the log-ratio of the input vectors into buckets of a fixed size and then choosing the bucket with the most elements. The notation $\vec{u} \vecdivide \vec{v}$ on line~\ref{getlogratios} means elementwise ratio of vectors $\vu$ and $\vv$.

It accepts two additional parameters: $\bucketSize$ and $\delta$. If the largest bucket has fewer than $\bucketSize$ elements, the function will return an empty set -- this is done because very small patterns do not reveal much structure and are mostly accidental. The width of the buckets is determined by the parameter $\delta$.

\begin{algorithm}[tbp]
  \flushleft\small%
  \caption{\CorrelationsWithRow}\label{alg:correlationsWithRow}
  \begin{algorithmic}[1]
    \Input $\matr{R} \in \Region^{n \times m}$, $idx \in [n]$, $\bucketSize>0$, $\delta>0$, $\tau\in[0,1]$
    \Output $\matr{H} \in \lbrace 0,\, 1 \rbrace^{n\times m}$
    \Function{\CorrelationsWithRow}{\matr{R}, \idx, \bucketSize, \delta, \tau}
    \State turn all $\NaN$ elements of $\matr{R}$ to 0
    \State $\matr{H} \gets 0^{n \times m}$
    \For {$i \gets 1$ \textbf{to} $n$} \label{startcorr}
    \State $V_i \gets \FindRowSet(\matr{R}_{\idx}, \matr{R}_{i}, \bucketSize, \delta)$
    \State $\matr{H}(i, V_i) \gets 1$  \label{endcorr}
    \EndFor
    \State $s \gets \argmax_{i \setcond i\neq \idx}\sum_j h_{ij}$ \label{replaceseedbegin} 
    \State $\matr{H}_{\idx} \gets \matr{H}_{s}$ \label{replaceseedend}
    \For {$i \gets 1$ \textbf{to} $n$}
    \If {$\phi(\matr{H}, \idx, i) < \phi(\matr{H}, \idx, s) - \tau$}
    \State $\matr{H}_{i} \gets 0$
    \EndIf
    \EndFor
    \State \textbf{return}  $\matr{H}$
    \EndFunction
  \end{algorithmic}
\end{algorithm}

\begin{algorithm}[tbp]
  \flushleft\small%
  \caption{\FindRowSet}\label{alg:FindRowSet}
  \begin{algorithmic}[1]
     \Input $\vec{u} \in \Region^m, \vec{v} \in \Region^m, \bucketSize > 0, \delta > 0$
     \Output $V \subset [m]$
     \Function{\FindRowSet}{\vec{u}, \vec{v}, \bucketSize, \delta}
     \State $\vec{r} \gets \log(\vec{u} \vecdivide \vec{v})$ \label{getlogratios}
     \State $\nBuckets \gets \ceil{(\max\{r\}-\min\{r\})/\delta}$  
    \For {$i \gets 0$ \textbf{to} $\nBuckets$}
    \State $V_i \gets \{ \idx \in [m]\setcond \min\{\vec{r}\}+i\delta \le r_{\idx} < \min\{\vec{r}\}+(i+1)\delta\}$
    \EndFor
    \State $V \gets \argmax\{\abs{V_i} \setcond i=1,\ldots,\nBuckets\}$
    \If {$\abs{V} < \bucketSize$}
    \State $V \gets \emptyset$
    \EndIf
    \State \textbf{return}  $V$
    \EndFunction
  \end{algorithmic}
\end{algorithm}

At this point we know the pattern of the new block, that is, the locations of its non-zeros. To fill in the actual values, we consider the submatrix defined by the pattern, and find the best rank-1 approximation of it. We do this using the \RecoverBlock function (Algorithm~\ref{alg:recoverblock}). It begins by setting all elements outside of the pattern to 0 as they are irrelevant to the block (line \ref{recovercancel}). Then it chooses one row to represent the block (lines~\ref{representbegin}--\ref{representend}), which will be used to find a good rank-1 cover.

Finally, we find the optimal column vector for the block by computing the best weights to be used for covering different rows of the block with its representing row (line \ref{getb}). Here we optimize with respect to the Frobenius norm, rather than $L_1$ matrix norm, since it allows to solve the optimization problem in closed form.

\begin{algorithm}[tbp]
  \flushleft\small%
  \caption{\RecoverBlock}\label{alg:recoverblock}
  \begin{algorithmic}[1]
     \Input $\matr{R} \in \Region^{n\times m}, \bIdx \subset [n], \cIdx \subset [m]$
     \Output $\vec{b} \in \Region^{n\times 1}$,  $\vec{c} \in \Region^{1\times m}$
     \Function{\RecoverBlock}{\matr{R}, \bIdx, \cIdx}
     \State turn $\matr{R}$ to 0 except elements with indices $(\bIdx, \cIdx)$ \label{recovercancel}
     \State $p \gets \RowRepresentingBlock(\matr{R}, \bIdx)$ \label{representbegin}
     \State $\vec{c} \gets \matr{R}_{p}$ \label{representend}
     \State $\vec{b} \gets {{\argmin}_{\vec{t}\in \Region^{n \times 1}}\norm{\matr{R}-\vec{t}\vec{c}}_F}$ \label{getb}
    \State \textbf{return}  $\vec{b}$, $\vec{c}$
    \EndFunction
  \end{algorithmic}
\end{algorithm}

Since blocks often heavily overlap, we are susceptible to finding only fragments of patterns in the data -- some parts of a block can be dominated by another block and subsequently not recognized. Hence, we need to expand found blocks to make them complete. This is done separately for rows and columns in the method called \AddRows (Algorithm~\ref{alg:AddRows}), which, given a starting block $\matr{X}=\vec{b}\vec{c}$ and the original matrix $\matr{A}$, tries to add new nonzero elements to $\vec{b}$. It iterates through all rows of $\matr{A}$ and adds those that would make a positive impact on the objective without unnecessarily overcovering the data. In order to decide whether a given row should be added, it first extracts a set $V_i$ of indices where this row is a multiple of the row vector $\vec{c}$ of the block (if they are not sufficiently correlated, then the row does not belong to the block) (line~\ref{addrowlab1}). A row is added if the evaluation of the following function  (line~\ref{impact})
\begin{equation} \label{myimpact}
\psi(\alpha) =       \frac{\sum_{s \in V_i} \max\lbrace 0, \,\alpha c_s - \matr{A}_{is} \rbrace} {\sum_{s \in V_i} \matr{A}_{is} - \abs{\matr{A}_{is} - \alpha c_s}} 
\end{equation}
is below the threshold $\theta$. 
In \eqref{myimpact} the numerator measures by how much the new row would overcover the original matrix, and the denominator reflects the improvement in the objective compared to a zero row.


\begin{algorithm}[tbp]
  \flushleft\small%
  \caption{\AddRows}\label{alg:AddRows}
  \begin{algorithmic}[1]
     \Input $\vec{b} \in \Region^{n \times 1}$, $\vec{c} \in \Region^{1 \times m}$, $\matr{A} \in \Region^{n\times m}$, $\theta > 0$, $\bucketSize>0$, $\delta>0$
     \Output $\vec{b} \in \Region^{n\times 1}$
     \Function{\AddRows}{\vec{b}, \vec{c}, \matr{A}, \theta, \bucketSize, \delta}
     \State ${b\_idx} \gets \lbrace t \setcond \vec{b}_t > 0\rbrace$
     \For {$i \in [n]\setminus {b\_idx}$}
     \State $V_i \gets \FindRowSet(\vec{c}, \matr{R}_i, \bucketSize, \delta)$ \label{addrowlab1}
     \If {$V_i = \emptyset$}
     \State \textbf{continue}
     \EndIf
     \State $\alpha \gets mean(\matr{R}_{iV_i}./\vec{c}_{V_i})$ \label{getalpha}
     \State $\mathit{impact} \gets 
      \frac{\sum_{s \in V_i} \max\{ 0,\, \alpha c_s - \matr{A}_{is} \} } { \sum_{s \in V_i} \matr{A}_{is} - \abs{\matr{A}_{is} - \alpha c_s}}
     $ \label{impact}
     \If {$impact \le \theta$}
          \State $\vec{b}_i \gets \alpha$ \label{getbi}
     \EndIf
     \EndFor
     \State \textbf{return}  $\vec{b}$
    \EndFunction
  \end{algorithmic}
\end{algorithm}

\textbf{Parameters}. \Capricorn has four parameters in addition to the common parameters in the Equator framework:  $\bucketSize>0$, $\delta>0$, $\theta>0$, and $\tau\in[0,1]$. The first one, $\bucketSize$ determines the minimum number of elements in two rows that must have ``approximately'' the same ratio for them to be considered for building a block. The parameter $\delta$ defines the bucket width when computing row correlations. When expanding a block, $\theta$ is used to decide whether to add a row (or column) to it -- the decision is positive whenever the expression~\eqref{myimpact} is at most $\theta$. Finally $\tau$ is used during the discovery of correlated rows. The value of $\tau$ belongs to the closed unit interval, and the higher it is, the more rows will be added.

\subsection{\Cancer}

We now present our second algorithm, \Cancer, which is a counterpart of \Capricorn specifically designed to work in the presence of high levels of continuous noise. The reason why \Capricorn cannot deal with continuous noise is that it expects the rows in a block to have an ``almost'' constant elementwise ratio, which is not the case when too many entries in the data are disturbed. For example, even low levels of Gaussian noise would make the ratios vary enough to hinder \Capricorn's ability to spot blocks. With \Cancer we take a new approach which is based on polynomial approximation of the objective. We also replace the $L_1$ matrix norm, which was used as an objective for \Capricorn, with the Frobenius norm. The reason for that is that when the noise is continuous, its level is defined as the total deviation of the noisy data from the original, rather than a count of the altered elements. This makes the Frobenius norm a good estimator for the amount of noise. \Cancer conforms to the general framework of \Equator (Algorithm~\ref{alg:equator}), and differs from \Capricorn only in how it finds the blocks and in the objective function.

Observe that in order to solve the problem \eqref{maxtimesrank1} we need to find a column vector $\vec{b} \in \Region^{n\times 1}$ and a row vector $\vec{c} \in \Region^{1\times m}$ such that they provide the best rank-1 approximation of the input matrix given the current factorization. The objective function is not convex in either $\vec{b}$ or $\vec{c}$ and is generally hard to optimize directly, so we have to simplify the problem, which we do in two steps. First, instead of doing full optimization of $\vec{b}$ and $\vec{c}$ simultaneously, we update only a single element of one of them at a time. This way the problem is reduced to single variable optimization. Even then the objective is hard to minimize, and we replace it with a polynomial approximation, which is easy to optimize directly.

The \Cancer version of the \UpdateBlock function is described in Algorithm~\ref{alg:updateblockcancer}. It alternatingly updates the vectors $\vec{b}$ and $\vec{c}$ using the \AdjustOneElement routine. Both $\vec{b}$ and $\vec{c}$ will be updated $\lfloor f (n+m)/2\rfloor$ times. \UpdateBlock  starts by finding the index of the block that has to be changed (line~\ref{getcurrentindexcancer}). Since the purpose of \UpdateBlock is to find the best rank-1 matrix to replace the current block, we also need to compute the reconstructed matrix without it, which is done on line~\ref{findN}. We then find the number of times \AdjustOneElement will be called (line~\ref{getfraccancer}) and change the degree of polynomials used for objective function approximation (line~\ref{computedegree}). This is needed because high degree polynomials are better at finalizing a solution that is already reasonably good, but tend to overfit the data and cause the algorithm to get stuck in local minima at the beginning. It is therefore beneficial to start with polynomials of lower degrees and then gradually increase it. The actual changes to $\vec{b}$ and $\vec{c}$ happen in the loop (lines~\ref{innercycle}--\ref{updatebcancer}), where we update them using \AdjustOneElement.

The \AdjustOneElement function (Algorithm~\ref{alg:adjustoneelement}) updates a single entry in either a column vector $\vec{b}$ or a row vector $\vec{c}$. Let us consider the case when $\vec{b}$ is fixed and $\vec{c}$ varies. In order to decide which element of $\vec{c}$ to change, we need to compare the best changes to all $m$ entries and then choose the one that yields the most improvement to the objective. 
A single element $\vec{c}_l$ only has an effect on the error along the column $l$. Assume that we are currently updating block with index $q$ and let $\matr{N}$ denote the reconstruction matrix without this block, that is $\matr{N} = \matr{B}^{-q} \maxprod \matr{C}_{-q}$. Minimizing $E(\matr{A}, \matr{B}, \matr{C})$ with respect to $\vec{c}_l$ is then equivalent to minimizing 
\begin{equation}
  \label{eq:gamma}
  \gamma(\matr{A}_l, \matr{N}_l, \vec{b}, \vec{c}_l) = \sum_{i=1}^n (\matr{A}_{il} - \max \lbrace \matr{N}_{il}, \vec{b}_i \vec{c}_l\rbrace)^2\; . 
\end{equation}

Instead of minimizing~\eqref{eq:gamma} directly, we use polynomial approximation in the \PolyMin routine (line~\ref{line:polymin}). It returns the (approximate) error $\mathit{err}$ and the value $x$ achieving that.
Since we are only interested in the improvement of the objective achieved by updating a single entry of $\vec{c}$, we compute the improvement of the objective after the change (line~\ref{improvement}). After trying every column of $\vec{c}$, we update only the column that yield the largest improvement. 

\begin{algorithm}[tbp]
  \flushleft
  \caption{\UpdateBlock (\Cancer)}\label{alg:updateblockcancer}
  \begin{algorithmic}[1]
        \Input $\matr{A} \in \Region^{n \times m}$, $\matr{B} \in \Region^{n \times k}$, $\matr{C} \in \Region^{k \times m}$, $\mathit{count}>0$
    \Output $\vec{b} \in \Region^{n \times 1}$, $\vec{c} \in \Region^{1 \times m}$ 
    \Parameters $t>2$, $0<f<1$ 
    \Function{\UpdateBlock}{\matr{A}, \matr{B}, \matr{C}, \mathit{count}}
    \State $l \gets (\mathit{count}-1) \pmod k + 1$                     \Comment{Index of the current block} \label{getcurrentindexcancer}
            \State $\matr{N} \gets \matr{B}^{-l} \maxprod \matr{C}_{-l}$   \Comment{Reconstructed matrix without the $i$-th block} \label{findN}
            \State $\mathit{niters} \gets \lfloor f(n+m)/2 \rfloor$ \label{getfraccancer}
    \State $\mathit{deg} \gets 2 + \lfloor(\mathit{count}-1) / k\rfloor \pmod t$  \label{computedegree}
            \State $\vec{b} \gets \matr{B}^l$, $\vec{c} \gets \matr{C}_l$ \label{initbc}
    \For{$\mathit{iter} \gets 1$ \textbf{to} $\mathit{niters}$} \label{innercycle}
      \State $\vec{c} = \AdjustOneElement(\matr{A}, \matr{N}, \vec{b}, \vec{c}, \mathit{deg})$
      \State $\vec{b} = \AdjustOneElement(\matr{A}^T, \matr{N}^T, \vec{c}^T, \vec{b}^T, \mathit{deg})^T$ \label{updatebcancer}
    \EndFor
      \State \textbf{return} $\vec{b}$,  $\vec{c}$
    \EndFunction
  \end{algorithmic}
\end{algorithm}

\begin{algorithm}[tbp]
  \flushleft
  \caption{\AdjustOneElement}\label{alg:adjustoneelement}
  \begin{algorithmic}[1]
    \Input $\matr{A} \in \Region^{n \times m}$, $\matr{N} \in \Region^{n \times m}$, $\vec{b} \in \Region^{n \times 1}$, $\vec{c} \in \Region^{1 \times m}$,                  $\mathit{deg} \ge 2$
    \Output $\vec{c} \in \Region^{1 \times m}$
    \Function{\AdjustOneElement}{\matr{A}, \matr{N}, \vec{b}, \vec{c}, deg}
      \For{$j \gets 1$ \textbf{to} $m$}
        \State $\mathit{baseError} \gets \sum_{i=1}^n \left(\matr{A}_{ij} - \max\lbrace\matr{N}_{ij}, \vec{b}_i \vec{c}_j \rbrace \right)^2$ \label{baseerror}
        \State $[err, \vec{x}_i] \gets \PolyMin(\matr{A}^j, \matr{N}^j, \vec{b}, \mathit{deg})$ \label{line:polymin}
        \State $\vec{u}_i \gets \mathit{baseError} - \mathit{err}$ \label{improvement}
      \EndFor
      \State $i \gets $ the index $i$ of largest value of $\vec{u}$
      \State $\vec{c}_i \gets \vec{x}_i$
      \State \textbf{return} $\vec{c}$
    \EndFunction
  \end{algorithmic}
\end{algorithm}

The function $\gamma$ that we need to minimize in order to find the best change to the vector $\vec{c}$ in \AdjustOneElement is hard to work with directly since it is not convex, and also not smooth because of the presence of the maximum operator. To alleviate this, we approximate the error function $\gamma$ with a polynomial $g$ of degree $deg$. Notice that when updating $\vec{c}_l$, other variables of $\gamma$ are fixed and we only need to consider function $\gamma'(x) = \gamma(\matr{A}_l, \matr{N}_l, \vec{b}, x)$. To build $g$ we sample $deg+1$ points from $(0,1)$ and fit $g$ to the values of $\gamma'$ at these points. We then find the $x\in\Region$ that minimizes $g(x)$ and return $g(x)$ (the approximate error) and $x$ (the optimal value).



\textbf{Parameters}. \Cancer has two parameters, $t>2$ and $0<f<1$, that control its execution. The first one, $t$, is the maximum allowed degree of polynomials used for approximation of the objective, which we set to 16 in all our experiments. The second parameter, $f$, determines the number of single element updates we make to the row and column vectors of a block in \UpdateBlock.

\textbf{Generalized Cancer}\label{generalcancer}. The \Cancer algorithm can be adapted to optimize other objective functions. Its general polynomial approximation framework allows for a wide variety of possible objectives, the only constraint being that they have to be additive (we call a function $E(\matr{A}, \matr{R})$ \textit{additive} if there exists a mapping $\phi\colon \Region \times \Region \rightarrow \Region$ such that for all $\matr{A} \in \Region^{n \times m}$ and $\matr{R} \in \Region^{n \times m}$ we have $E(\matr{A}, \matr{R}) = \sum_{ij}\phi(\matr{A}_{ij}, \matr{R}_{ij})$). Some examples of such functions are $L_1$ and Frobenius matrix norms, as well as Kullback--Leibler and Jensen--Shannon divergences. In order to use the generalized form of \Cancer one simply has to replace the Frobenius norm with another cost function wherever the error is evaluated.

\subsection{Time complexity}
The main work in \Equator is performed inside the \UpdateBlock routine, which is called $M k$ times. Since $M$ is a constant parameter, the complexity of \Equator is $k$ times the complexity of \UpdateBlock. In the following we find the theoretical bounds on the execution time of \UpdateBlock for both \Capricorn and \Cancer.

\textbf{Capricorn.}
In the case of \Capricorn there are three main contributors to \UpdateBlock (Algorithm~\ref{alg:updateblock}): \CorrelationsWithRow, \RecoverBlock, and \AddRows. \lword{\CorrelationsWithRow} compares every row to the seed row, each time calling \FindRowSet, which in turn has to process all $m$ elements of both rows. This results in the total complexity of \lword{\CorrelationsWithRow} being $O(nm)$. To find the complexity of \RecoverBlock, first observe that any ``pure'' block $\matr{X}$ can be represented as $\matr{X}=\vec{b}\vec{c}$, where $\vec{b}\in \Region^{n'\times 1}$ and $\vec{c}\in \Region^{1\times m'}$ with $n'\le n$  and $m'\le m$. \RecoverBlock selects $\vec{c}$ from the rows of $\matr{X}$ and then finds the corresponding column vector $\vec{b}$ that minimizes $\norm{\matr{X}-\vec{b}\vec{c}}_F$. In order to select the best row, we have to try each of the $n'$ candidates, and since finding the corresponding $\vec{b}$ for each of them takes time $O(n'm')$, this gives the runtime of \RecoverBlock as $O(n')O(n'm') = O(n^2m)$. The most computationally expensive parts of \AddRows are \FindRowSet (line \ref{addrowlab1}), finding the mean (line \ref{getalpha}), and computing the impact (line \ref{impact}), which all run in $O(m)$ time. All of these operations have to be repeated $O(n)$ times, and hence the runtime of \AddRows is $O(nm)$. Thus, we can now estimate the complexity of \UpdateBlock to be $O(nm)+O(n^2m)+O(nm) = O(n^2m)$, which leads to the total runtime of \Capricorn to be $O(n^2mk)$.

\textbf{Cancer.}
Here \UpdateBlock (Algorithm~\ref{alg:updateblockcancer}) is a loop that calls \AdjustOneElement $\lfloor f(n+m) \rfloor$ times. In \AdjustOneElement the contributors to the complexity are computing the base error (line~\ref{baseerror}) and a call to \PolyMin (line~\ref{line:polymin}). Both of them are performed $n$ or $m$ times depending on whether we supplied the column vector $\vec{b}$ or the row vector $\vec{c}$ to \AdjustOneElement. Finding the base error takes time $O(m)$ for $\vec{b}$ and $O(n)$ for $\vec{c}$. The complexity of \PolyMin boils down to that of evaluating the max-times objective at $deg+1$ points and then minimizing a degree $deg$ polynomial. Hence, \PolyMin runs in time $O(m)$ or $O(n)$ depending on whether we are optimizing $\vec{b}$ or $\vec{c}$, and the complexity of \AdjustOneElement is $O(nm)$.

Since \AdjustOneElement is called $\lfloor f(n+m)/2 \rfloor$ times and $f$ is a fixed parameter, this gives the complexity $O\bigl((n+m)nm\bigr)$ for \UpdateBlock and $O\bigl((n+m)nmk\bigr) = O(\max\{n,m\}nmk)$ for \Cancer.


\section{Experiments}
\label{sec:experiments}

We tested both \Capricorn and \Cancer on synthetic and real-world data. In addition we also compare against a variation of \Cancer that optimizes the Jensen--Shannon divergence, which we call \CancerJS. The purpose of the synthetic experiments is to evaluate the properties of the algorithm in controlled environments where we know the data has the max-times structure. They also demonstrate on what kind of data each algorithm excels and what their limitations are. The purpose of the real-world experiments is to confirm that these observations also hold true in real-world data, and to study what kinds of data sets actually have max-times structure. The source code of \Capricorn and \Cancer and the scripts that run the experiments in this paper are freely available for academic use.\!\footnote{\url{http://people.mpi-inf.mpg.de/~pmiettin/tropical/}}

\paragraph{Parameters of \Capricorn.} In both synthetic and real-world experiments we used the following default set of parameters: $M=4$, $\bucketSize=3$, $\delta=0.01$, $\theta=0.5$, and $\tau=0.5$.
\paragraph{Parameters of \Cancer.} Both variations of \Cancer use the same set of parameters. For the synthetic experiments we used $M=14$, $t=16$, and $f=0.1$. For the real world experiments we set $t=16$, $f=0.1$, and $M=40$ (except for \Eigenfaces, where we used $M=50$).

\subsection{Other methods.}
\label{sec:exp:other-methods}

We compared our algorithms against \SVD and six versions of NMF. For \SVD, we used Matlab's built-in implementation. The first NMF method, called simply \NMF, by \cite{kim2008toward}, is based on the block principal pivoting algorithm. The second form of NMF is a sparse NMF algorithm by \cite{hoyer04non-negative},\!\footnote{\url{https://github.com/aludnam/MATLAB/tree/master/nmfpack}, accessed 18 July 2017} which we call \SNMF. It defines the sparsity of a vector $\vec{x}\in\Region^n$ as
\begin{equation}
  \label{eq:hoyer}
  \text{sparsity}(\vec{x}) = \frac{\sqrt{n} - \left(\sum_i\abs{\vec{x}_i}\right)/\sqrt{\sum_i\vec{x}_i^2}}{\sqrt{n}-1}\; ,
\end{equation}
and returns factorizations where the sparsity of the factor matrices is user-controllable. In all of our experiments, we used the sparsity of \Cancer's factors as the sparsity parameter of \SNMF. We also compare against a standard alternating least squares algorithm  called \ALS \citep{cichocki09nonnegative}. Next we have two versions of NMF that are essentially the same as \ALS, but they use $L_1$ regularization for increased sparsity~\citep{cichocki09nonnegative}, that is, they aim at minimizing
\[
\norm{\mA - \mB\mC}_F + \alpha\norm{\mB}_1 + \beta\norm{\mC}_1\; .
\]
The first method is called \ALSR and uses regularizer coefficient $\alpha=\beta=1$, and the other, called \ALSRfive, has  regularizer coefficient $\alpha=\beta=5$. The last NMF algorithm, \WNMF by \citet{li2013non}, is designed to work with missing values in the data.

\subsection{Synthetic experiments.}
\label{sec:synth-exper}
The purpose of synthetic experiments is to prove the concept, that is that our algorithms are capable of identifying the max-times structure when it is there. In order to test this, we first generate the data with the pure max-times structure, then pollute it with some level of noise, and finally run the methods. The noise-free data is created by first generating  random factors of some density with  nonzero elements drawn from a uniform distribution on the $[0, 1]$ interval and then multiplying them using the max-times matrix product. 

We distinguish two types of noise. The first one is the discrete (or tropical) noise, which is introduced in the following way. Assume that we are given an input matrix $\matr{A}$ of size \by{n}{m}. We first generate an \by{n}{m} noise matrix  $\matr{N}$ with elements drawn from a uniform distribution on the $[0, 1]$ interval. Given a level of noise $l$, we then turn $\lfloor (1 - l)nm \rfloor$ random elements of $\matr{N}$ to 0, so that its resulting density is $l$. Finally, the noise is applied by taking elementwise maximum between the original data and the noise matrix $\matr{F} = \max \lbrace \matr{A}, \matr{N}\rbrace$. This is the kind of noise that \Capricorn was designed to handle, so we expect it to be better than \Cancer and other comparison algorithms.

We also test against continuous noise, as it is arguably more common in the real world. For that we chose Gaussian noise with 0 mean, where the noise level is defined to be its standard deviation. Since adding this noise to the data might result in negative entries, we truncate all values in a resulting matrix that are below zero. 

Unless specified otherwise, all matrices in the synthetic experiments are of size \by{1000}{800} with true max-times rank 10. All results presented in this section are averaged over 10 instances. For reconstruction error tests, we compared our algorithms \Capricorn, \Cancer, and \CancerJS against \SVD, \NMF, \SNMF, \ALS, \ALSR, and \ALSRfive. The error is measured as the relative Frobenius norm $\norm*{\matr{\tilde{A}} - \matr{A}}_F/\norm{\matr{A}}$, where $\matr{A}$ is the data and $\matr{\tilde{A}}$ its approximation, as that is the measure both \SVD and \NMF aim at minimizing. We also report the sparsity $s$ of factor matrices obtained by algorithms, which is defined as a fraction of zero elements in the factor matrices, 
 \begin{equation}
   \label{eq:sparsity}
   s(\mA) = \abs{\{(i, j) \setcond \mA_{ij} = 0\}}/(nm)\; ,
 \end{equation}
 for an \by{n}{m} matrix $\mA$. For the experiments with tropical noise, the reconstruction errors are reported in Figure~\ref{fig:synth:reconstruct:frob} and factor sparsity in Figure~\ref{fig:synth:sparsity}. For the Gaussian noise experiments, the reconstruction errors and factor sparsity are shown in Figure~\ref{fig:synth:err} and Figure~\ref{fig:synth:sparsity:gaussian} respectively.

\paragraph{Varying density with tropical noise.}
In our first experiment we studied the effects of varying the density of the factor matrices in presence of the tropical noise. We changed the density of the factors from 10\% to 100\% with an increment of 10\%, while keeping the noise level at 10\%. Figure~\ref{density:cap} shows the reconstruction error and Figure~\ref{density:frob:sparse} the sparsity of the obtained factors. \Capricorn is consistently the best method, obtaining almost perfect reconstruction; only when the density approaches 100\% does its reconstruction error deviate slightly from 0. This is expected since the data was generated with the tropical (flipping) noise that \Capricorn is designed to optimize. Compared to \Capricorn all other methods clearly underperform, with \Cancer being the second best. With the exception of \ALSRfive, all NMF methods obtain results similar to those of \SVD, while having a somewhat higher reconstruction error than \Cancer. That \SVD and NMF methods (except \ALSRfive) start behaving better at higher levels of density indicates that these matrices can be explained relatively well using standard algebra. \Capricorn and \Cancer also have the highest sparsity of factors, with \Capricorn exhibiting a decrease in sparsity as the density of the input increases. This behaviour is desirable since ideally we would prefer to find factors that are as close to the original ones as possible. For NMF methods there is a trade-off between the reconstruction error and the sparsity of the factors -- the algorithms that were worse at reconstruction tend to have sparser factors.

\paragraph{Varying tropical noise.}
The amount of noise is always with respect to the number of nonzero elements  in a  matrix, that is, for a matrix $\matr{A}$ with $\kappa(\matr{A})$ nonzero elements and noise level  $\alpha$, we flip  $\alpha \kappa(\matr{A})$ elements to random values. There are two versions of this experiment -- one with factor density 30\% and the other with 60\%. In both cases we varied the noise level from 0\% to 110\% with increments of 10\%. Figure~\ref{noise:cap} and Figure~\ref{noise:frobhd} show the respective reconstruction errors and Figure~\ref{noise:frob:sparse} and Figure~\ref{noise:frobhd:sparse} the corresponding sparsities of the obtained factors.
In the low-density case, \Capricorn is consistently the best method with essentially perfect reconstruction for up to $80\%$ of noise. In the high-density case, however, the noise has more severe effects, and in particular after $60\%$ of noise, \Cancer, \SVD, and all versions of NMF are better than \Capricorn. The severity of the noise is, at least partially, explained by the fact that in the denser data we flip more elements than in sparser data: for example when the data matrices are full, at 50\% of noise, we have already replaced half of the values in the matrices with random values. Further, the quick increase of the reconstruction error for \Capricorn hints strongly that the max-times structure of the data is mostly gone at these noise levels. 
\Capricorn also produces clearly the sparsest factors for the low density case, and is mostly tied with \Cancer and \ALSRfive when the density is high. It should be noted however that \ALSRfive generally has the highest reconstruction error among all the methods, which suggests that its sparse factors come at the cost of recovering little structure from the data. 

\paragraph{Varying rank with tropical noise.}
Here we test the effects of the (max-times) rank, with the assumption that higher-rank matrices are harder to reconstruct. The true max-times rank of the data varied from 2 to 20 with increments of 2. There are three variations of this experiment: with 30\% factor density and 10\% noise (Figure~\ref{dim:frob}), with 30\%  factor density and 50\% noise (Figure~\ref{dim:frobhn}), and with  60\% factor density and 10\% noise (Figure~\ref{dim:frobhd}). The corresponding sparsities are shown on Figures~\ref{dim:frob:sparse:sparse}, \ref{dim:frobhn:sparse}, and \ref{dim:frobhd:sparse}. \Capricorn has a clear advantage for all settings, obtaining nearly perfect reconstruction. \Cancer is generally second best, except for the high noise case, where it is mostly tied with a bunch of NMF methods. Interestingly, on the last two plots the reconstruction error actually drops for \Cancer, \SVD, and NMF-based methods. This is a strong indication that at this point they no longer can extract meaningful structure in the data, and the improvement of the reconstruction error is largely due to uniformization of the data caused by high density and high noise levels.

\begin{figure} [tp]  
  \centering
  \subfigure[Varying density test.] {%
       \includegraphics[width=\subfigwidth]{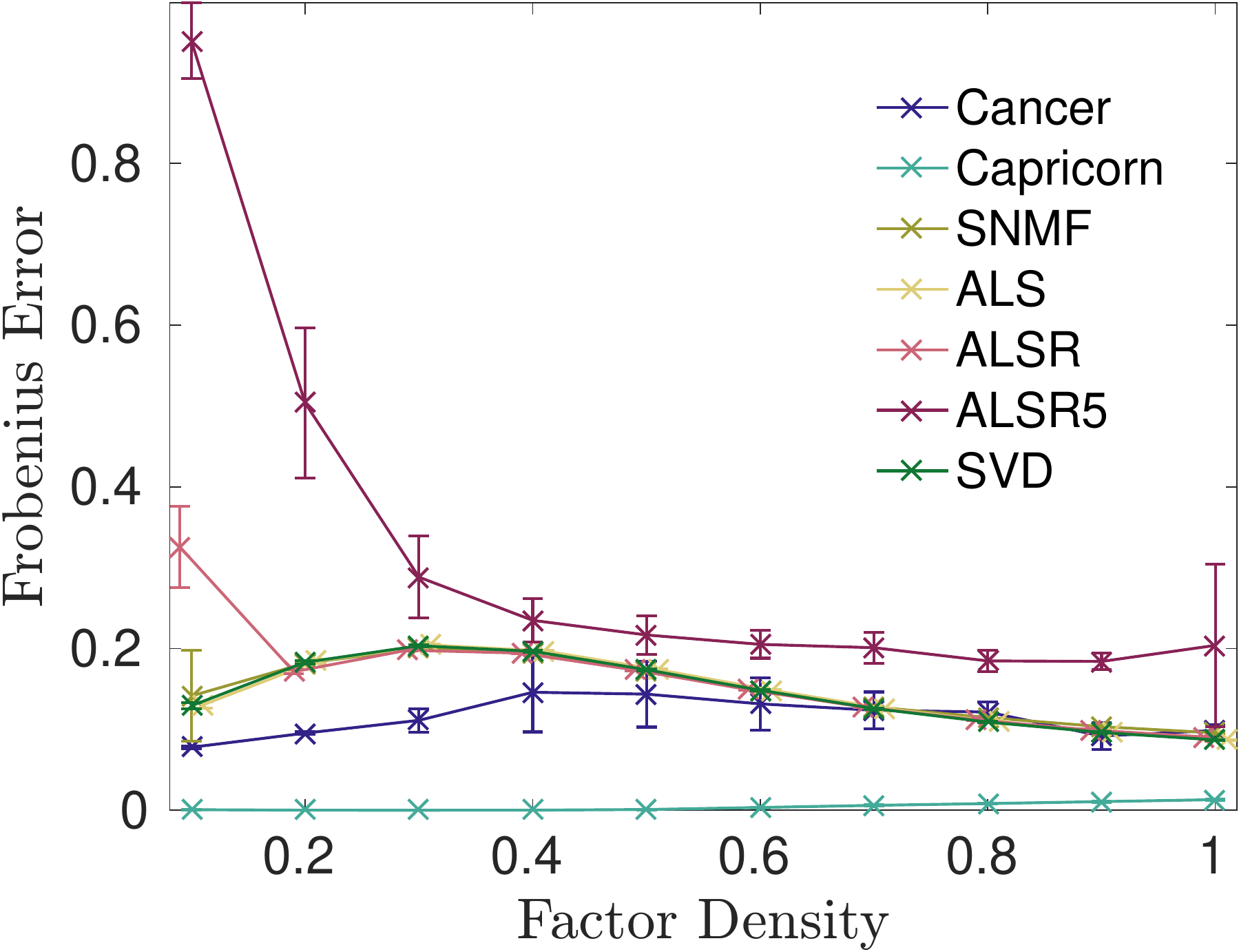}%
       \label{density:cap}%
       }
       \hspace{\subfigspace}
       \hspace{\subfigspace}
         \subfigure[Varying noise test.] {%
       \includegraphics[width=\subfigwidth]{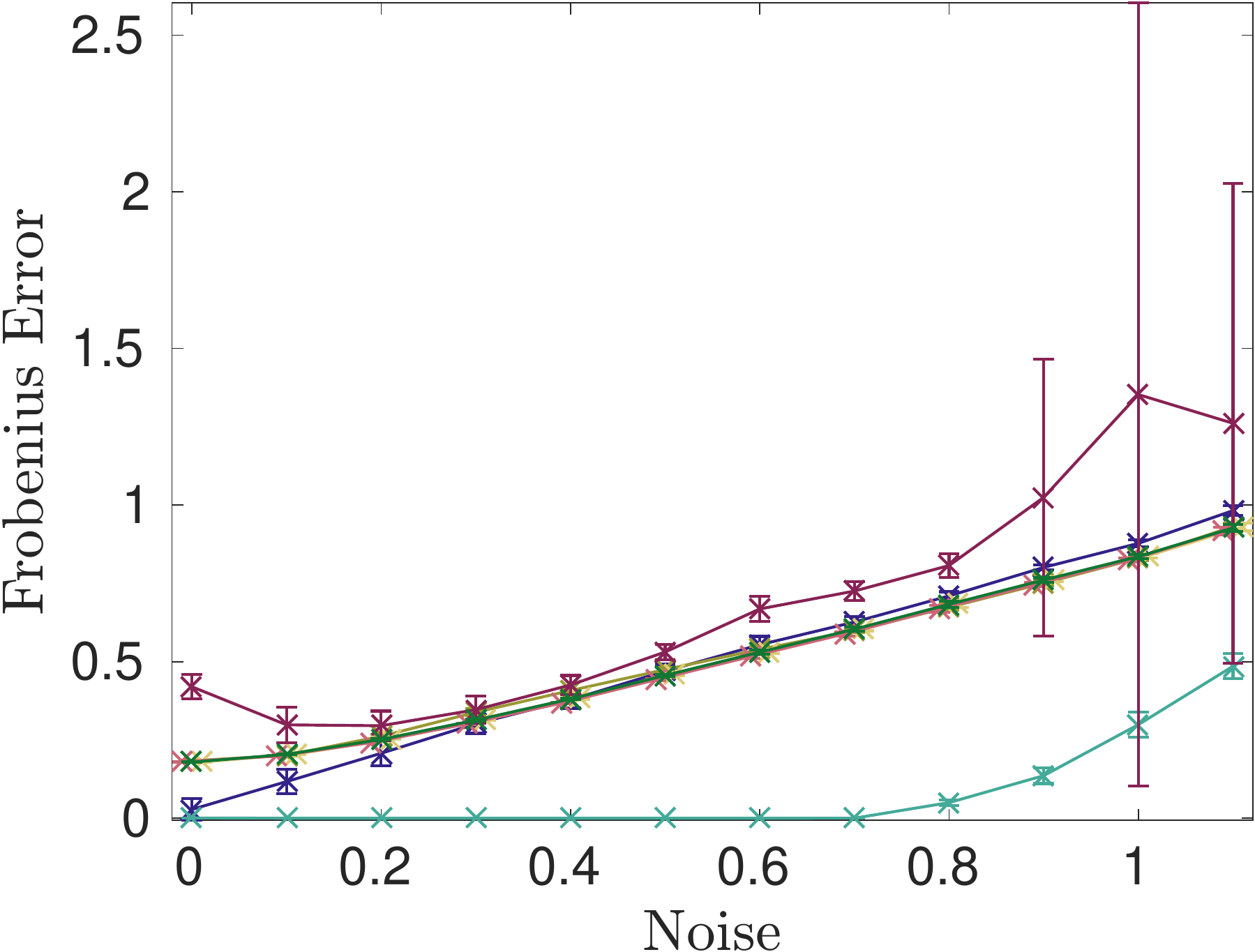}%
       \label{noise:cap}%
       }
       \hspace{\subfigspace}
          \subfigure[Varying noise with high density.] {%
        \includegraphics[width=\subfigwidth]{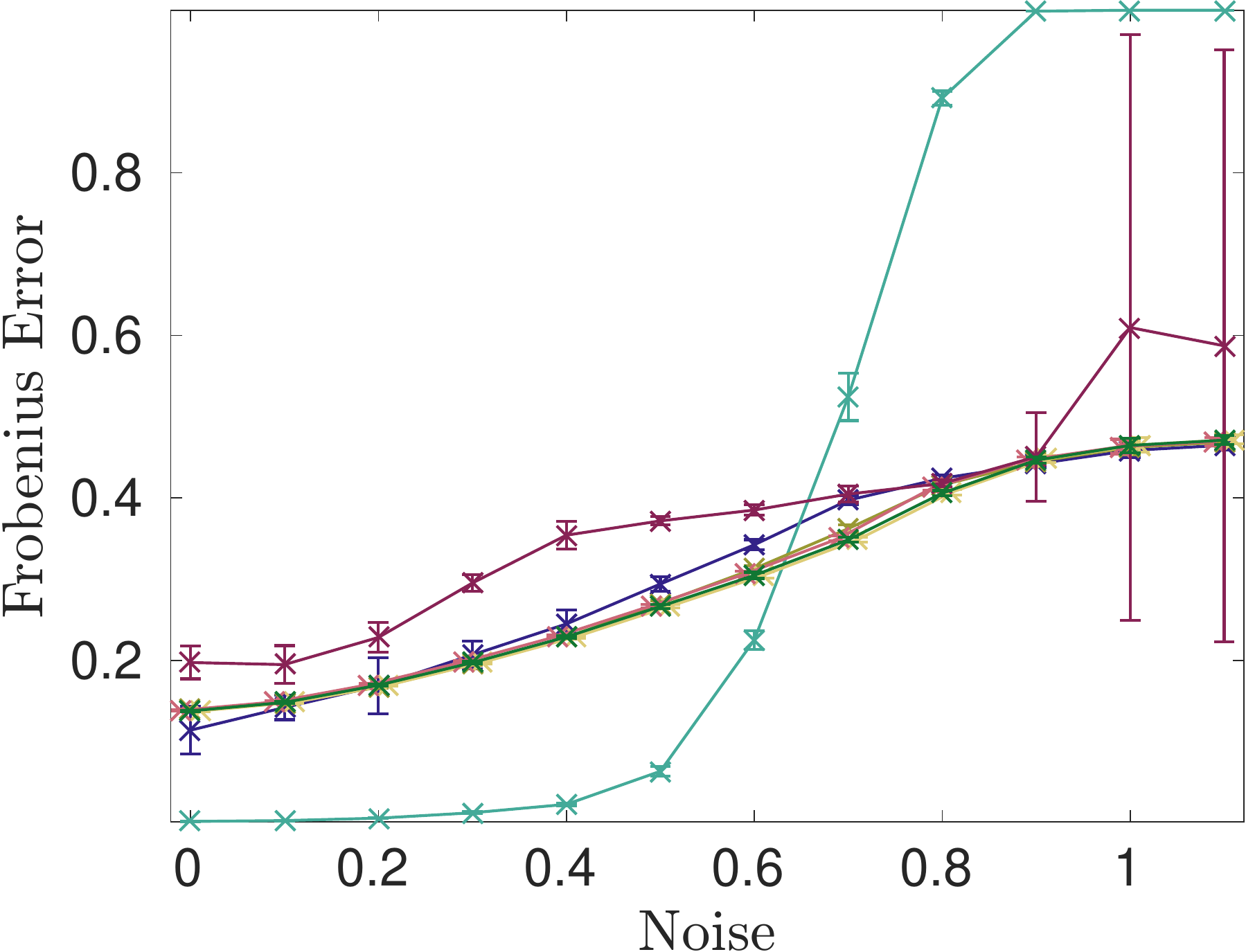}%
        \label{noise:frobhd}%
        }
        \hspace{\subfigspace}
  \\
           \subfigure[Varying rank test with 10\% noise and 30\% factor density.] {%
       \includegraphics[width=\subfigwidth]{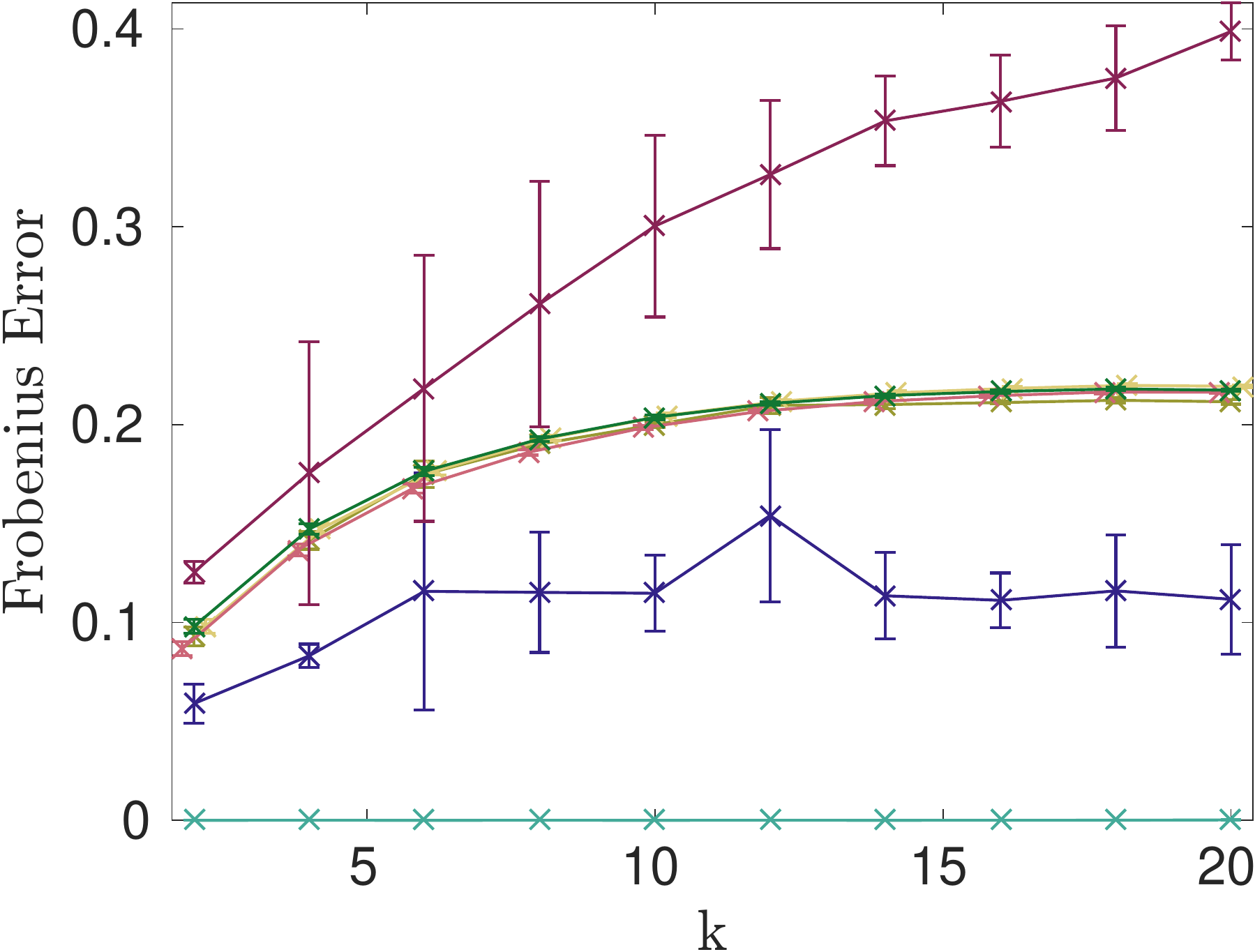}%
       \label{dim:frob}%
       }
       \hspace{\subfigspace}
           \subfigure[Varying rank test with 50\% noise and 30\% factor density.] {%
       \includegraphics[width=\subfigwidth]{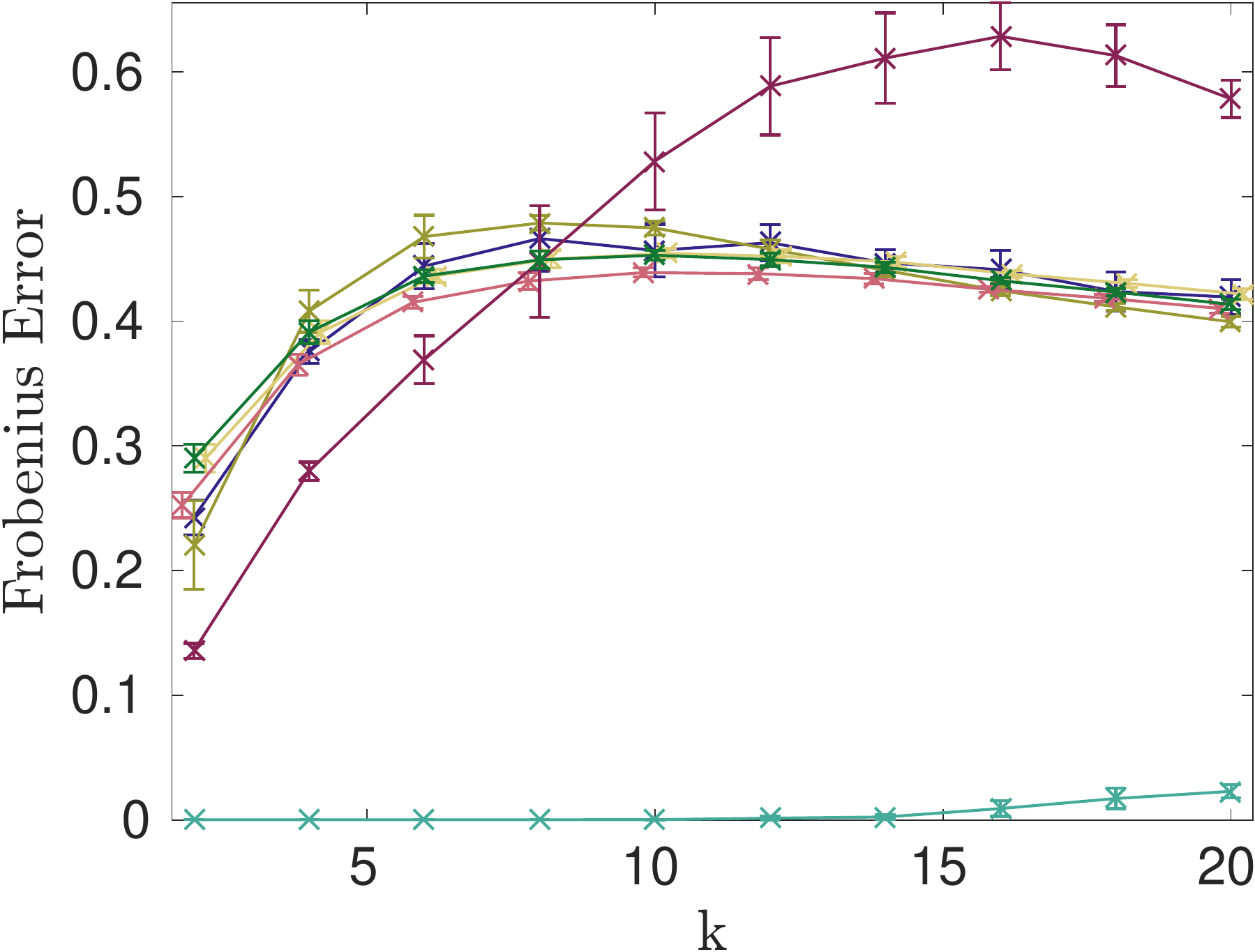}%
       \label{dim:frobhn}%
       }
       \hspace{\subfigspace}
           \subfigure[Varying rank test with 10\% noise and 60\% factor density.] {%
       \includegraphics[width=\subfigwidth]{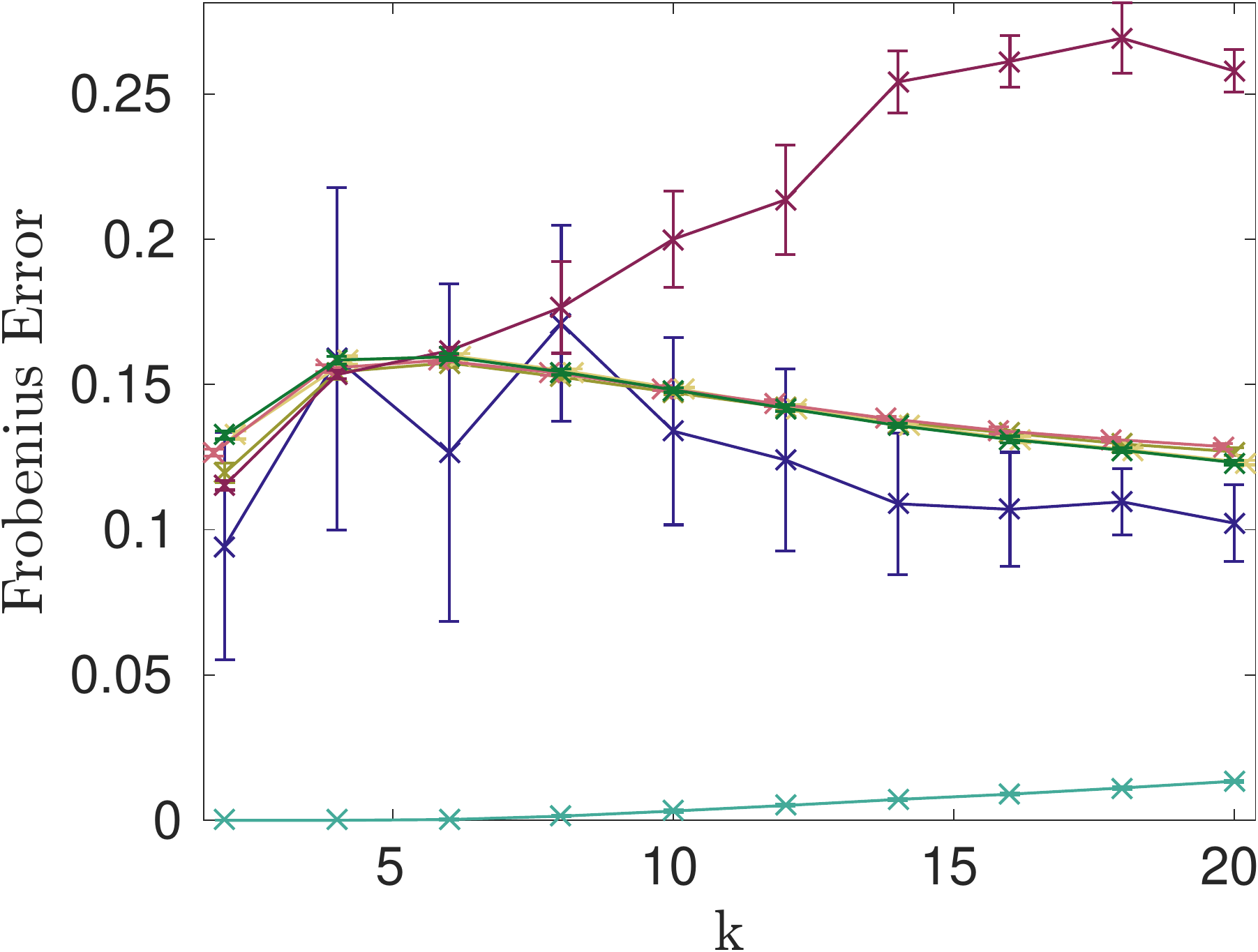}%
       \label{dim:frobhd}%
     }
       \caption{\textbf{Reconstruction errors on synthetic data with tropical noise}. $x$-axis is the parameter varied and $y$-axis is the relative Frobenius norm. All results are averages over 10 random matrices and the width of the error bars is twice the standard deviation. }
       \label{fig:synth:reconstruct:frob}
     \end{figure}

\begin{figure} [tp]  
  \centering
  \subfigure[Varying density test.] {%
       \includegraphics[width=\subfigwidth]{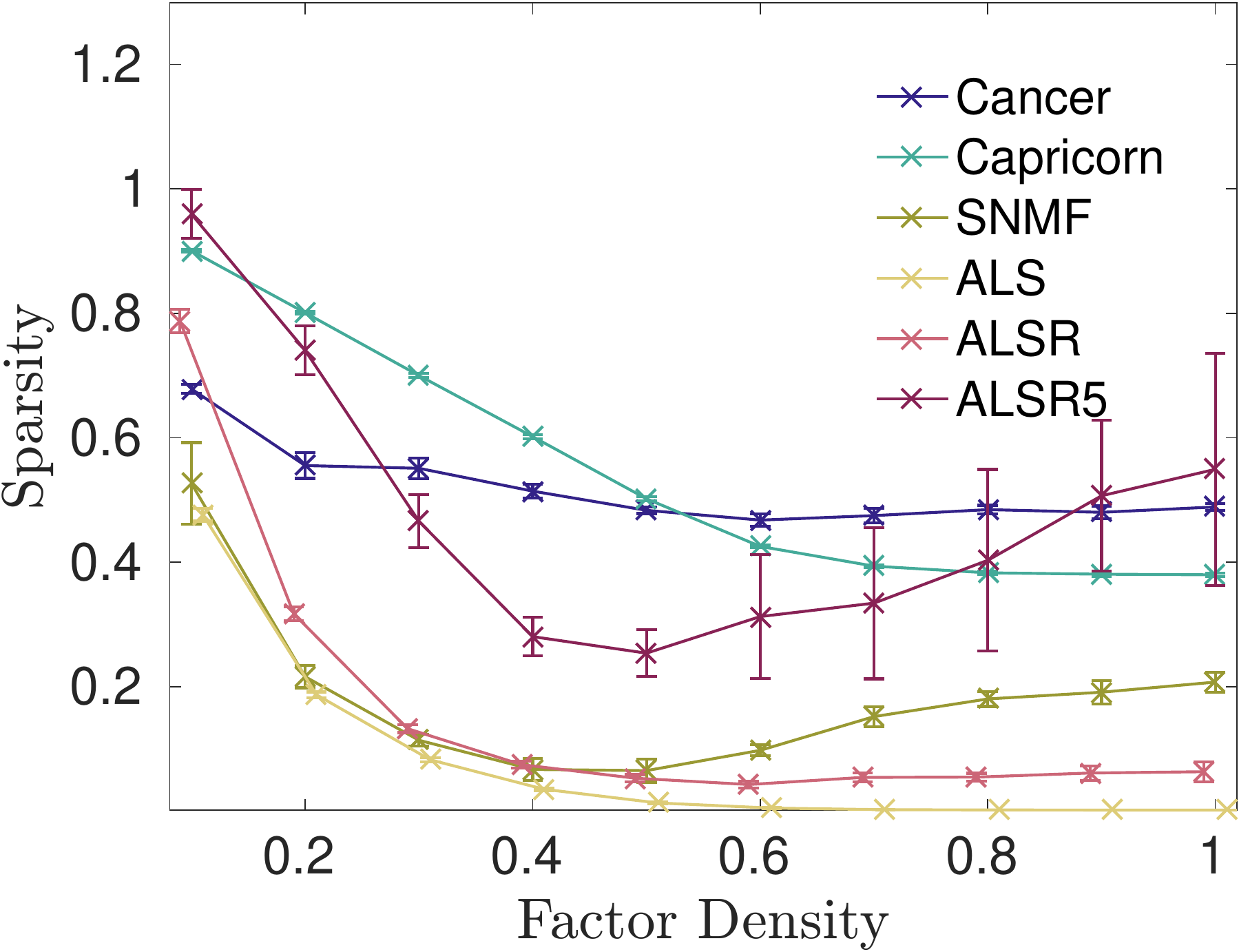}%
       \label{density:frob:sparse}%
       }
       \hspace{\subfigspace}
       \hspace{\subfigspace}
         \subfigure[Varying noise test.] {%
       \includegraphics[width=\subfigwidth]{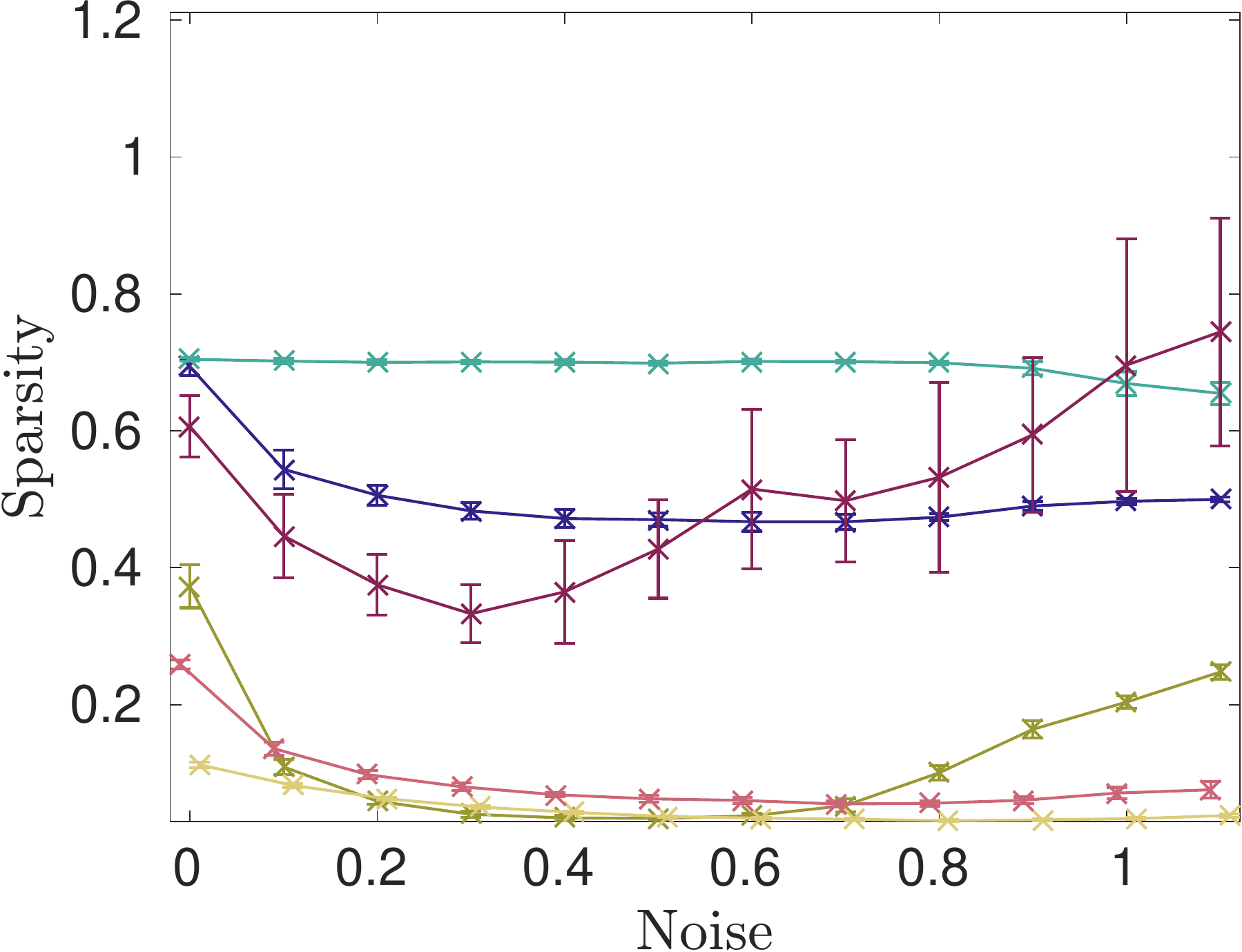}%
       \label{noise:frob:sparse}%
       }
       \hspace{\subfigspace}
          \subfigure[Varying noise with high density.] {%
        \includegraphics[width=\subfigwidth]{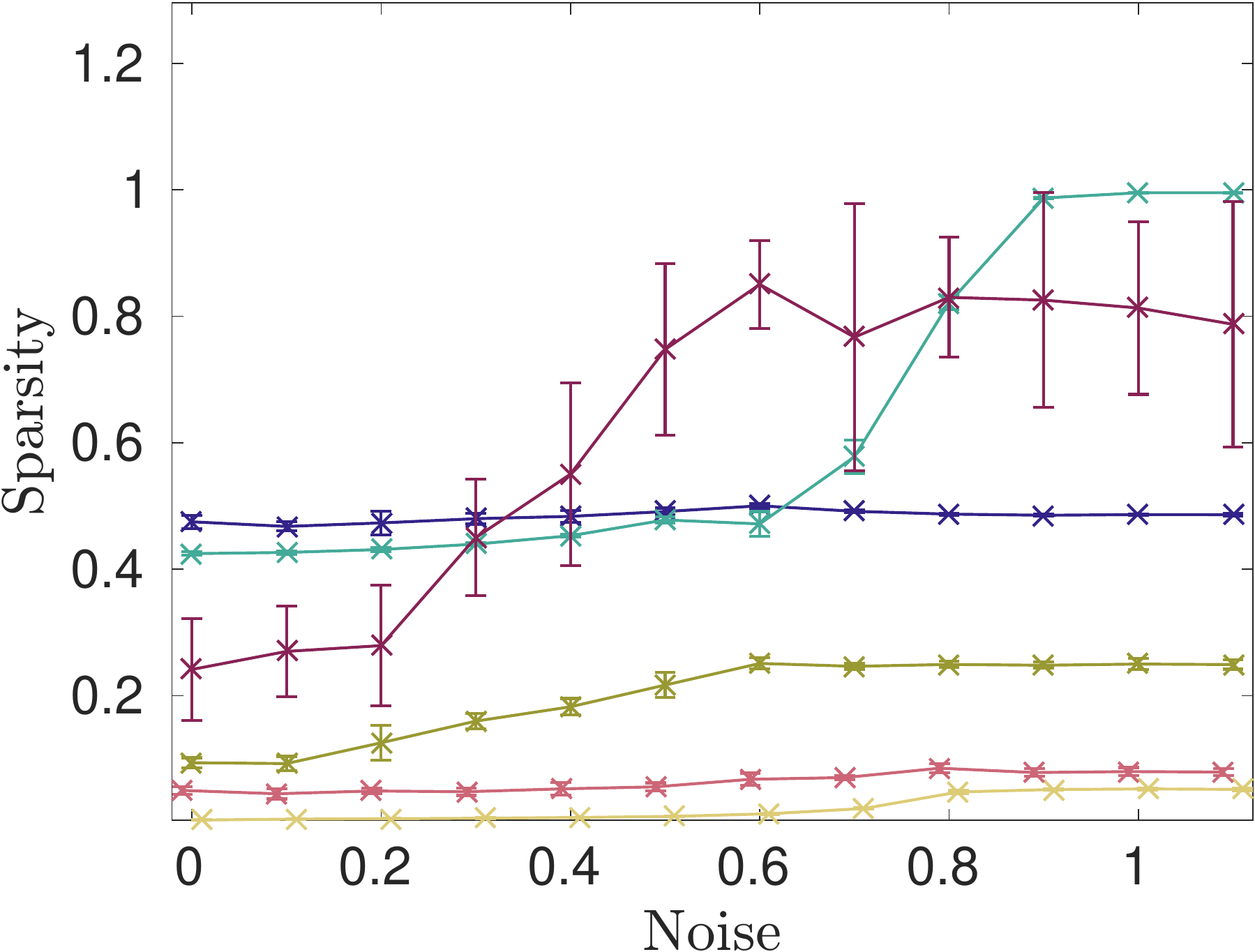}%
        \label{noise:frobhd:sparse}%
        }
        \hspace{\subfigspace}
  \\
           \subfigure[Varying rank test with 10\% noise and 30\% factor density.] {%
       \includegraphics[width=\subfigwidth]{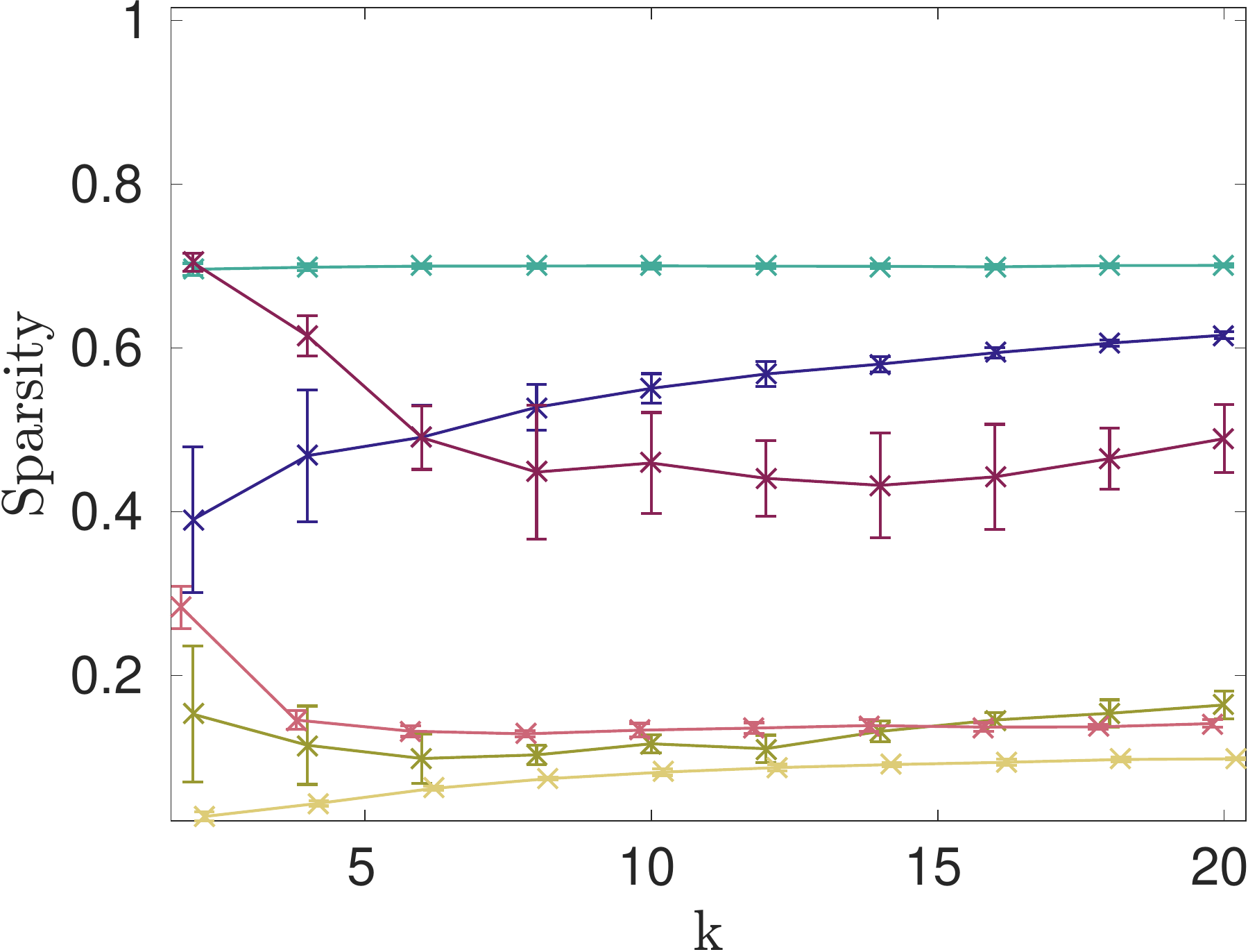}%
       \label{dim:frob:sparse:sparse}%
       }
       \hspace{\subfigspace}
           \subfigure[Varying rank test with 50\% noise and 30\% factor density.] {%
       \includegraphics[width=\subfigwidth]{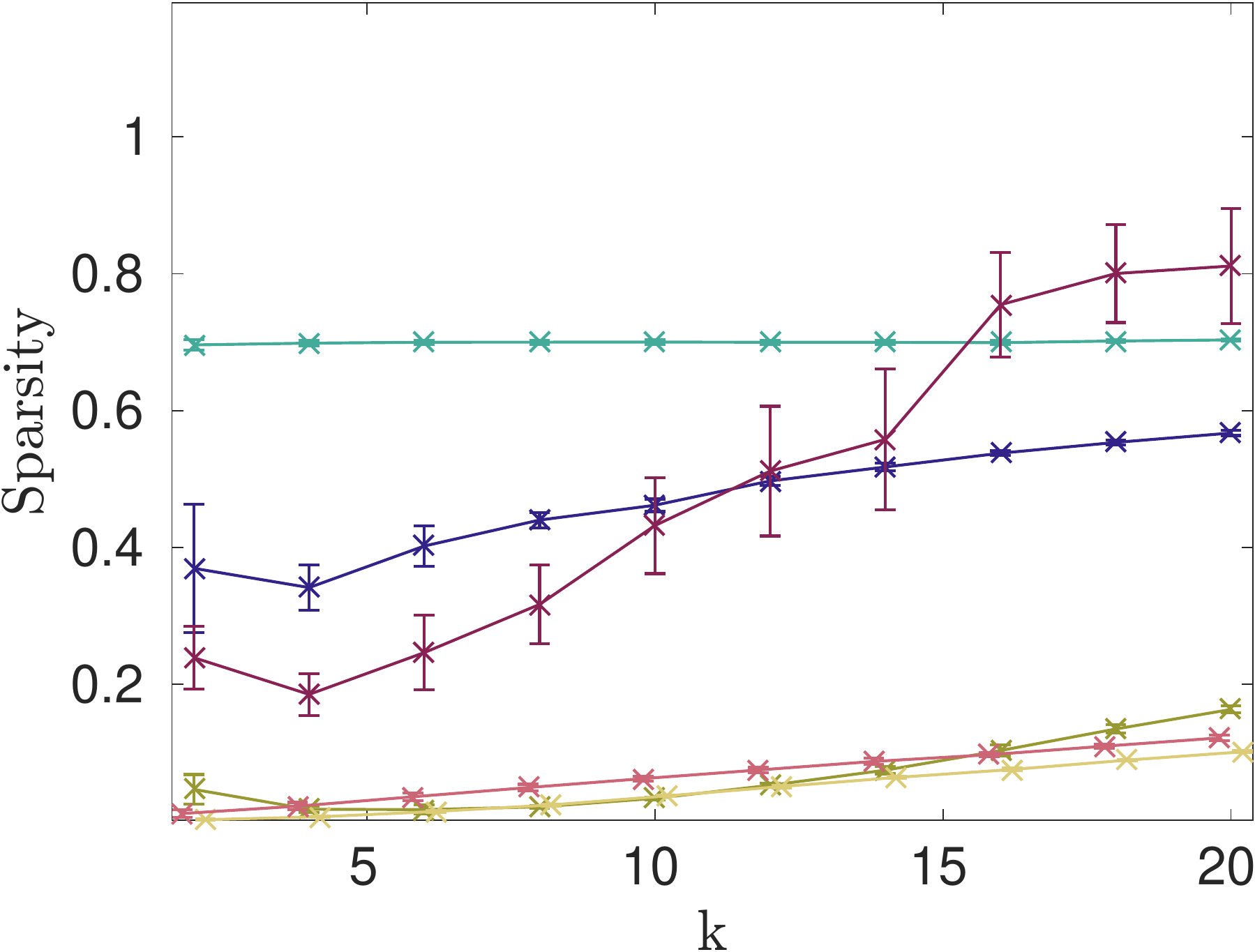}%
       \label{dim:frobhn:sparse}%
       }
       \hspace{\subfigspace}
           \subfigure[Varying rank test with 10\% noise and 60\% factor density.] {%
       \includegraphics[width=\subfigwidth]{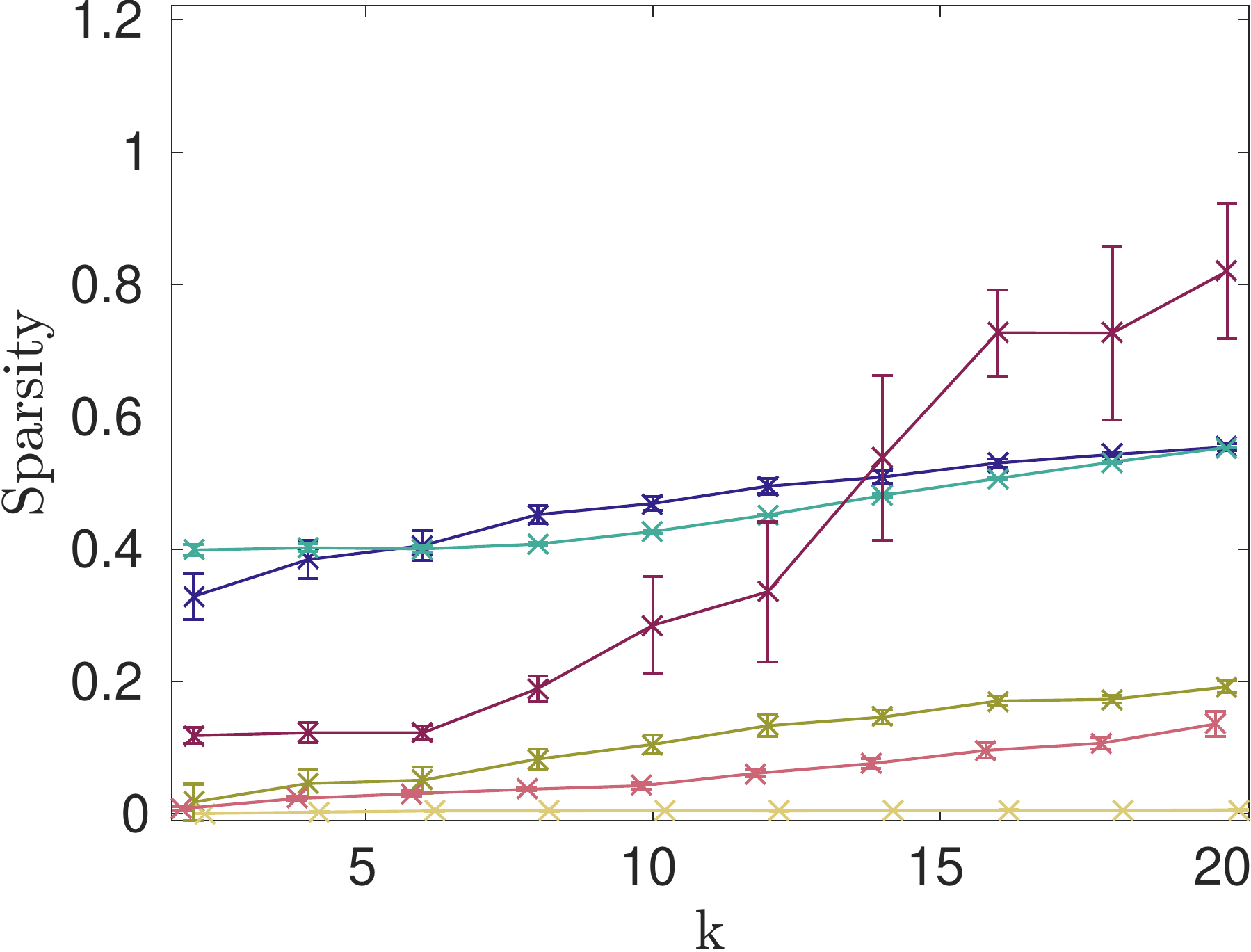}%
       \label{dim:frobhd:sparse}%
     }
     \caption{\textbf{Sparsity (fraction of zeroes) of the factor matrices for synthetic data with tropical noise.} $x$-axis is the parameter varied and $y$-axis is the sparsity of the factors. The markers are averages of 10 random matrices and the width of the error bars is twice the standard deviation.}
       \label{fig:synth:sparsity}
     \end{figure}

\paragraph{Varying Gaussian noise.}
Here we investigate how the algorithms respond to different levels of Gaussian noise, which was varied from 0 to 0.14 with increments of 0.01. A level of noise is a standard deviation of the Gaussian noise used to generate the noise matrix as described earlier. The factor density was kept at 50\%. The results are given on Figure~\ref{fig:err:noise} (reconstruction error) and Figure~\ref{fig:sparsity:noise} (sparsity of factors).

Here \Cancer is generally the best method in reconstruction error, and second in sparsity only to \Capricorn. The only time it loses to any method is when there is no noise, and \Capricorn obtains a perfect decomposition. This is expected since \Capricorn is by design better at spotting pure subtropical structure.

\paragraph{Varying density with Gaussian noise.}
In this experiment we studied what effects the  density of factor matrices used in data generation has on the algorithms' performance. For this purpose we varied the density from 10\% to 100\% with increments of 10\% while keeping the other parameters fixed. There are two versions of this experiment, one with low noise level of 0.01 (Figures~\ref{fig:err:density} and \ref{fig:sparsity:density}), and a more noisy case at 0.08 (Figures~\ref{fig:err:densityHN} and~\ref{fig:sparsity:densityHN}). 

\Cancer provides the least reconstruction error in this experiment, being clearly the best until the density is $0.7$, from which point on it is tied with \SVD and the NMF-based methods (the only exception being the least-dense high-noise case, where \ALSR obtains a slightly better reconstruction error). \Capricorn is the worst by a wide margin, but this is not surprising, as the data does not follow its assumptions. On the other hand, \Capricorn does produce generally the sparsest factorization, but these are of little use given its bad reconstruction error. \Cancer produces the sparsest factors from the remaining methods, except in the first few cases where \ALSRfive is sparser (and worse in reconstruction error), meaning that \Cancer produces factors that are both the most accurate and very sparse.

\paragraph{Varying rank with Gaussian noise.}
The purpose of this test is to study the performance of algorithms on data of different max-times ranks. We varied the true rank of the data from 2 to 20 with increments of 2. The factor density  was fixed at 50\% and Gaussian noise at 0.01. The results are shown on Figure~\ref{fig:err:dim} (reconstruction error) and Figure~\ref{fig:sparsity:dim} (sparsity of factors). The results are similar to those considered above, with \Cancer returning the most accurate and second sparsest factorizations.

\paragraph{Optimizing the Jensen--Shannon divergence.}
By default \Cancer  optimizes the Frobenius reconstruction error, but it can be replaced by an arbitrary additive cost function. We performed experiments with Jensen--Shannon divergence, which is given by the formula
\begin{equation} \label{obj:JS}
J(\mA, \mB) = \sum_{ij} \mA_{ij}\log\left(\frac{2\mA_{ij}}{\mA_{ij}+\mB_{ij}}\right) + \mB_{ij}\log\left(\frac{2\mB_{ij}}{\mA_{ij}+\mB_{ij}}\right)\;.
\end{equation}
It is easy to see that \eqref{obj:JS} is an additive function, and hence can be plugged into \Cancer.
Figure~\ref{fig:synth:reconstruct:js}, shows how this version of \Cancer compares to other methods. The setup is the same as in the corresponding experiments on Figure~\ref{fig:synth:err}, except that we have removed \ALSRfive because of its overall bad performance. In all these experiments it is apparent that this version of \Cancer is inferior to that optimizing the Frobenius error, but is generally on par with \SVD and NMF-based methods. Also for the varying density test (Figure~\ref{density:js}) it produces better reconstruction errors than \SVD and all the NMF methods, until the density reaches 50\%, after which they become tied.

\begin{figure}[tp]
  \centering  
 \subfigure[Varying noise with density 50\%]{%
    \includegraphics[width=\subfigwidth]{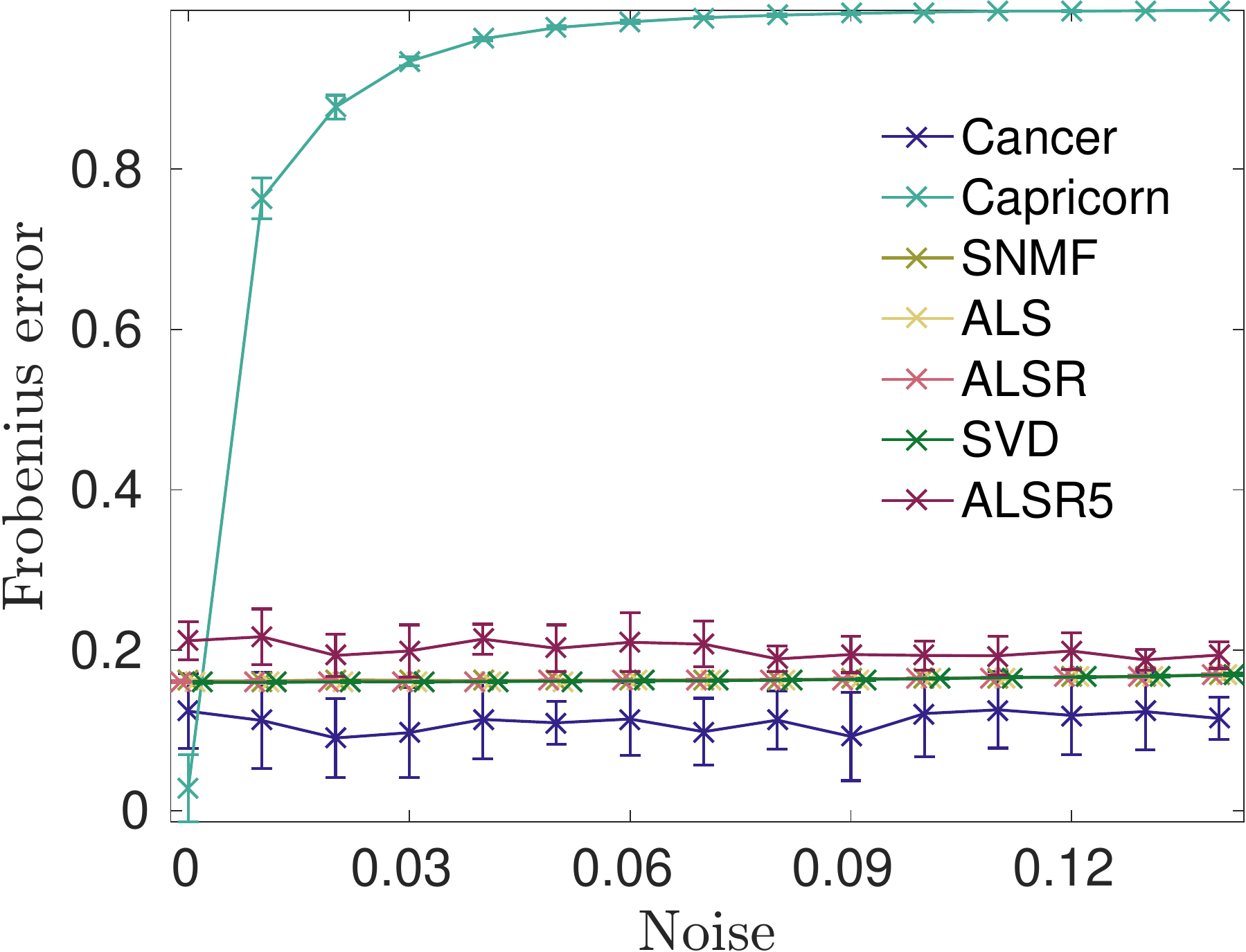}%
    \label{fig:err:noise}%
  }  
   \hspace{\subfigspace}
    \subfigure[Varying density with low Gaussian noise]{%
      \includegraphics[width=\subfigwidth]{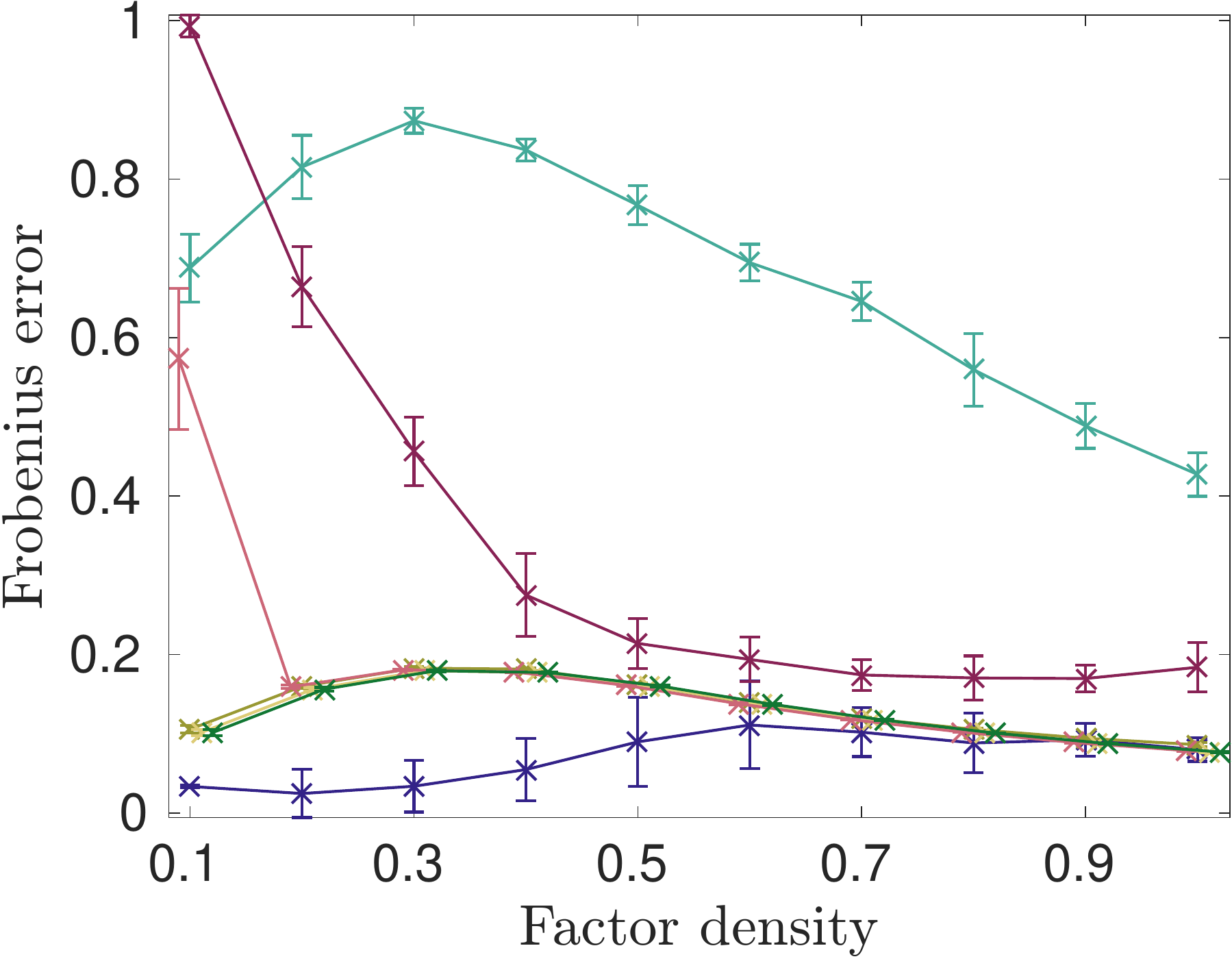}%
     \label{fig:err:density}%
   }
  \\
      \subfigure[Varying density with high Gaussian noise]{%
    \includegraphics[width=\subfigwidth]{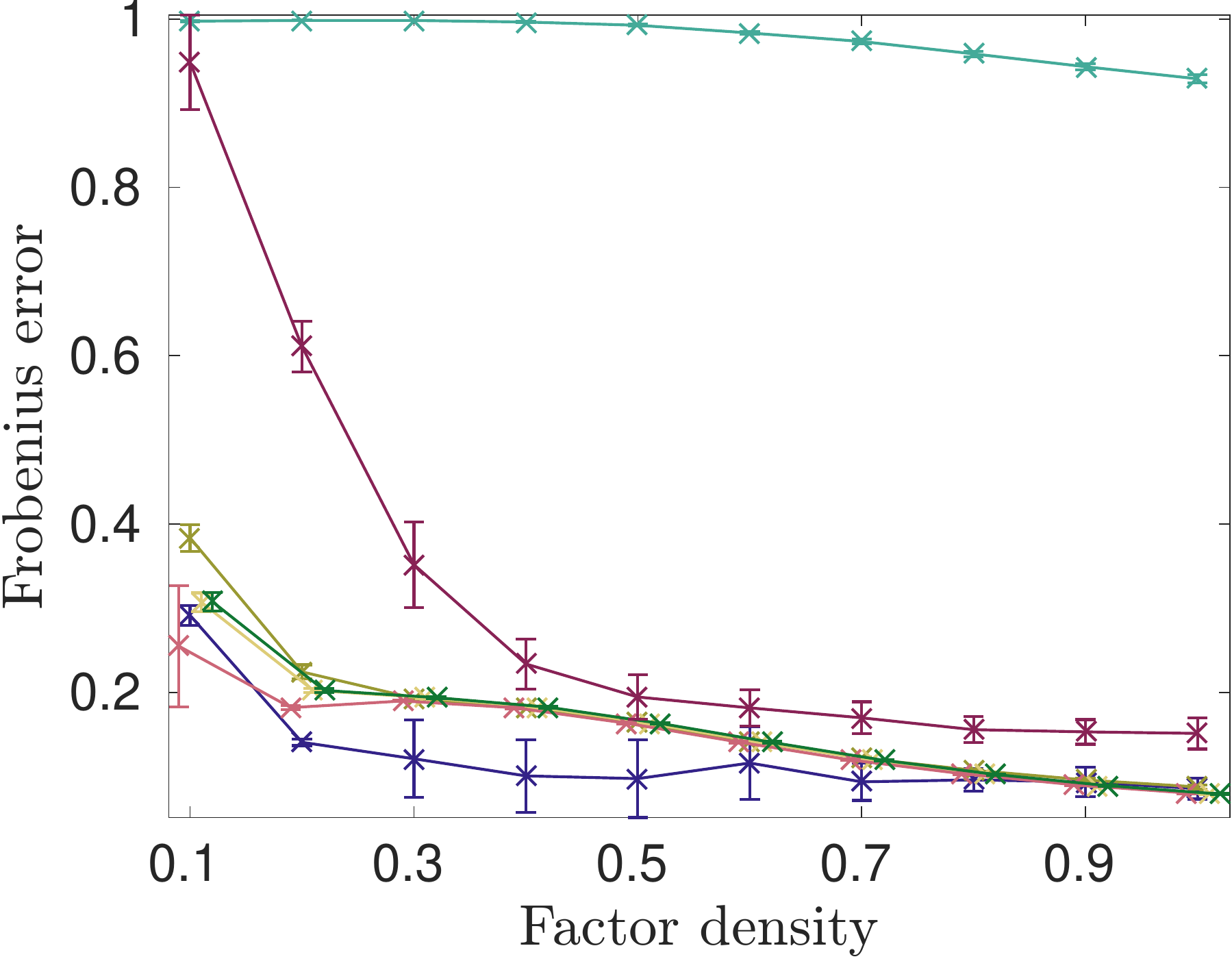}%
    \label{fig:err:densityHN}%
  }
    \hspace{\subfigspace}
  \subfigure[Varying rank; 50\% density and low Gaussian noise]{%
    \includegraphics[width=\subfigwidth]{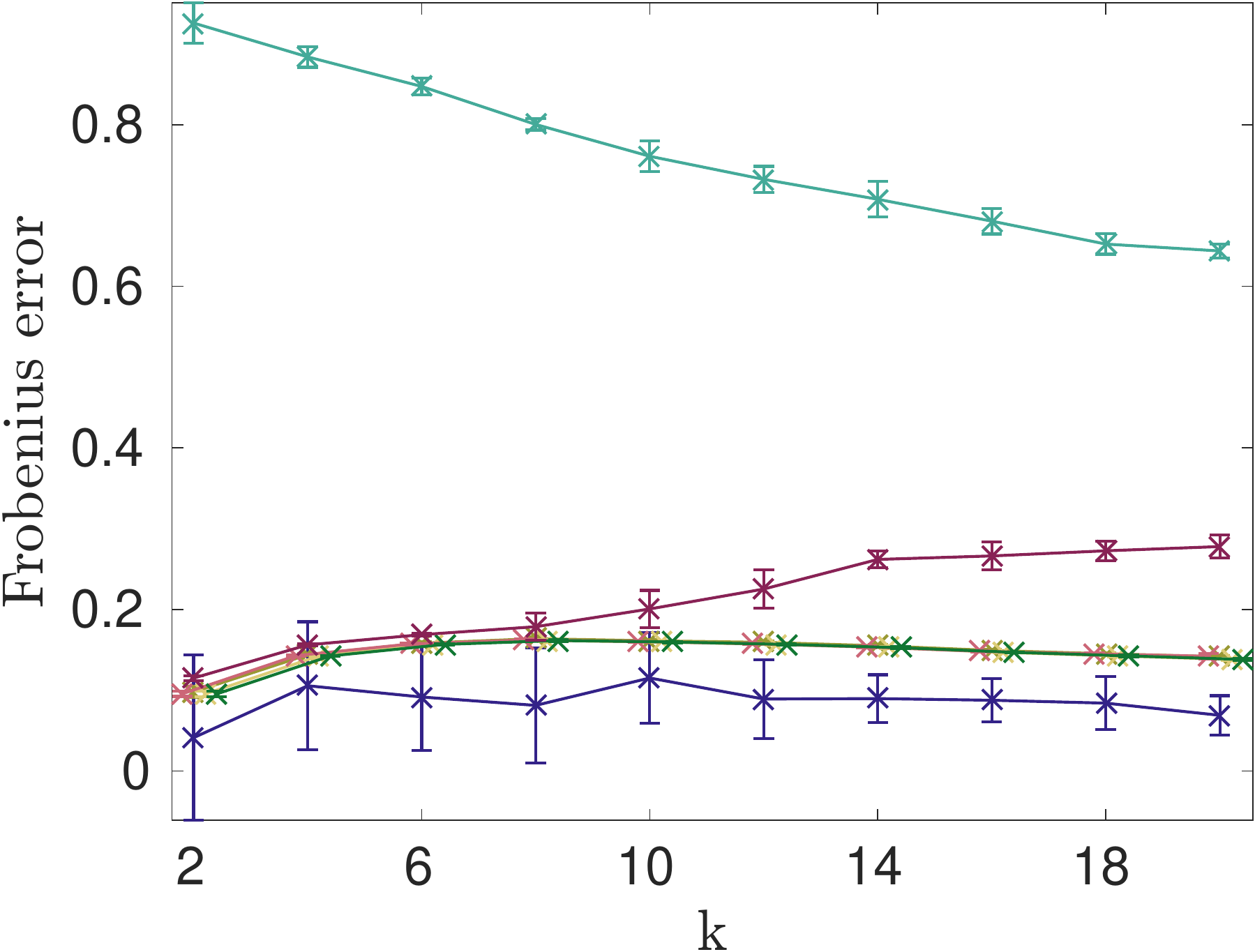}%
    \label{fig:err:dim}%
  }
  \caption{\textbf{Reconstruction error (Frobenius norm) for synthetic data with Gaussian noise noise.} The markers are averages of 10 random matrices and the width of the error bars is twice the standard deviation.}
  \label{fig:synth:err}
\end{figure}

\begin{figure}
    \centering
    \subfigure[Varying noise with density 50\%]{%
    \includegraphics[width=\subfigwidth]{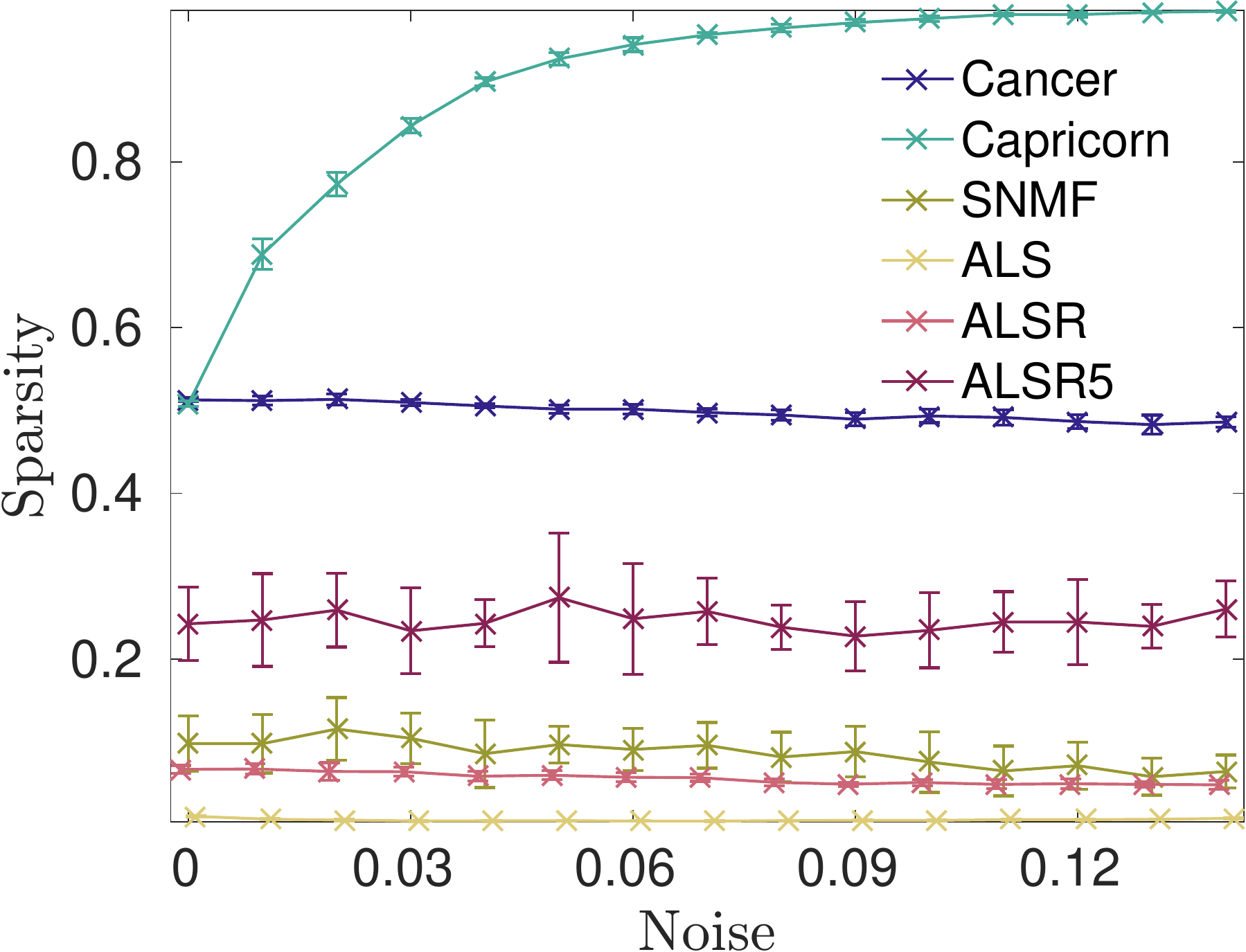}%
    \label{fig:sparsity:noise}%
  }  
  \hspace{\subfigspace}  
    \subfigure[Varying density with low Gaussian noise]{%
    \includegraphics[width=\subfigwidth]{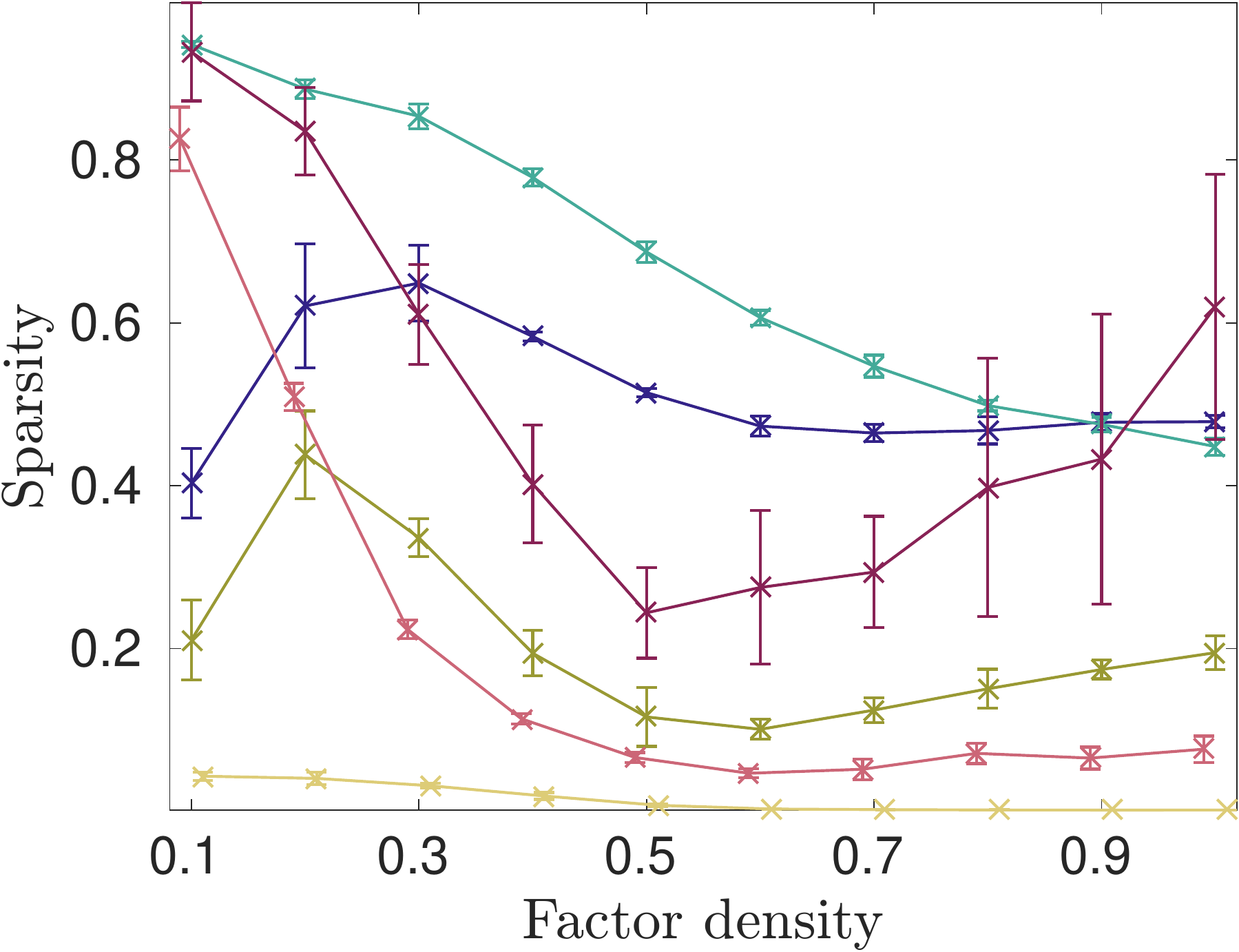}%
    \label{fig:sparsity:density}
  }
  \\
  \hspace{\subfigspace}
  \subfigure[Varying density with high Gaussian noise]{%
    \includegraphics[width=\subfigwidth]{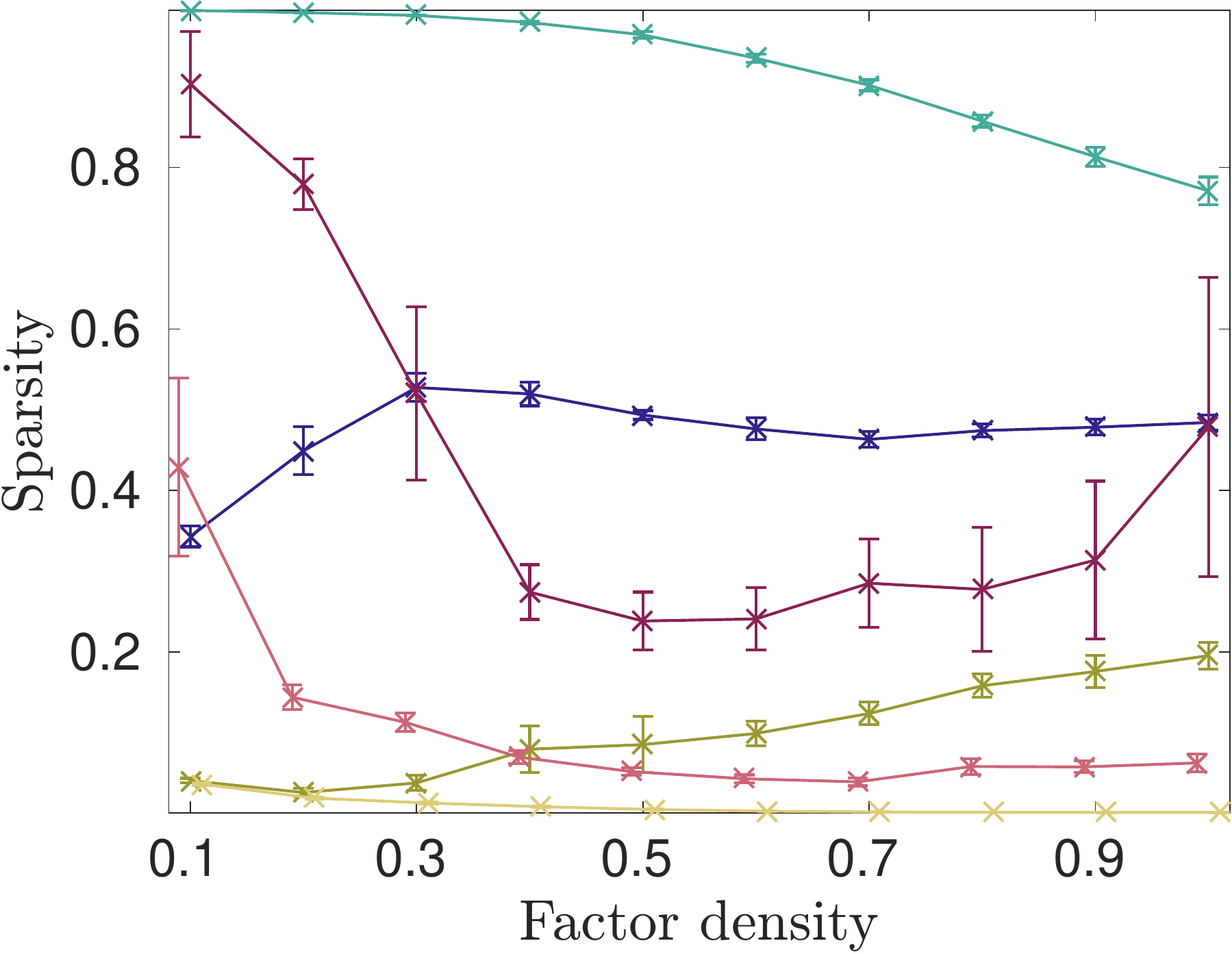}%
    \label{fig:sparsity:densityHN}%
  }
  \subfigure[Varying rank; 50\% density and low Gaussian noise]{%
    \includegraphics[width=\subfigwidth]{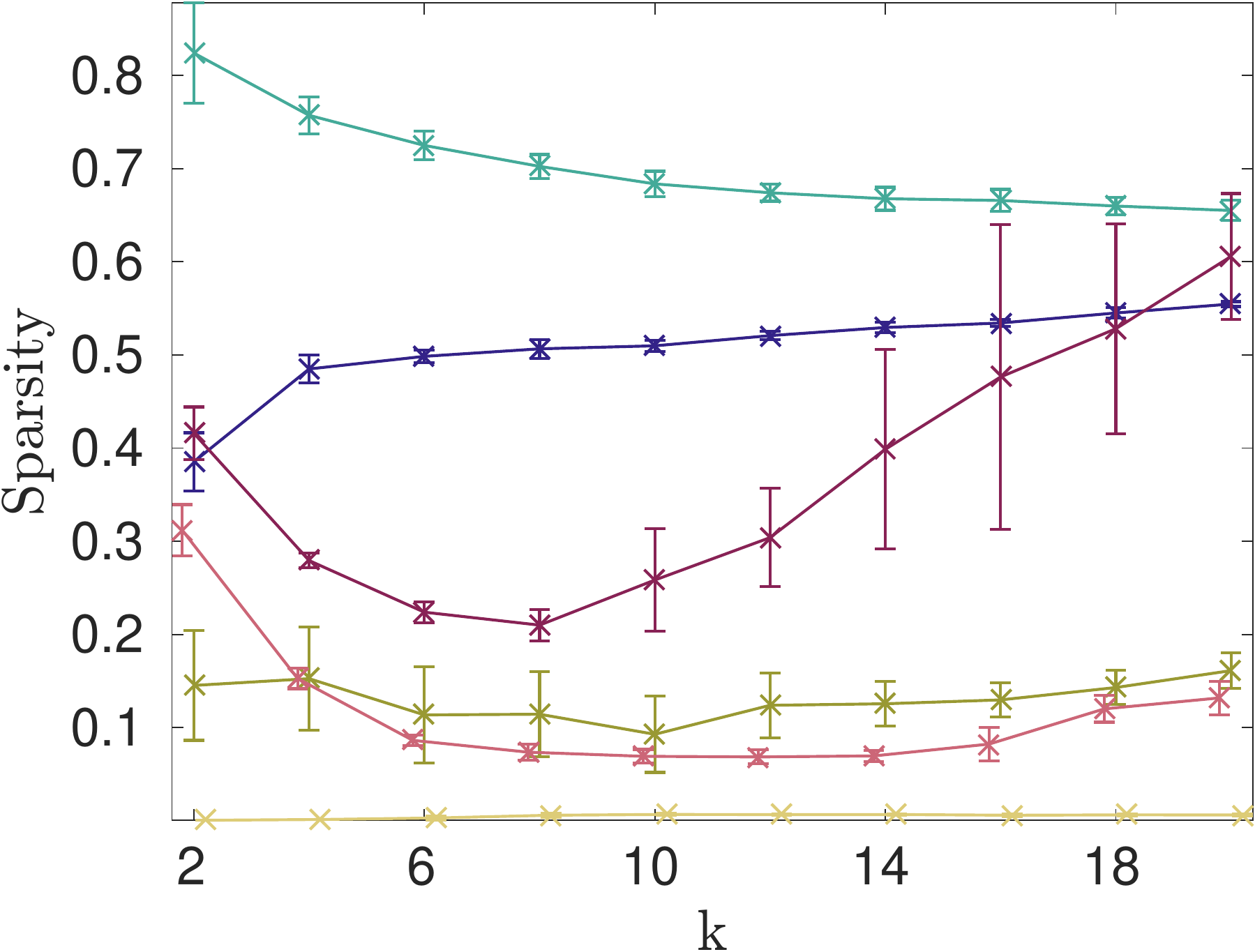}%
    \label{fig:sparsity:dim}%
  }
  \caption{\textbf{Sparsity (fraction of zeroes) of the factor matrices for synthetic data with Gaussian noise.} The markers are averages of 10 random matrices and the width of the error bars is twice the standard deviation.}
  \label{fig:synth:sparsity:gaussian}
\end{figure}

\begin{figure} [tp]  
  \centering
         \subfigure[Varying noise test.] {%
       \includegraphics[width=\subfigwidth]{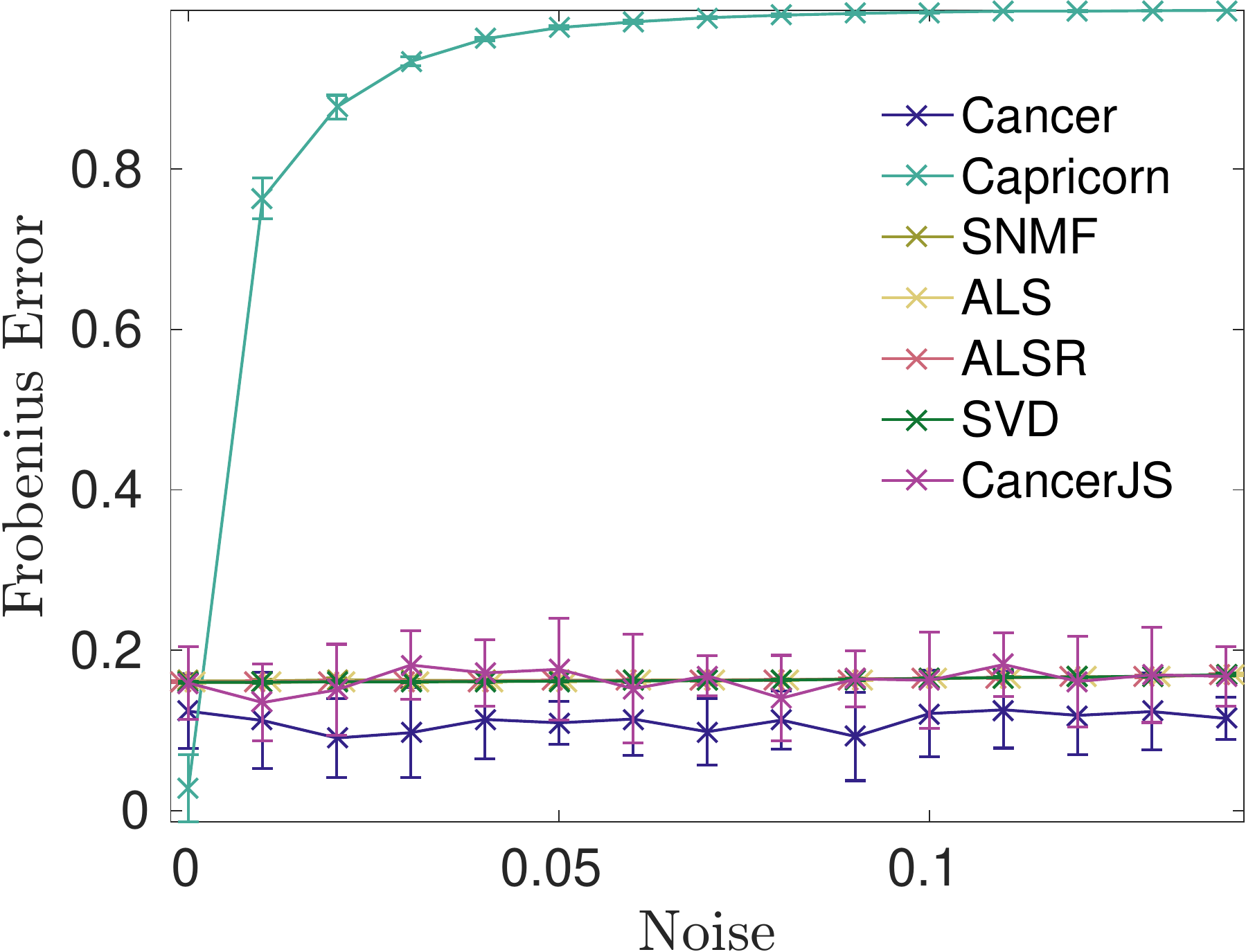}%
       \label{noise:js}%
       }
       \hspace{\subfigspace}
  \subfigure[Varying density test.] {%
       \includegraphics[width=\subfigwidth]{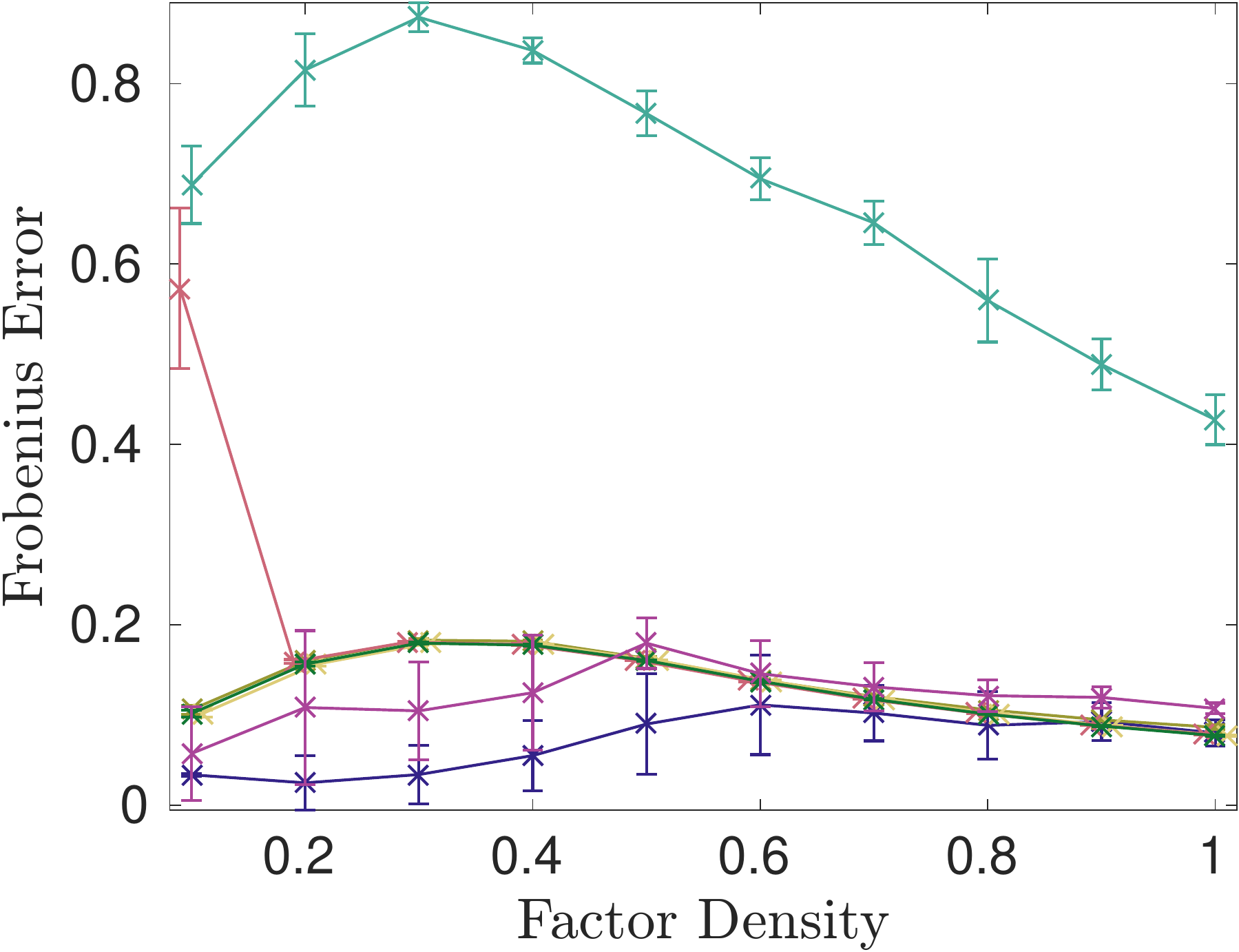}%
       \label{density:js}%
       }
       \hspace{\subfigspace}
           \subfigure[Varying rank test with 10\% noise and 30\% factor density.] {%
       \includegraphics[width=\subfigwidth]{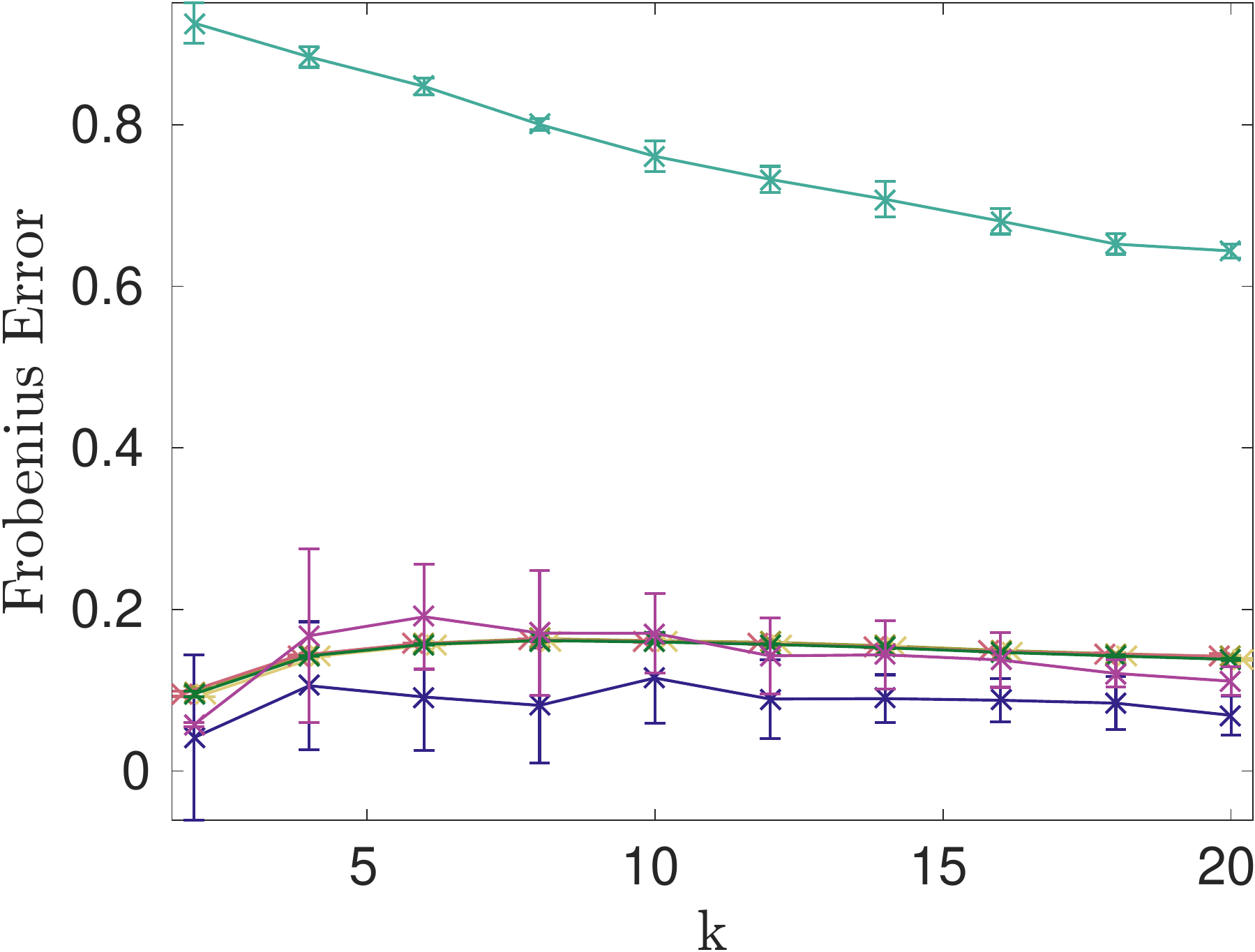}%
       \label{dim:js}%
       }
       \caption{\textbf{Comparison of \Cancer with Jensen--Shannon objective and other methods on synthetic data with Gaussian noise.} $x$-axis is the parameter varied and $y$-axis is the relative Frobenius error. All results are averages over 10 random matrices and the width of the error bars is twice the standard deviation. }
       \label{fig:synth:reconstruct:js}
     \end{figure}

\paragraph{Prediction.}
In this experiment we choose a random holdout set and remove it from the data (elements of this set are marked as missing values). We then try to learn the structure of the data from its remaining part using \Capricorn and \WNMF, and finally test how well they predict the values inside the holdout set. All input matrices are integer-valued and since the recovered data produced by the algorithms can be continuous-valued, we round it to the nearest integer. The quality of the prediction is measured as the fraction of correct values in the hold-out set, and the results are reported in Figure~\ref{fig:synth:predict}.  
\begin{figure} [tp]
  \centering
    \includegraphics[width=\subfigwidth]{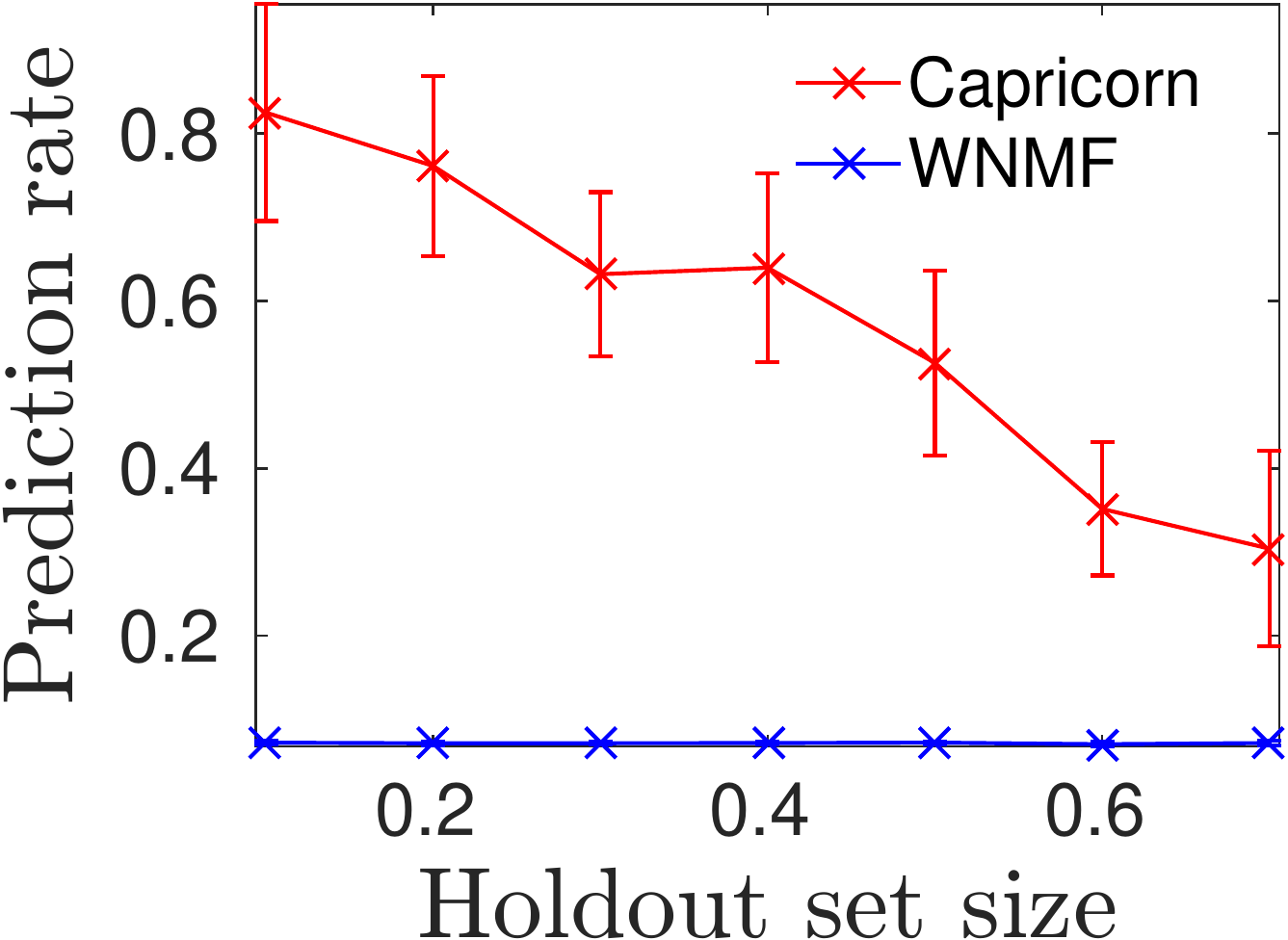}
      \caption{\textbf{Prediction rate on synthetic data}. $x$-axis represents the size of the holdout set and $y$-axis is the correct prediction rate (higher is better). All results are averages over 10 random matrices and the width of the error bars is twice the standard deviation.}\label{fig:synth:predict}
  \end{figure}
It is easy to see that as the fraction of held-out data increases, \Capricorn's results get worse, as expected, but it still is consistently better than \WNMF that does not seem to be able to recover any specific structure.

\paragraph{Discussion.}
The synthetic experiments confirm that both \Capricorn and \Cancer are able to recover matrices with max-times structure. The main practical difference between then is that \Capricorn is designed to handle the tropical (flipping) noise, while \Cancer is meant for the data that is perturbed with white (Gaussian) noise. While \Capricorn is clearly the best method when the data has only the flipping noise -- and is capable of tolerating very high noise levels -- its results deteriorate when we apply Gaussian noise. Hence, when the exact type of noise is not known a priori, it is advisable to try both methods. It is also important to note that \Cancer is actually a framework of algorithms as it can optimize various objective. In order to demonstrate that, we performed experiments with Jensen--Shannon divergence as objective and obtained results that are, while inferior to \Cancer that optimizes the Frobenius error, still slightly better than the rest of the algorithms. Overall we can conclude that \SVD and the NMF-based methods generally cannot recover the structure from subtropical data, that is,  we cannot use existing methods as a substitute to find the max-times structure neither for the reconstruction nor for the prediction tasks.

\subsection{Real-world experiments.}
\label{sec:real-world-exper}
The main purpose of the real-world experiments is to study to which extend \Capricorn and \Cancer can find max-times structure from various real-world data sets. Having established with the synthetic experiments that both algorithms are capable of finding the structure when it is present, here we look at what kind of results they obtain in the real-world data. 

It is probably unrealistic to expect real-world data sets to have ``pure'' max-times structure, as in the synthetic experiments. Rather, we expect \SVD to be the best method (in reconstruction error's sense), and our algorithms to obtain reconstruction error comparable to the NMF-based methods. We will also verify that the results from the real-world data sets are intuitive.

\subsubsection*{The datasets} 
\label{sec:real:data}
\BasLP represents a linear program.\!\footnote{Submitted to the matrix repository by Csaba Meszaros.} It is available from the University of Florida Sparse Matrix Collection\footnote{\url{http://www.cise.ufl.edu/research/sparse/matrices/}, accessed 18 July 2017} \citep{davis11university}. 

\Trec  is a  brute force disjoint product matrix in tree algebra on $n$ nodes.\!\footnote{Submitted by Nicolas Thiery.} It can be obtained from the same repository as \BasLP.

\Worldclim was obtained from the global climate data repository.\!\footnote{The raw data is available at \url{http://www.worldclim.org/}, accessed 18 July 2017.} It describes historical climate data across different geographical locations in Europe. Columns represent minimum, maximum, and average temperatures and precipitation, and rows are \by{50}{50} kilometer squares of land where measurements were made. We preprocessed every column of the data by first subtracting its mean,  dividing by the standard deviation, and then subtracting its minimum value, so that the smallest value becomes 0.

\NPAS is a nerdiness personality test that uses different attributes to determine the level of nerdiness of a person.\!\footnote{Tha dataset can be obtained on the online personality website \url{http://personality-testing.info/_rawdata/NPAS-data.zip}, accessed 18 July 2017.} It contains answers by 1418 respondents to a set of 36 questions that asked them to self-assess various statements about themselves on a scale of 1 to 7. We preprocessed \NPAS  analogously to \Worldclim.

\Eigenfaces is a subset of the Extended Yale Face collection of face images~\citep{georghiades2000few}. It consists of \by{32}{32} pixel images under different lighting conditions. We used a preprocessed data by Xiaofei He et al.\!\footnote{\url{http://www.cad.zju.edu.cn/home/dengcai/Data/FaceData.html}, accessed 18 July 2017} We selected a subset of pictures with lighting from the left and then preprocessed the input matrix by first subtracting from every column its smallest element and then dividing it by its standard deviation.

\News is a subset of the {20Newsgroups} dataset,\!\footnote{\url{http://qwone.com/~jason/20Newsgroups/}, accessed 18 July 2017} containing the usage of 800 words over 400 posts for 4 newsgroups.\!\footnote{The authors are  grateful to Ata Kab{\'a}n for pre-processing the data, see~\cite{miettinen09matrix}.} Before running the algorithms we represented the dataset as a TF-IDF matrix, and then scaled it by dividing each entry by the greatest entry in the matrix.

\HPI is a land registry house price index.\!\footnote{Available at \url{https://data.gov.uk/dataset/land-registry-house-price-index-background-tables/}, accessed 18 July 2017} Rows represent months, columns are locations, and entries are residential property price indices. We preprocessed the data by first dividing each column by its standard deviation and then subtracting its minimum, so that each column has minimum 0.

\Movielense is a collection of user ratings for a set of movies. The original dataset\footnote{Available at \url{http://grouplens.org/datasets/movielens/100k/}, accessed 18 July 2017} consists of 100000 ratings  from 1000 users on 1700 movies, with ratings ranging from 1 to 5. In order to be able to perform cross-validation on it, we had to preprocess \Movielense by removing users that rated fewer than 10 movies and movies that were rated less than 5 times. After that we were left with 943 users, 1349 movies and 99287 ratings. 

The basic properties of these data sets are listed in Table~\ref{tab:real:specs_all}. 

\setlength{\tabcolsep}{0.5em}
   \begin{table}[tb]
   \centering
   \caption{Real world datasets properties.}
   \label{tab:real:specs_all}
   \begin{tabular}{@{}lRRR@{}}
     \toprule
     Algorithm & $Rows$ & $Columns$ & $Density$ \\
     \midrule
     \BasLP  & 9825 & 5411  & 1.1\%     \\
     \Trec  & 2726 & 551  & 10.0\%     \\
     \Worldclim  & 2575 & 48  & 99.9\%     \\
     \NPAS         & 1418 & 36  & 99.6\%     \\
     \Eigenfaces & 1024 & 222 & 97.0\%     \\
     \News       & 400  & 800 & 3.5\%      \\
     \HPI            & 253  & 177 & 99.5\%     \\
     \Movielense  & 943 & 1349  & 7.8\%     \\
     \bottomrule
   \end{tabular}
 \end{table}

 \subsubsection*{Quantitative results: reconstruction error, sparsity, and convergence}
 \label{sec:real:quantitative}
 The following experiments are meant to test \Cancer and \Capricorn, and how they compare versus other methods, such as \SVD and NMF. Table~\ref{tab:real:world:error} provides the relative Frobenius reconstruction errors for various real-world data sets. We omitted \ALSRfive from these experiments due to its bad performance with the synthetic data. \SVD is, as expected, consistently the best method. Somewhat surprisingly, Hoyer's \SNMF is usually the second-best method, even though it did not show any advantage over other methods in the synthetic experiments. \Cancer is usually the third-best method (with the exception of \News and \NPAS), and often very close to \SNMF in reconstruction error. Overall, it seems \Cancer is capable of finding max-times structure that is comparable to what NMF-based methods provide. Consequently, we can study the max-times structure found by \Cancer, knowing that it is (relatively) accurate. On the other hand \Capricorn has a high reconstruction error. The discrepancy between \Cancer's and \Capricorn's results indicates that the datasets used cannot be represented using ``pure'' subtropical structure. Rather they are either a mix of NMF and subtropical patterns or have relatively high levels of continuous noise.

 \begin{table}[tb] 
   \centering
   \caption{Reconstruction error for various real-world datasets.}
\label{tab:real:world:error}
   \begin{tabular}{@{}lRRRRR@{}}
     \toprule
                      & \text{\Worldclim} & \text{\NPAS} & \text{\Eigenfaces} & \text{\News} & \text{\HPI}  \\
     $k=$         & 10      & 10      & 40      & 20      & 15 \\
     \midrule
     \Cancer     & 0.071 & 0.240 & 0.204 & 0.556 & 0.027     \\
     \Capricorn  & 0.392 & 0.395 & 0.972 & 0.987 & 0.217     \\
     \SNMF       & 0.046 & 0.225 & 0.178 & 0.546 & 0.023     \\
     \ALS        & 0.087 & 0.227 & 0.313 & 0.538 & 0.074     \\
     \ALSR       & 0.122 & 0.226 & 0.294 & 1.000 & 0.045     \\
     \SVD        & 0.025 & 0.209 & 0.140     & 0.533 & 0.015     \\
     \bottomrule
   \end{tabular}
 \end{table}

The sparsity of the factors for real-world data is presented in Table~\ref{tab:real:sparsity_all}, except for \SVD. Here, \Cancer often returns the second-sparsest factors (being second only to \Capricorn), but with \News and \HPI, \ALSR obtains sparser decompositions. 

  \begin{table}[tb]
   \centering
   \caption{Factor sparsity for various real-world datasets.}
   \label{tab:real:sparsity_all}
   \begin{tabular}{@{}lRRRRR@{}}
     \toprule
                      & \text{\Worldclim} & \text{\NPAS} & \text{\Eigenfaces} & \text{\News} & \text{\HPI}  \\
     $k=$         & 10      & 10      & 40      & 20      & 15 \\
     \midrule
     \Cancer     & 0.645 & 0.528 & 0.571 & 0.812 & 0.422     \\
     \Capricorn  & 0.795 & 0.733 & 0.949 & 0.991 & 0.685     \\
     \SNMF       & 0.383 & 0.330 & 0.403 & 0.499 & 0.226     \\
     \ALS        & 0.226 & 0.120 & 0.434 & 0.513 & 0.331     \\
     \ALSR       & 0.275 & 0.117 & 0.480 & 1.000 & 0.729     \\
     \bottomrule
   \end{tabular}
 \end{table}

We also studied the convergence behavior of \Cancer using some of the real-world data sets. The results can be seen in Figure~\ref{fig:convergence}, where we plot the relative error with respect to the iterations over the main for-loop in \Cancer. As we can see, in both cases \Cancer has obtained a good reconstruction error already after few full cycles, with the remaining runs only providing minor improvements. We can deduce that \Cancer reaches quickly an acceptable solution.

 \begin{figure*}[tp]
  \centering  
 \subfigure[\NPAS]{
    \includegraphics[width=\subfigwidth]{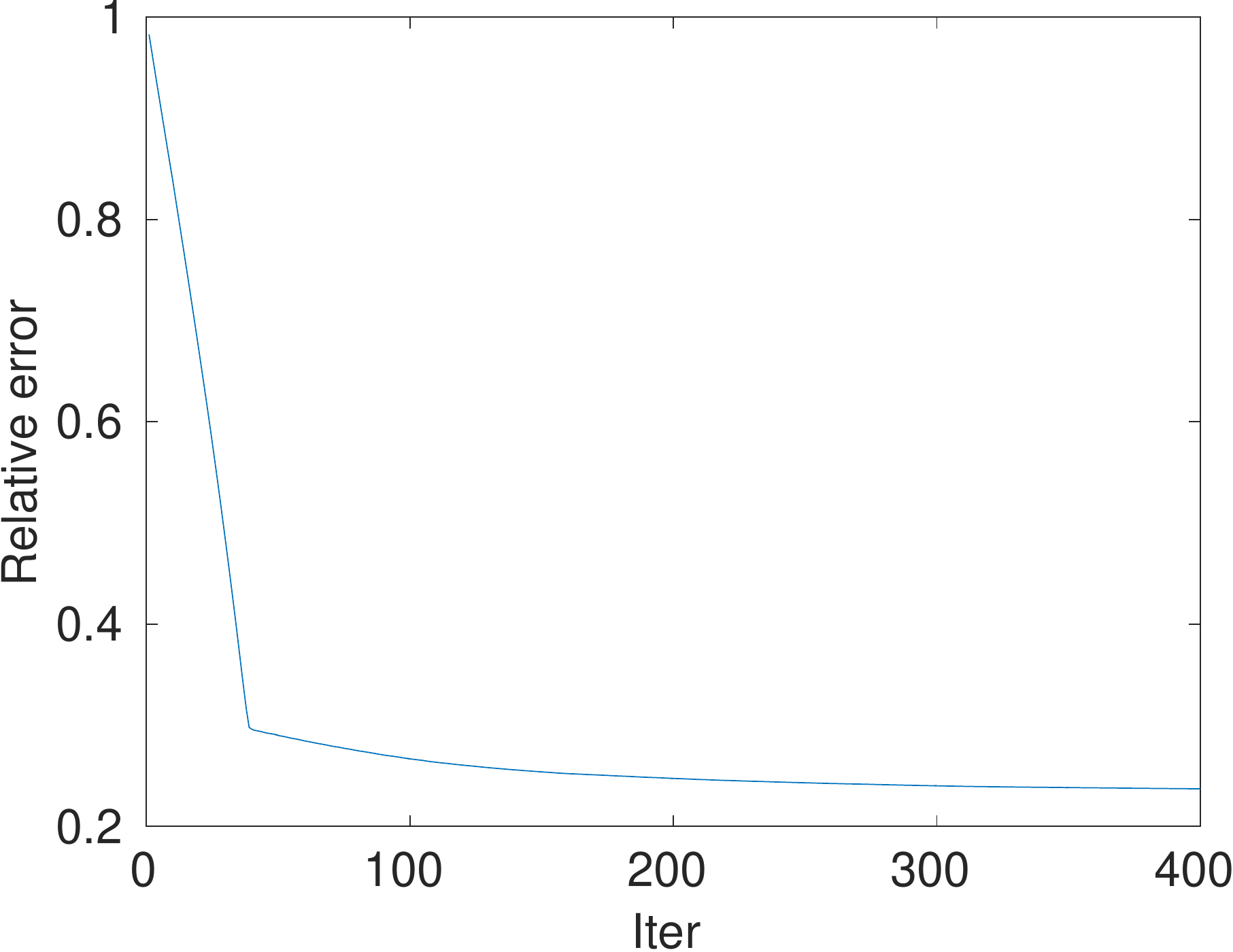}
    \label{noise}
  }  
  \hspace{\subfigspace}
   \subfigure[\HPI]{
     \includegraphics[width=\subfigwidth]{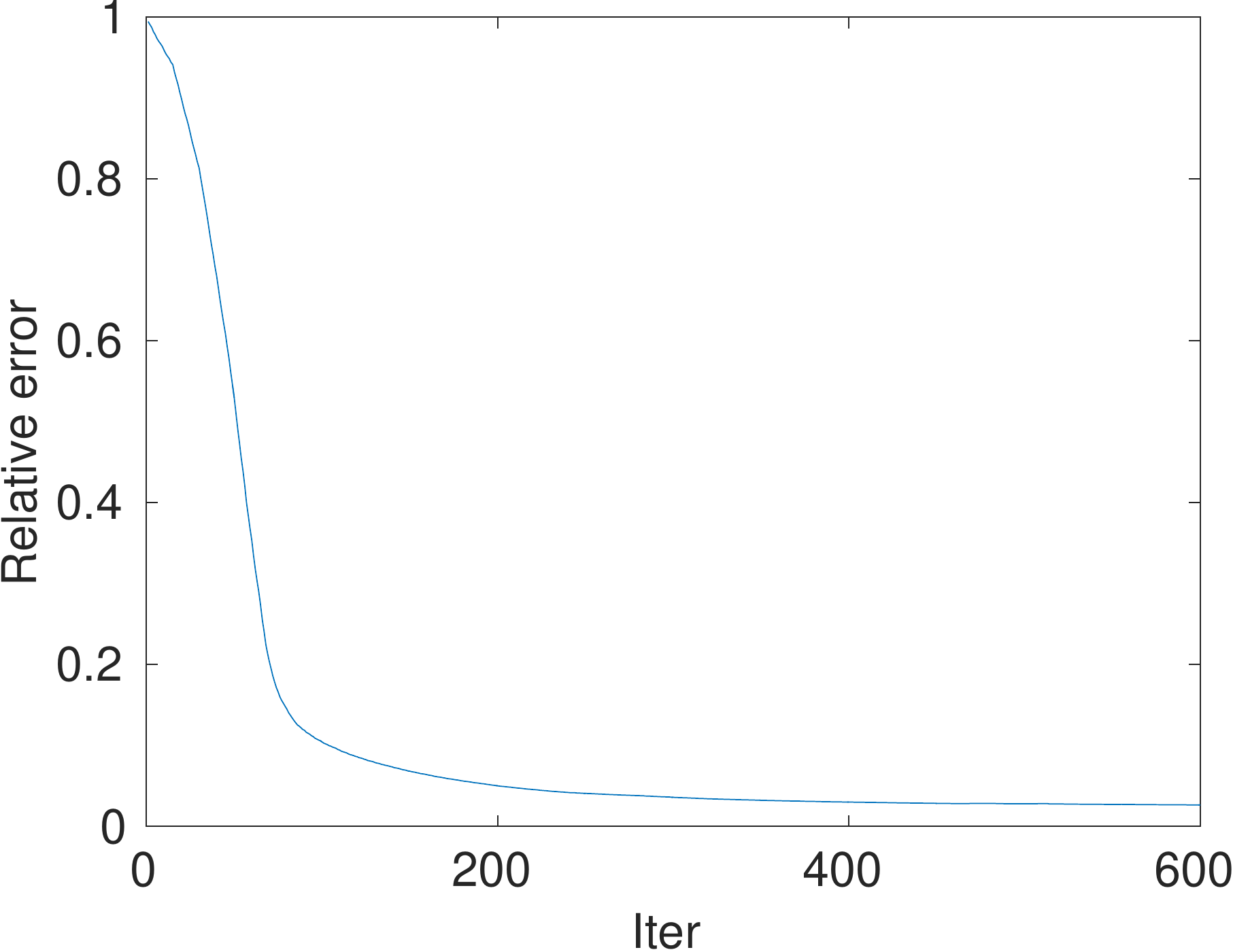}
    \label{density}
  }  
  \caption{Convergence rate of \Cancer for two real-world datasets. Each iteration is a single run of \UpdateBlock, that is if a factorization has rank $k$, then one full cycle would correspond to $k$ iterations.}
  \label{fig:convergence}
\end{figure*}



\subsubsection*{Prediction}
\label{sec:real:prediction}

Here we investigate how well both \Capricorn and \Cancer can predict missing values in the data.

In order to test \Capricorn, we ran missing value prediction tests on \BasLP and \Trec datasets, and compare it against \NMF, \WNMF, and \SVD. The setup is as follows. A random holdout set is chosen that comprises 10\% of the nonzero elements and then removed from the data. Since the input matrices are integer valued, we round the output of the algorithms to the nearest integer and report the fraction of correctly predicted values. There are two versions of this experiment -- one where all elements in the data are taken into account and one where zero entries are ignored, that is, they do not contribute to the error. The motivation for this test is that \Capricorn always aims to extract subtropical patterns, sometimes even at the expense of covering zeros with nonzero values. We therefore want to see how well it performs when only the ``significant'' part of the data is counted. It is worth noting though that while \Capricorn and \WNMF have an option to ignore certain entries in an input matrix, \NMF does not. Hence the \NMF algorithm is at a disadvantage here, though we still show its result for completeness. The results for both prediction experiments where zeros ``count'' and ``don't count'' are shown in Table~\ref{tab:real:accuracy_all}, left and right, respectively. 
In both cases \WNMF is the best method, whereas \Capricorn is normally the second-best. As expected, \Capricorn's results improve greatly when zero elements are ignored.

\begin{table}[tb]
  \centering
  \caption{Prediction accuracy on \BasLP and \Trec datasets. Left: accuracy is computed over all entries. Right: accuracy is computed over the non-zero entries.}\label{tab:real:accuracy_all}
  \begin{tabular}{@{}lRR@{}}
    \toprule
    Algorithm & $Bas1LP$ & \text{Trec12} \\
    \midrule
    \Capricorn &   74.0        &   19.8        \\
    \NMF &     23.4      &       18.3          \\
    \WNMF &   85.2       &      39.9          \\
    \SVD &   28.2        &        20.5         \\
    \bottomrule
  \end{tabular}
  \hspace*{3em}
  \begin{tabular}{@{}lRR@{}}
    \toprule
    Algorithm & $Bas1LP$ & $Trec12$ \\
    \midrule
    \Capricorn &   85.2        &   39.3        \\
    \NMF &     29.1      &       19.6          \\
    \WNMF &   93.1        &      49.8          \\
    \SVD &   29.1        &        22.5         \\
    \bottomrule
  \end{tabular}  
\end{table}


Next we conduct prediction experiments with \Cancer. We tested it on the \Movielense dataset and compared against \WNMF. The choice of \WNMF is motivated by its ability to ignore elements in the input data and its generally good performance on the previous tests. To get a more complete view on how good the predictions are, we report various measures of quality: Frobenius error, root mean square error (RMSE), reciprocal rank, Spearman's $\rho$, mean absolute error (MAE), Jensen--Shannon divergence (JS), optimistic reciprocal rank, and Kendall's $\tau$. The tests can be divided into two categories. The first one, which comprises Frobenius error, root mean square error, mean absolute error, and Jensen--Shannon divergence, aims to quantify the distance between the original data and the reconstructed matrix. The second group of tests finds the correlation between rankings of movies for each user. It includes Spearman's $\rho$, Kendall's $\tau$, reciprocal rank, and optimistic reciprocal rank. All these measures are well known, with perhaps only the reciprocal rank requiring some explanation. Let us first denote by $U$ the set of all users. In the following, for each user $u\in U$ we only consider the set of movies $M(u)$ that this user has rated that belong to the holdout set. The ratings by user $u$ induce a natural ranking on $M(u)$. On the other hand both \Cancer and \WNMF produce approximations $r'(u, m)$ to the true ratings $r(u, m)$, which also induce a corresponding ranking of the movies. The reciprocal rank is a convenient way of comparing the rankings obtained by the algorithms to the original one. For any user $u \in U$, denote by $H(u)$  a set of movies that this user ranked the highest (that is $H(u) = \lbrace m\in M(u) \, \vert \,  r(u, m) = \max_{m'\in M(u)} r(u, m') \rbrace$). The reciprocal rank for user $u$ is now defined as

\begin{equation} \label{recip:rank}
  RR(u) = \frac{1}{\min\limits_{m\in H} R(u, m)}\;,
  \end{equation}
where $R(u, m)$ is the rank of the movie $m$ within $M(u)$ according to the rating approximations given by the algorithm in question. Now the mean reciprocal rank is defined as the average of the reciprocal ranks for each individual user $MRR = \frac{1}{\abs{U}} \sum_{u\in U} RR(u)$. When computing the ranks $R(u, m)$, all tied elements receive the same rank, which is computed by averaging. That means that if, say, movies $m_1$ and $m_2$ have tied ranks of 2 and 3, then they both receive the rank of 2.5. An alternative way is to always assign the smallest possible rank. In the above example both $m_1$ and $m_2$ will receive rank 2. When ranks $R(u, m)$ are computed like this, the equation \eqref{recip:rank} defines the optimistic reciprocal rank. 

We perform standard cross-validation tests where a random selection of elements is chosen as a holdout set and removed from the data. The data has 943 users, each having rated from 19 to 648 movies. A holdout set is chosen by sampling uniformly at random 5 ratings from each user. We run the algorithms, while treating the elements from the holdout set as missing values, and then compare the reconstructed matrices to the original data. This procedure is repeated 10 times.

For each test, Table~\ref{tab:movielens} shows the mean and the standard deviation of the results of each algorithm. In addition we report the $p$-value based on the Wilcoxon signed-rank test. It shows if an advantage of one method over the other is statistically significant. We say that a method $A$ is significantly better than method $B$ if the $p$-value is $<0.05$. \Cancer is significantly better for the Frobenius error, root mean square error, mean absolute error, and Jensen--Shannon divergence. For the remaining tests the results are less clear, with \Cancer winning on both version of the reciprocal rank, and \WNMF being better on Spearman's $\rho$ and Kendall's $\tau$ tests. None of these results are statistically significant as the $p$-values are quite high. In summary, our experiments show that \Cancer is significantly better in tests that measure the direct distance between the original and the reconstructed matrices, whereas for the ranking experiments it is difficult to give any of the algorithms an edge.

\begin{table}
  \centering
  \caption{\label{tab:movielens}Comparison between the predictive power of \Cancer and \WNMF on the \Movielense data. The arrow after the value indicates whether higher or lower values are preferable. The $p$-values are computed using the Wilcoxon signed-rank test.} 
  \begin{tabular}{@{}lRRRR@{}}
    \toprule
    & \multicolumn{2}{c}{Frobenius} & \multicolumn{2}{c}{RMSE} \\ 
    \cmidrule{2-3} \cmidrule{4-5}
    & \text{value} (\downarrow) & \text{$p$-value} & \text{value} (\downarrow) & \text{$p$-value}\\
    \midrule
    \Cancer & \mathbf{0.2851}\pm0.003 & \multirow{2}{*}{$<0.0001$} & \mathbf{1.0724}\pm 0.013 & \multirow{2}{*}{$<0.0001$}\\
    \WNMF & 0.2969\pm0.003 & & 1.1169\pm 0.011 & \\
    \midrule
    & \multicolumn{2}{c}{Recip. rank} & \multicolumn{2}{c}{Spearman's $\rho$} \\
    \cmidrule{2-3} \cmidrule{4-5}
    & \text{value} (\uparrow) & \text{$p$-value} & \text{value} (\uparrow) & \text{$p$-value} \\
    \midrule
        \Cancer & \mathbf{0.7472}\pm 0.011 & \multirow{2}{*}{$0.0994$} & 0.3097\pm 0.016 & \multirow{2}{*}{$0.2124$}\\
        \WNMF & 0.7423\pm 0.009 & & \mathbf{0.3133}\pm 0.015 & \\
            \midrule
    & \multicolumn{2}{c}{MAE} & \multicolumn{2}{c}{JS} \\
    \cmidrule{2-3} \cmidrule{4-5}
    & \text{value} (\downarrow) & \text{$p$-value} & \text{value} (\downarrow) & \text{$p$-value} \\
    \midrule
        \Cancer & \mathbf{0.8158}\pm 0.008 & \multirow{2}{*}{$<0.0001$} & \mathbf{0.0198}\pm 0.001 & \multirow{2}{*}{$<0.0001$}\\
        \WNMF & 0.8503\pm 0.007 & & 0.0206\pm 0.000 & \\
                    \midrule
    & \multicolumn{2}{c}{Recip. rank opt.} & \multicolumn{2}{c}{Kendall's $\tau$} \\
    \cmidrule{2-3} \cmidrule{4-5}
    & \text{value} (\uparrow) & \text{$p$-value} & \text{value} (\uparrow) & \text{$p$-value} \\
    \midrule
        \Cancer & \mathbf{0.7472}\pm 0.011 & \multirow{2}{*}{$0.0994$} & 0.2685\pm 0.014 & \multirow{2}{*}{$0.2204$}\\
    \WNMF & 0.7423\pm 0.0093 & & \mathbf{0.2712}\pm 0.013 & \\
    \bottomrule
   \end{tabular}
\end{table}


\subsubsection*{Interpretability of the results}
\label{sec:real:interpretability}
The crux of using max-times factorizations instead of standard (nonnegative) ones is that the factors (are supposed to) exhibit the ``winner-takes-it-all'' structure instead of the ``parts-of-whole'' structure. To demonstrate this, we plotted the left factor matrices for the \Eigenfaces data for \Cancer and \ALS in Figure~\ref{fig:faces}. At first, it might look like \ALS provides more interpretable results, as most factors are easily identifiable as faces. This, however, is not very interesting result: we already knew that the data has faces, and many factors in the \ALS's result are simply some kind of `prototypical' faces. The results of \Cancer are harder to identify on the first sight. Upon closer inspection, though, one can see that they identify areas that are lighter in the different images, that is, have higher grayscale values. These factors tell us the variances in the lightning in the different photos, and can reveal information we did not know a priori. Further, as seen in Table~\ref{tab:real:accuracy_all}, \Cancer obtains better reconstruction error than \ALS with this data, confirming that these factors are indeed useful to recreate the data. 

\begin{figure}
  \centering
  \subfigure[\Cancer]{%
    \includegraphics[width=0.7\textwidth]{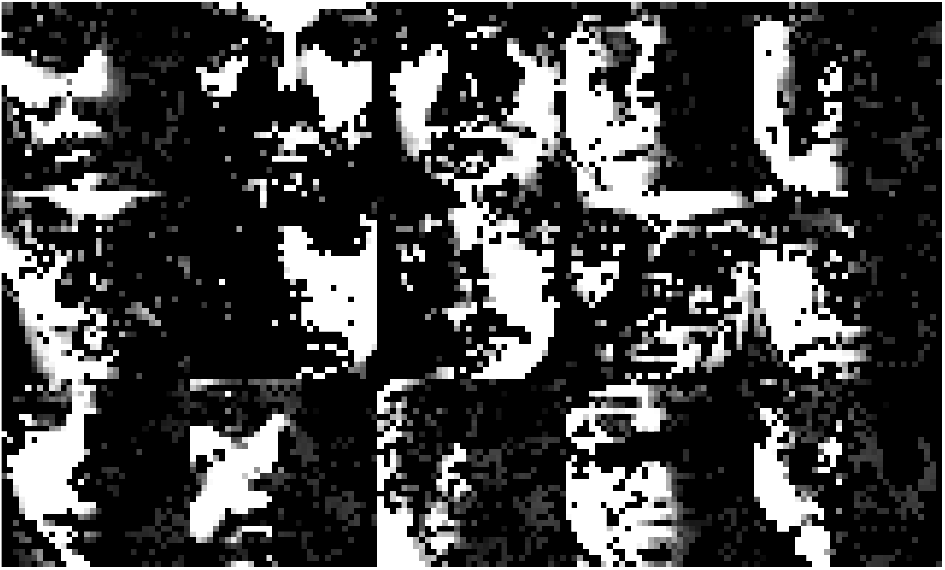}%
    \label{fig:faces:cancer}
  }
  \\
  \subfigure[\ALS]{%
    \includegraphics[width=0.7\textwidth]{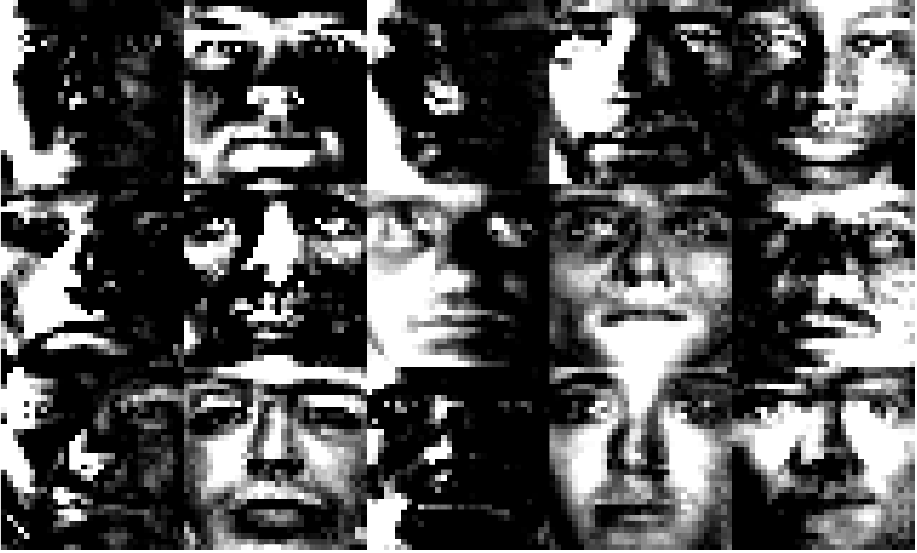}
    \label{fig:faces:als}
  }
  \caption{\Cancer finds the dominant patterns from the \Eigenfaces data. Pictured are the left factor matrices for the \Eigenfaces data.}
  \label{fig:faces}
\end{figure}

In Figure~\ref{fig:wc}, we show some factors from \Cancer when applied to the \Worldclim data. These factors clearly identify different bioclimatic areas from Europe: In Figure~\ref{fig:wc:1} we can identify the mountainous areas in Europe, including the Alps, the Pyrenees, the Scandes, and Scottish Highlands. In Figure~\ref{fig:wc:2} we can identify the mediterranean coastal regions, while in Figure~\ref{fig:wc:3} we see the temperate climate zone in blue, with the green color extending to the boreal zone. In all pictures, red corresponds to (near) zero values. As we can see, \Cancer identifies these areas crisply, making it easy for the analyst to know which areas to look at.

\newlength{\oldsubfiglabelskip}
\setlength{\oldsubfiglabelskip}{\subfiglabelskip}
\subfiglabelskip=0pt
\begin{figure*}[tp]
  \centering
  \subfigure[]{%
    \includegraphics[width=\subfigwidth]{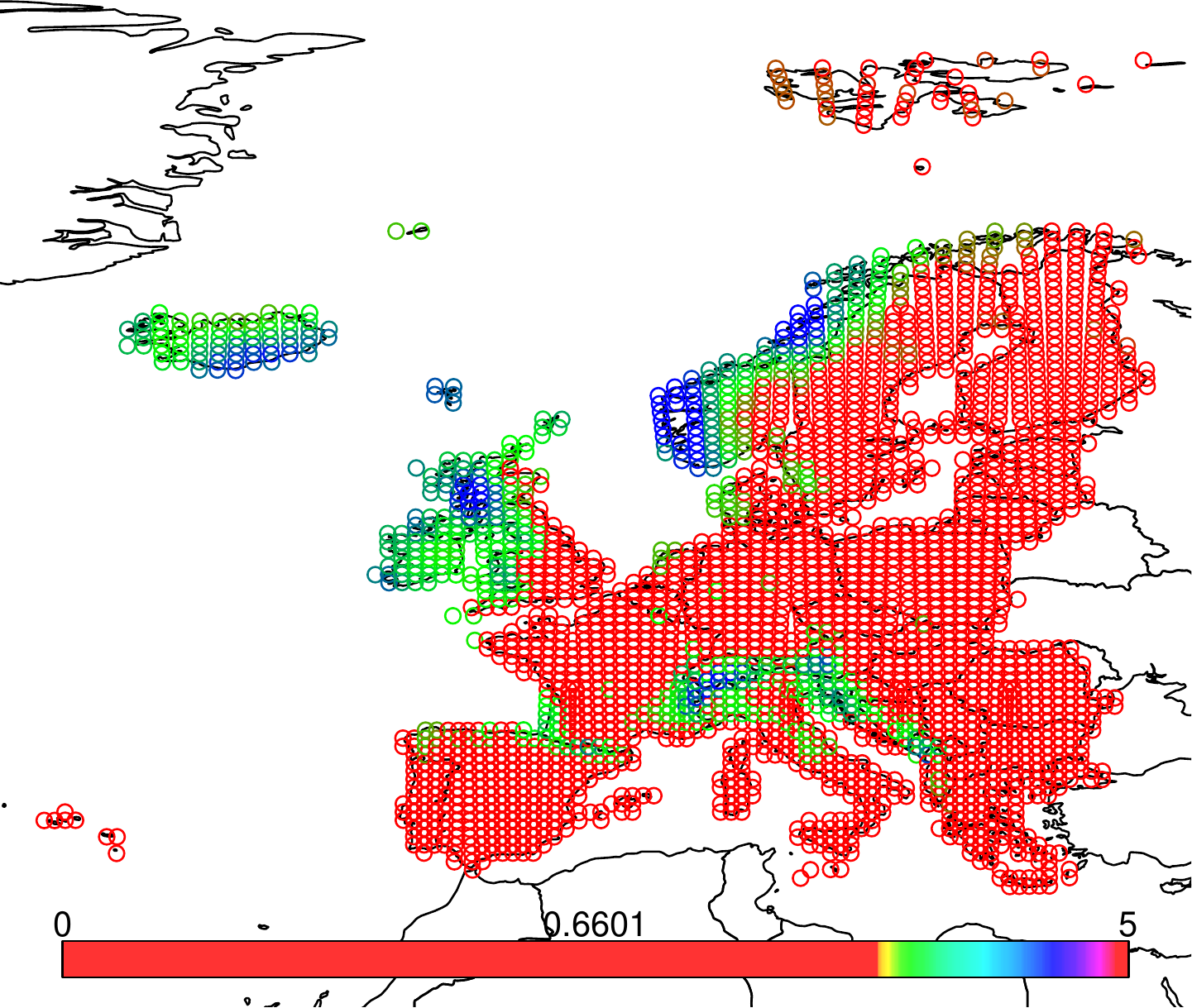}%
    \label{fig:wc:1}%
  }
  \hspace{\subfigspace}
  \subfigure[]{%
    \includegraphics[width=\subfigwidth]{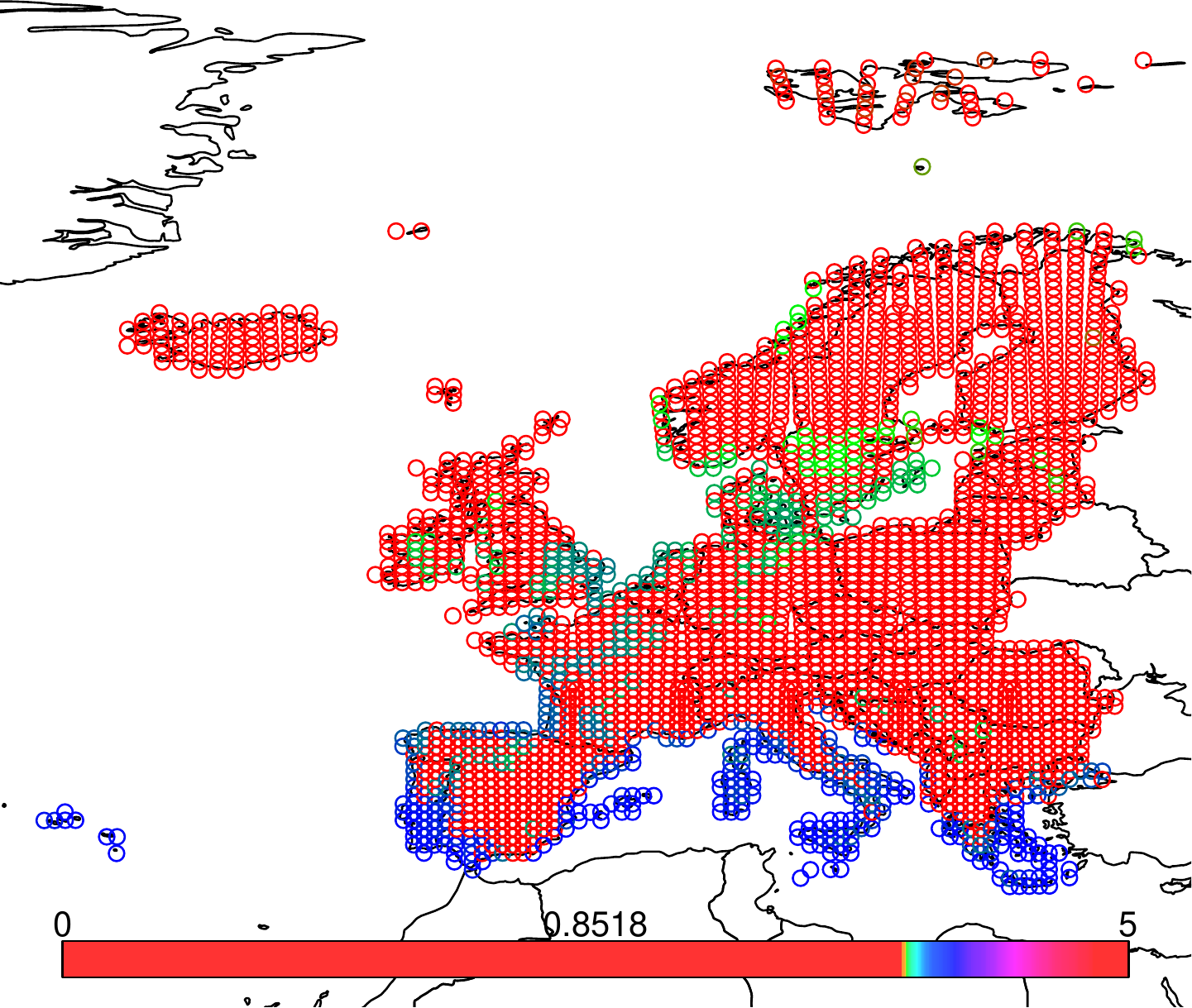}%
    \label{fig:wc:2}%
  }
  \hspace{\subfigspace}
  \subfigure[]{%
    \includegraphics[width=\subfigwidth]{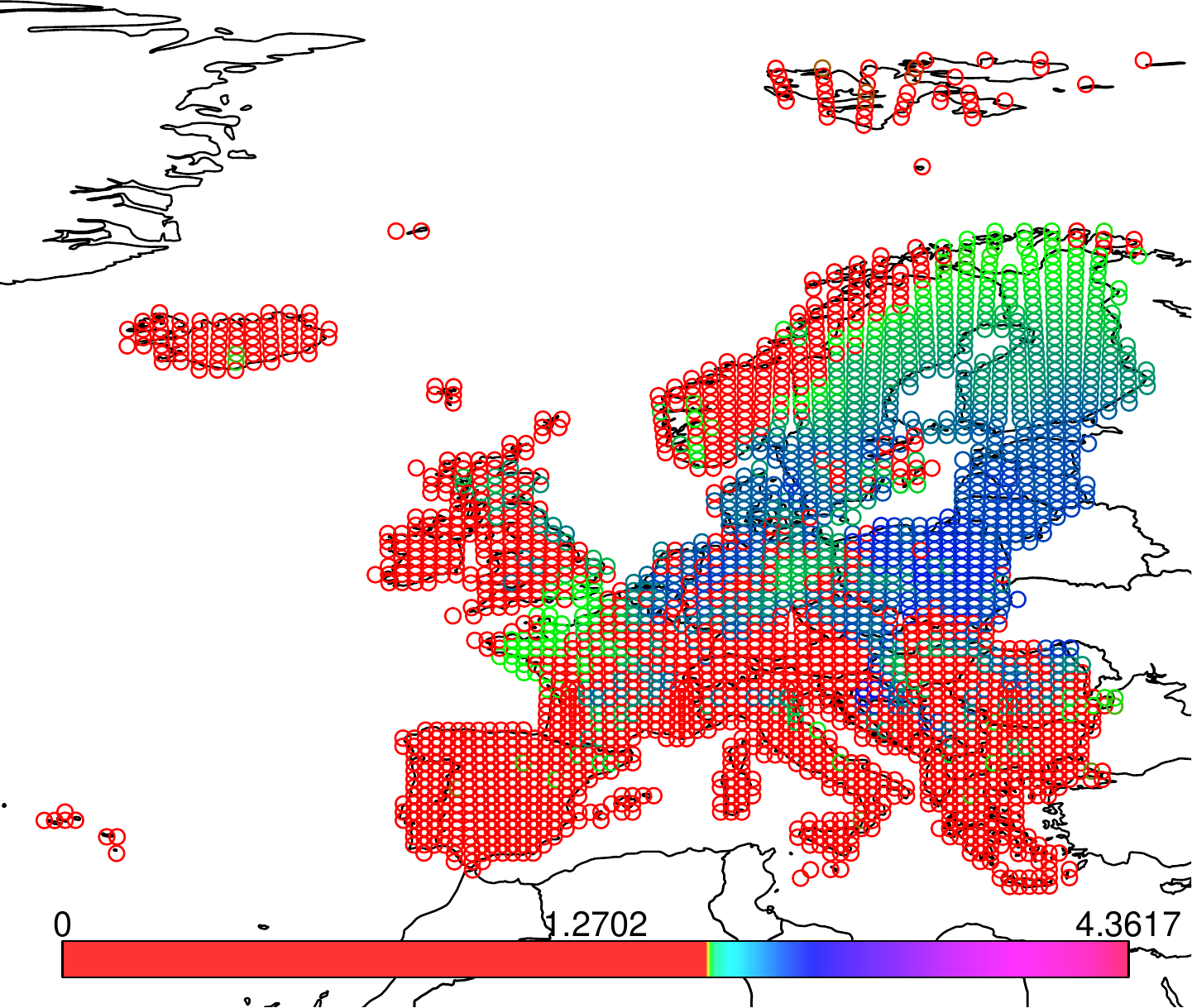}%
    \label{fig:wc:3}
  }
  \caption{\Cancer can find interpretable factors from the \Worldclim data. Shown are the values for three columns in the left-hand factor matrix $\mB$  on a map. Red is zero.}
  \label{fig:wc}
\end{figure*}
\subfiglabelskip=\oldsubfiglabelskip

In order to interpret \NPAS we first observe that each column represents a single personality attribute. Denote by $\matr{A}$ the obtained approximation of the original matrix. For each rank-1 factor $\matr{X}$ and each column $\matr{A}_i$ we define the score $\sigma(i)$ as the number of elements in $\matr{A}_i$ that are determined by $\matr{X}$. By sorting attributes in descending order of $\sigma(i)$ we obtain relative rankings of the attributes for a given factor. The results are shown in Table~\ref{tab:real:npas_interp}. The first factor clearly shows introverted tendencies, while the second one can be summarized as having interests in fiction and games.

\setlength{\tabcolsep}{0.5em}
   \begin{table}[tb]
   \centering
   \caption{Top three attributes for the first two factors of \NPAS.}
   \label{tab:real:npas_interp}
   \begin{tabular}{ll}
     \toprule
Factor 1 & Factor 2 \\
     \midrule
 I am more comfortable with my hobbies & I have played a lot of video games \\ \hspace{1cm} than I am with other people &           \hspace{1cm} \\
 I gravitate towards introspection   & I collect books      \\
I sometimes prefer fictional people to real ones  & I care about super heroes      \\
     \bottomrule
   \end{tabular}
 \end{table}


\section{Related Work}
\label{sec:related-work}
Here we present earlier research that is related to the subtropical matrix factorization. We start by discussing classic methods, such as SVD and NMF, that have long been used for various data analysis tasks, and then continue with approaches that use idempotent structures. Since the tropical algebra is very closely related to the subtropical algebra, and since there has been a lot of research on it, we dedicate the last subsection to discuss it in more detail. 
\subsection{Matrix factorization in data analysis.}
Matrix factorization methods play a crucial role in data analysis as they help to find low-dimensional representations of the data and uncover the underlying latent structure. A classic example of a real-valued matrix factorization is the singular value decomposition (SVD)  \citep[see e.g.]{golub}, which is very well known and finds extensive applications in various disciplines, such as for example signal processing and natural language processing. The SVD of a real $n$-by-$m$ matrix $\mA$ is a factorization of the form  $\mA = \mU \matr{\Sigma} \mV^T,$ where $\mU\in \mathbb{R}^{n \times n}$ and $\mV \in \mathbb{R}^{m \times m}$ are orthogonal matrices, and $\matr{\Sigma} \in \mathbb{R}^{n \times m}$ is a rectangular diagonal matrix with nonnegative entries. An important property of SVD is that it provides the best low-rank approximation of a given matrix with respect to the Frobenius norm \citep{golub}, giving rise to the so called truncated SVD. This property is frequently used to separate important parts of data from the noise. For example, it was used by \citet{jha2011denoising} to remove the noise from sensor data in electronic nose systems. Another prominent usage of the truncated SVD is in dimensionality reduction \citep[see for example][]{sarwar2000application, deerwester1990indexing}.

Despite SVD being so ubiquitous, there are some restrictions to its usage in data mining due to possible presence of negative elements in the factors. In many applications negative values are hard to interpret, and thus other methods have to be used. Nonnegative matrix factorization (NMF) is a way to tackle this problem. For a given nonnegative real matrix $\mA$, the NMF problem is to find a decomposition of $\mA$ into two matrices $\mA \approx \mB\mC$ such that $\mB$ and $\mC$ are also nonnegative. Its applications are extensive and include text mining \citep{pauca2004text}, document clustering \citep{xu2003document}, pattern discovery \citep{brunet2004metagenes}, and many other. This area drew considerable attention after a publication by \citet{lee_seung}, where they provided an efficient algorithm for solving the NMF problem. It is worth mentioning that even though the paper by \citeauthor{lee_seung} is perhaps the most famous in NMF literature, it was not the first one to consider this problem. Earlier works include \citet{paatero1994positive} \citep[see also][]{paatero1997least}, \citet{paatero1999multilinear}, and \citet{cohen1993nonnegative}. \citet{berry2007algorithms} provide an overview of NMF algorithms and their applications. There exist various flavours of NMF that impose different constraints on the factors; for example \citet{hoyer04non-negative} used sparsity constraints. Though both NMF and SVD perform approximations of a fixed rank, there are also other ways to enforce compact representation of data. For example, in maximum-margin matrix factorization constraints are imposed on the norms of factors.
This approach was exploited by \citet{srebro2004maximum}, who showed it to be a good method for predicting unobserved values in a matrix. The authors also indicate that posing constraints on the factor norms, rather than on the rank, yields a convex optimization problem, which is easier to solve.
\subsection{Idempotent semirings.}
The concept of the subtropical algebra is relatively new, and as far as we know, its applications in data mining are not yet well studied. Indeed, its only usage for data analysis that we are aware of was by \citet{weston13nonlinear}, where it was used as a part of a model for collaborative filtering. The authors modeled users as a set of vectors, where each vector represents a single aspect about the user (e.g. a particular area of interest). The ratings are then reconstructed by selecting the highest scoring prediction using the $\max$ operator. Since their model uses $\max$ as well as the standard plus operation, it stands on the border between the standard and the subtropical worlds.


Boolean algebra, despite being limited to the binary set $\{0,1\}$, is related to the subtropical algebra by virtue of having the same operations, and is thus a restriction of the latter to $\{0,1\}$. By the same token, when both factor matrices are binary, their subtropical product coincides with the Boolean product, and hence the Boolean matrix factorization can be seen as a degenerate case of the subtropical matrix factorization problem. The dioid properties of the Boolean algebra can be checked trivially. The motivation for the Boolean matrix factorization comes from the fact that in many applications data is naturally represented as a binary matrix (e.g. transaction databases), which makes it reasonable to seek decompositions that preserve the binary character of the data. The conceptual and algorithmic analysis of the problem was done by \citet{miettinen09matrix}, with the focus mainly on the data mining perspective of the problem. For a linear algebra perspective see \citet{kim1982boolean}, where the emphasis is put on the existence of exact decompositions. A number of algorithms have been proposed for solving the BMF problem \citep{miettinen2008discrete, lu2008optimal, lucchese2014unifying, karaev2015getting}.

\subsection{Tropical algebra.}
Another close cousin of the max-times algebra is the max-plus, or so called tropical algebra, which uses plus in place of multiplication. It is also a dioid due to the idempotent nature of the $\max$ operation. As was mentioned earlier, the two algebras are isomorphic, and hence many of the properties are identical (see Sections \ref{sec:notation} and \ref{sec:theory} for more details).

Despite the theory of the tropical algebra being relatively young, it has been thoroughly studied in recent years. The reason for this is that it finds extensive applications in various areas of mathematics and other disciplines. An example of such a field is the discrete event systems (DES)~\citep{cassandras08introduction}, where the tropical algebra is ubiquitously used for modeling \citep[see e.g.][]{baccelli92synchronization,cohen99max-plus}. Other mathematical disciplines where the tropical algebra plays a crucial role are optimal control \citep{gaubert1997methods}, asymptotic analysis \citep{dembo2010large, maslov1992idempotent, akian1999densities}, and decidability \citep{simon1978limited, simon1994semigroups}.

Research on tropical matrix factorization is of interest for us because of the above mentioned isomorphism between the two algebras. However as was explained in Section \ref{sec:theory}, the approximate matrix factorizations are not directly transferable as the errors can differ dramatically. It should be mentioned that in the general case the problem of the tropical matrix factorization is NP-complete \citep[see e.g.][]{shitov2014complexity}. \Citet{de2002qr} demonstrated that if the max-plus algebra is extended in such a way that there is an additive inverse for each element, then it is possible to solve many of the standard matrix decomposition problems. Among other results the authors obtained max-plus analogues of QR and SVD. They also claimed that the techniques they propose can be readily extended to other types of classic factorizations (e.g. Hessenberg and LU decomposition). Despite the apparent successes in the realm of tropical matrix factorization, its subtropical counterpart has not received much attention, and to the best of our knowledge the first work on the subject was done by \citet{karaev16capricorn}.

The problem of solving tropical linear systems of equations arises naturally in numerous applications, and is also closely related to matrix factorization. In order to illustrate this connection, assume that we are given a tropical matrix $\mA \in \tropalg^{n\times m}$ and one of the factors $\mB \in \tropalg^{n\times k}$. Then the other factor $\mC \in \tropalg^{k\times m}$ can be found by solving the following set of problems
\begin{equation} \label{alternating_updates}
  \mC_j = \argmin_{\vc \in \tropalg^k} \|\mB \tropprod \vc - \mA_j\|_F, \; j=1,\dots, m\;.
  \end{equation}
  Each problem in \eqref{alternating_updates} requires ``approximately'' solving a system of tropical linear equations. The minus operation in \eqref{alternating_updates} does not belong to the tropical semiring, so the approximation here should be understood in terms of minimizing the classical distance. The general form of tropical linear equations
\begin{equation} \label{trop:lin}
  \mA \vx \tropadd \vb = \mC \vx \tropadd \vd
\end{equation}
is not always solvable \citep[see e.g.][]{gaubert1997methods}; however various techniques exist for checking the existence of the solution for particular cases of \eqref{trop:lin}.

For equations of the form $\mA \vx = \vb$ the feasibility can be established for example through the so called \emph{matrix residuation}. There is a general result that for an $n$-by-$m$ matrix $\mA$ over a complete idempotent semiring, the existence of the solution can be checked in $O(nm)$ time \citep[see][]{gaubert1997methods}. Although the tropical algebra is not complete, there is an efficient way of finding if the solution exists \citep{cuninghame1979minimax, zimmermann2011linear}. It was shown by \citet{butkovivc2003max} that this type of tropical equations is equivalent to the set cover problem, which is known to be NP-hard. This directly affects the max-times algebra through the above-mentioned isomorphism and makes the problem of precisely solving max-times linear systems of the form $\mA \vx = \vb$ infeasible for high dimensions.

Homogeneous equations $\mA\vx = \mB\vx$ can be solved using the \emph{elimination} method, which is based on the fact that the set of solutions of a homogeneous system is a finitely generated semimodule \citep{butkovic1984elimination} \citep[independently rediscovered by][]{gaubert1992theorie}. If only a single solution is required, then according to \citet{gaubert1997methods}, a method by \citet{walkup1998general} is usually the fastest in practice.

Now let $\mA$ be a tropical square matrix of size $n\times n$. For complete idempotent semirings a solution to the equation $\vx = \mA \vx \tropadd \vb$ is given by $\vx = \matr{A^*}\vb$ \citep[see e.g.][]{salomaa2012automata}, where the operator $\matr{A^*}$ is defined as 
\[
\matr{A^*} = \tropadd_{k=1}^\infty \mA^k \;.
\]
Since the tropical semiring is not complete (it is missing the $\infty$ element), $\matr{A^*}$ can not always be computed. However, when there are no positive weight circuits in the graph defined by $\mA$, then we have $\matr{A^*} = \mA^0 \tropadd \dots \tropadd \mA^{n-1}$, and all entries of $\matr{A^*}$ belong to the tropical semiring \citep{baccelli92synchronization}. Computing the operator $\mA^*$ takes time $O(n^3)$ \citep[see e.g.][]{gondran1984graphs, gaubert1997methods}.

Another important direction of research is the eigenvalue problem $\mA\vx = \lambda\vx$. Tropical analogues of the Perron--Frobenius theorem \citep[see e.g.][]{vorobyev2extremal, maslov1992idempotent}, and Collatz--Wielandt formula \citep{bapat1995pattern, gaubert1992theorie} were developed. For a general overview of the results in the $(\max, +)$ spectral theory, see for example \citet{gaubert1997methods}.

Tropical algebra and tropical geometry were used by \citet{gartnertropical} to construct a tropical analogue of an SVM. Unlike in the classical case, tropical SVMs are localized, in the sense that the kernel at any given point is not influenced by all the support vectors. Their work also utilizes the fact that tropical hyperplanes are somewhat more complex than their counterparts in the classical geometry, which makes it possible to do multiple category classification with a single hyperplane.



\section{Conclusions}
\label{sec:conclusions}

Subtropical low-rank factorizations are a novel approach for finding latent structure from nonnegative data. The factorizations can be interpreted using the winner-takes-it-all interpretation: the value of the element in the final reconstruction depends only on the largest of values in the corresponding elements of the rank-1 components (cf.\ NMF, where the value in the reconstruction is the \emph{sum} of the corresponding elements). That the factorizations are different does not necessarily mean that they are better in the terms of reconstruction error, although they can yield lower reconstruction error than even SVD. It does mean, however, that they find different structure from the data. This is an important advantage, as it allows the data analyst to use both the classical factorizations and the subtropical factorizations to get a broader understanding of the kinds of patterns that are present in the data.

Working in the subtropical algebra is harder than in the normal algebra, though. The various definitions for the rank, for example, do not agree, and computing many of them -- including the subtropical Schein rank, which is arguably the most useful one for data analysis -- is computationally hard. That said, our proposed algorithms, \Capricorn and \Cancer, can find the subtropical structure when it is present in the data. Not every data have subtropical structure, though, and due to the complexity of finding the optimal subtropical factorization we cannot distinguish between the cases where our algorithms fail to find the latent subtropical structure, and where it does not exist. Based on our experiments with synthetic data, our hypothesis is that the failure of finding a good factorization indicates the lack of the subtropical structure rather than the algorithms' failure.

That said, the presented algorithms are heuristics. Developing algorithms that achieve better reconstruction error is naturally an important direction of future work. In our \Equator framework, this hinges on the task of finding the rank-1 components. In addition, the scalability of the algorithms could be improved. A potential direction could be to take into account the sparsity of the factor matrices in dominated decompositions. This could allow one to concentrate only on the non-zero entries in the factor matrices. 

The connection between Boolean and (sub-)tropical factorizations raises potential directions for future work. The continuous framework could allow for easier optimization in the Boolean algebra. Also, the connection allows us to model combinatorial structures (e.g.\ cliques in a graph) using subtropical matrices. This could allow for novel approaches on finding such structures using continuous subtropical factorizations.


\bibliographystyle{abbrvnat}
\bibliography{library}  
\end{document}